%% file: main.tex
\documentclass[lettersize,journal]{IEEEtran}
\usepackage{amsmath,amsfonts}
\usepackage{algorithmic}
\usepackage{algorithm}
\usepackage{array}
\usepackage[caption=false,font=normalsize,labelfont=sf,textfont=sf]{subfig}
\usepackage{textcomp}
\usepackage{stfloats}
\usepackage{url}
\usepackage{verbatim}
\usepackage{graphicx}
\usepackage{cite}
\usepackage{Jiazhi_style}

\begin{document}

% \title{Information-Theoretic Bounds on The Removal of Attribute-Specific Bias From Neural Networks}
\title{SABAF: Removing Strong Attribute Bias from Neural Networks with Adversarial Filtering}

\author{Jiazhi Li, Mahyar Khayatkhoei, Jiageng Zhu, Hanchen Xie, Mohamed E. Hussein, Wael AbdAlmageed
        % <-this % stops a space
% \thanks{Manuscript received XXX, 2023; revised XXX, 2023.}
        % <-this % stops a space
\thanks{Jiazhi Li, Mahyar Khayatkhoei, Jiageng Zhu, Hanchen Xie, Mohamed E. Hussein, Wael AbdAlmageed are with USC Information Sciences Institute, Marina del Rey, USA (e-mail: jiazhil@usc.edu; m.khayatkhoei@gmail.com; jiagengz@usc.edu; hanchenx@usc.edu; mehussein@isi.edu; wamageed@isi.edu).

Jiazhi Li and Jiageng Zhu are also with USC Ming Hsieh Department of Electrical and Computer Engineering, Los Angeles, USA.

Hanchen Xie is also with USC Thomas Lord Department of Computer Science, Los Angeles, USA.

Mohamed E. Hussein is also with Alexandria University, Alexandria, Egypt.

Wael AbdAlmageed is also with USC Ming Hsieh Department of Electrical and Computer Engineering, Los Angeles, USA, and USC Thomas Lod Department of Computer Science, Los Angeles, USA.}}

% \markboth{IEEE TRANSACTIONS ON PATTERN ANALYSIS AND MACHINE INTELLIGENCE, VOL. XXX, NO. XXX, AUGUST YYYY}%


\maketitle

%%%%%%%%% ABSTRACT
\input{sections/00_abstract}
\input{sections/pictures/CMNIST_EB}

%%%%%%%%% BODY TEXT
\section{Introduction}
\input{sections/01_introduction}

\section{Related Works}
\input{sections/02_related_work}

\section{Information Theoretical Bounds on the Performance of Attribute Bias Removal Methods}
\input{sections/04_theory}

\section{Necessary Condition to Remove Extreme Bias}
\input{sections/pictures/fig_idea}

\input{sections/05_necessary_assumption}

\section{Adversarial Filtering of Attribute Bias}
\input{sections/pictures/framework+visualization}
\input{sections/06_method}
% \input{sections/pictures/Paper_visualization}

%%% EXPERIMENTS
\section{Experimental Evaluation}
\input{sections/tables/CelebA_EB_BlondHair}
\input{sections/tables/Adult_EB}
\input{sections/pictures/CelebA_EB}
\input{sections/pictures/Adult_EB}
\input{sections/07_experiment}

\section{Conclusion}
\input{sections/08_discussion}

\input{sections/acknowledgement}

\bibliography{reference}
\bibliographystyle{IEEEtran}

% \newpage

% \input{sections/Biography}

\clearpage
\appendix
\begin{appendices}

% \section*{Appendix}
\input{sections/supplement}

\end{appendices}

\end{document}

%% file: sections/00_abstract.tex
\begin{abstract}
Ensuring a neural network is not relying on protected attributes (\eg race, sex, age) for prediction is crucial in advancing fair and trustworthy AI. 
While several promising methods for removing attribute bias in neural networks have been proposed, their limitations remain under-explored. 
To that end, in this work, we mathematically and empirically reveal the limitation of existing attribute bias removal methods in presence of strong bias and propose a new method that can mitigate this limitation.
Specifically, we first derive a general non-vacuous information-theoretical upper bound on the performance of any attribute bias removal method in terms of the bias strength, revealing that they are effective only when the inherent bias in the dataset is relatively weak. Next, we derive a necessary condition for the existence of any method that can remove attribute bias regardless of the bias strength. Inspired by this condition, we then propose a new method using an adversarial objective that directly filters out protected attributes in the input space while maximally preserving all other attributes, without requiring any specific target label. The proposed method achieves state-of-the-art performance in both strong and moderate bias settings. We provide extensive experiments on synthetic, image, and census datasets, to verify the derived theoretical bound and its consequences in practice, and evaluate the effectiveness of the proposed method in removing strong attribute bias\footnote{Code will be released at \url{https://github.com/jiazhi412/strong_attribute_bias}}.

\begin{IEEEkeywords}
Trustworthy AI, Bias, Neural Networks, Protected Attributes, Information Theory, Adversarial Filter
\end{IEEEkeywords}

\end{abstract}

%% file: sections/pictures/CMNIST_EB.tex
\begin{figure}[t]
\begin{center}
  \includegraphics[width=1\linewidth]{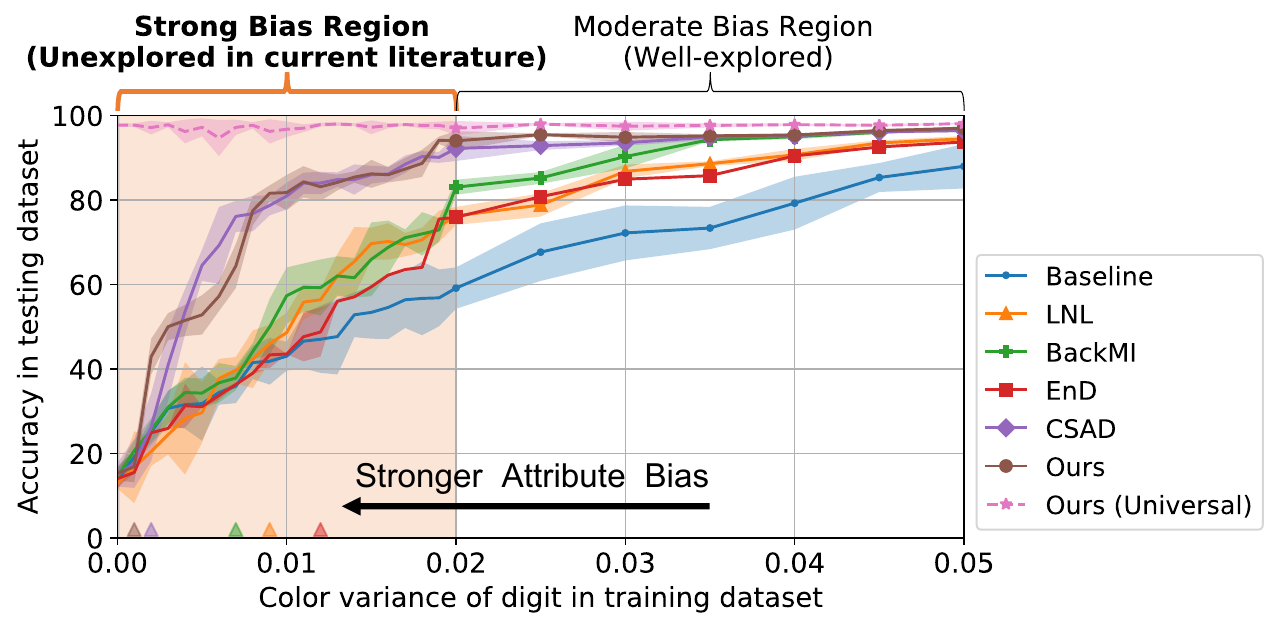}
\end{center}
  \caption{Digit prediction accuracy of bias removal methods trained under different levels of color bias strength in Colored MNIST, showing results on the unexplored region of color variance $<0.02$. The breaking point of each method, where its performance becomes statistically similar to the baseline classifier, is labeled with $\blacktriangle$ on the x-axis. While all methods clearly outperform the baseline in moderate bias region, their effectiveness sharply declines towards the baseline as the bias strength increases. Our proposed method shows a lower breaking point, and no breaking point when a \Univ{} distribution is available. The plot shows average accuracy (lines) with one standard deviation error (shaded) over 15 randomized training runs. Further details are provided in~\cref{appsec:breaking_point}.}
\label{fig:EB}
\end{figure}

%% file: sections/01_introduction.tex
\label{sec:introduction}

\emph{Protected attributes} is a term originating from Sociology~\cite{sociology} referring to a finite set of attributes that must not be used in decision-making to prevent exacerbating societal biases against specific demographic groups~\cite{protected_attributes}. For example, in deciding whether or not someone should be qualified for a bank loan, race (as one of the protected attributes)  must not influence the decision. 
Given the widespread use of neural networks in real-world decision-making, developing methods capable of explicitly excluding protected attributes from the decision process -- more generally referred to as removing attribute bias~\cite{minority_group_vs_sensitive_attribute} -- is of paramount importance.

While many promising methods for removing attribute bias in neural networks have been proposed in the recent years~\cite{BlindEye_IMDB_eb,learn_not_to_learn_Colored_MNIST,domain_independent_training, LfF_CelebA_Bias_conflicting,End,CSAD,BCL}, the limitations of these methods remain under-explored. In particular, existing studies explore the performance of these methods only in cases where the protected attribute (\eg race) is \emph{not strongly predictive} of the prediction target (\eg credit worthiness). However, this implicit assumption does not always hold in practice, especially in cases where training data is scarce. For example, in diagnosing Human Immunodeficiency Virus (HIV) from Magnetic Resonance Imaging (MRI), HIV-positive subjects were found to be significantly older than control subjects, making age a strong attribute bias for this task~\cite{dataset_vs_task}. Another example is the Pima Indians Diabetes Database which contains only 768 samples where several spurious attributes become strongly associated with diabetes diagnosis~\cite{diabetes_dataset, diabetes_chapter}. Even the widely-used CelebA dataset~\cite{CelebA} contains strong attribute biases: for example in predicting blond hair, sex is a strong predictor\footnote{We present detailed statistics of attribute biases in various real-world datasets in~\cref{appsec:stats}.}. Therefore, it is crucial to study bias removal methods beyond the moderate bias setting to better understand their limitations and the necessary conditions for their effectiveness.

%%% WE EXPLORE EXTREME BIAS
In~\cref{fig:EB}, we will illustrate by a specific example the limitation in bias removal methods that we will later investigate theoretically and empirically in several real-world datasets. In this example, we conduct an extended version of a popular controlled experiment for evaluating the performance of attribute bias removal~\cite{learn_not_to_learn_Colored_MNIST, Back_MI, CSAD}. The task is to predict digits from colored MNIST images~\cite{learn_not_to_learn_Colored_MNIST} where color is considered a protected attribute. During training, each digit is assigned a unique RGB color with a variance (\ie the smaller the color variance, the more predictive the color is of the digit, and the stronger the attribute bias). To measure how much the trained model relies on the protected attribute in its predictions, model accuracy is reported on a held-out subset of MNIST with uniformly random color assignments (\ie where the color is not predictive of the digit).  While state-of-the-art  methods~\cite{learn_not_to_learn_Colored_MNIST,Back_MI,End,CSAD} report results for the color variance only in the range $[0.02, 0.05]$ (without providing any justification for this particular range), we explore results for the missing range of $[0, 0.02]$, which we denote as the \emph{strong bias region}. 
In~\cref{fig:EB}, in the strong bias region, we observe that the effectiveness of all existing methods sharply declines and there exists a \emph{breaking point} in their effectiveness. 
The breaking point of a method is defined as the weakest bias strength at which its performance becomes indistinguishable from the baseline under a two-sample one-way Kolmogorov-Smirnov test with significance level of $0.05$.
The main goal of this paper is to study the cause and extent of this limitation empirically and theoretically. 
We summarize our contributions below\footnote{This work is an extended version of our paper~\cite{extreme_bias} presented in the Algorithmic Fairness through the Lens of Time Workshop at NeurIPS 2023.}:
 
\begin{itemize}
\item Providing comprehensive empirical evidence showing the acute degradation in the performance of existing attribute bias removal methods as the bias strength increases in the underlying training dataset (\cref{fig:EB} and \cref{sec:exp}).

\item Deriving and verifying a general non-vacuous information-theoretical upper bound for the performance of any attribute bias removal method, thereby formalizing the cause and extent of the limitation (\cref{sec:theory}).

\item Deriving a necessary condition for the existence of any method that can remove attribute bias
regardless of the bias strength (\cref{sec:necessary}).

\item Proposing a new method for strong attribute bias removal based on the necessary condition (\cref{sec:method}).

\item Conducting extensive experiments to evaluate the effectiveness of the proposed method in both moderate and strong bias settings, demonstrating state-of-the-art performance in both settings even when the necessary condition is not met, showing that we are not sacrificing performance in moderate bias region to gain performance in strong bias region (\cref{sec:exp}).
\end{itemize}

%% file: sections/02_related_work.tex
\label{sec:related_work}

\noindent
\textbf{Bias in Neural Networks.} Mitigating bias and improving fairness in neural networks has received considerable attention in recent years~\cite{EO_define,demographic_parity,counterfactual_fairness,fairness_through_awareness,fairness_under_unawareness,RLB}. The methods proposed for mitigating bias in neural networks can be broadly grouped into two categories: 1) methods that aim to mitigate the uneven performance of neural networks between majority and minority groups; and 2) methods that aim to reduce the dependence of neural network prediction on specific attributes. Most notable examples of the former group are methods for constructing balanced training set~\cite{Timnit_sex_PPB,Fairface}, synthesizing additional samples from the minority group~\cite{transect,CAT}, importance weighting the under-represented samples~\cite{RL_RBN}, and domain adaptation techniques that adapt well-learnt representations from the majority group to the minority group~\cite{RFW,MFR,BAE}. In this work, we focus on the second group of methods, which we will further divide into two subgroups discussed below: methods that implicitly or explicitly minimize the mutual information (MI) between learnt latent features and the specific protected attribute.

\noindent
\textbf{Explicit Mutual Information Minimization.} 
Several methods aim to directly minimize mutual information between a latent representation for the target classification and the protected attributes, in order to learn a representation that is predictive of the target but independent of the attributes, hence removing attribute bias. These methods mainly differ in the way they estimate MI. Most notable examples include LNL~\cite{learn_not_to_learn_Colored_MNIST} which minimizes the classification loss together with a MI regularization loss estimated by an auxiliary distribution; BackMI~\cite{Back_MI} which minimizes classification loss and MI estimated by a neural estimator~\cite{MINE} through the statistics network; and, CSAD~\cite{CSAD} which minimizes MI estimated by~\cite{deepinfomax} between a latent representation to predict target and another latent representation to predict the protected attributes.

\noindent
\textbf{Implicit Mutual Information Minimization.} 
Another group of methods aims to remove attribute bias by constructing surrogate losses that implicitly reduce the mutual information between protected attributes and the target of classification. Most notably, LfF~\cite{LfF_CelebA_Bias_conflicting} proposes training two models simultaneously, where the first model will prioritize easy features for classification by amplifying the gradient of cross-entropy loss with the predictive confidence (softmax score), and the second model will down-weight the importance of samples that are confidently classified by the first model, therefore avoiding features that are learnt easily during training, which are likely to be spurious features leading to large MI with protected attributes; EnD~\cite{End} adds regularization terms to the typical cross-entropy loss that push apart the feature vectors of samples with the same protected attribute label to become orthogonal (thereby increasing the conditional entropy of them given the protected attribute); BlindEye~\cite{BlindEye_IMDB_eb} pushes the distribution obtained by the attribute classifier operating on latent features towards uniform distribution by minimizing the entropy between them; domain independent training (DI)~\cite{domain_independent_training} learns a shared representation with an ensemble of separate classifiers per domain to ensure that the prediction from the unified model is not biased towards any domain; and, domain-invariant learning~\cite{ganin2016domain, zhao2019learning, albuquerque2019generalizing,invariant1, invariant2,DRO} minimize classification performance gap across domains by mapping data to a space where distributions are indistinguishable while retaining task-relevant information.

\noindent
\textbf{Generative Dataset Augmentation.} 
A recent group of methods~\cite{CGN,Camel, BiaSwap, GAN_Debiasing_hat_glasses_correlation} aim to mitigate attribute bias by generating counterfactual synthetic samples that can augment the original biased training set to reduce its inherent bias strength.
These methods use Generative Adversarial Networks (GANs)~\cite{GAN} to synthesize images of a given biased dataset by randomly altering the protected attribute, a technique commonly denoted \textit{attribute flipping}.
Compared with MI-based methods, GAN-based methods address attribute bias by constructing a semi-synthetic dataset with reduced bias strength rather than minimizing mutual information between learned features and protected attributes.
Most notably, CAMEL~\cite{Camel} starts by employing a CycleGAN~\cite{cyclegan} to learn the semantic transformations between latent features with the same target attribute but different protected attribute, and then performs data augmentations by manipulating the latent features for classifier training; BiaSwap~\cite{BiaSwap} first employs a biased classifier to divide samples into bias-guiding and bias-contrary categories, and then incorporates the style-transferring module of the image translation model to produce bias-swapped images which retain bias-irrelevant features from bias-guiding samples while inheriting protected attributes from bias-contrary samples; GAN-Debiasing~\cite{GAN_Debiasing_hat_glasses_correlation} formulates two hyperplanes to represent both the target attribute and the protected attribute, and generates synthetic images that retain the appearance of the target attribute while flipping the protected attribute by perturbing latent vector in the protected attribute hyperplane; and, CGN~\cite{CGN} learns three predefined independent mechanisms for shape, texture, and background based on domain knowledge, and leverages them to generate images with desired attributes.

\noindent
\textbf{Trade-offs between Bias Removal and Model Utility.}
The trade-offs between fairness and accuracy in machine learning models have garnered significant discussion.
Most notably, Kleinberg \etal~\cite{three_fairness_conditions} prove that except in highly constrained cases, no method can simultaneously satisfy three fairness conditions: \emph{calibration within groups} which requires that the expected number of individuals predicted as positive should be proportional to a group-specific fraction of individuals in each group, \emph{balance for the negative class} which requires that the average score of individuals predicted as negative should be equal across groups, and \emph{balance for the positive class} which requires the balance for the positive class across groups; and, Dutta \etal~\cite{Eopps_and_accuracy} theoretically demonstrate that, under certain conditions, it is possible to simultaneously achieve optimal accuracy and fairness in terms of \emph{equal opportunity}~\cite{EO_define} which requires even false negative rates or even true positive rates across groups.
Different from the fairness criteria discussed in these works, we focus on another well-known fairness criterion, \emph{demographic parity}~\cite{counterfactual_fairness, fairness_through_awareness}, which requires even prediction probability across groups, \ie independence between model prediction and protected attributes.
Regarding this criterion, Zhao and Gordon~\cite{dp_to_ap} show that any method designed to learn fair representations, while ensuring model predictions are independent of protected attributes, faces an information-theoretic lower bound on the joint error across groups.
In contrast, we derive an information-theoretic upper bound on the best attainable performance $I(Z;Y)$, which is not limited to the case where model predictions are independent of protected attributes and considers different levels of the retained protected attribute information in the learnt features.

%% file: sections/04_theory.tex
\label{sec:theory}
The observations in~\cref{fig:EB} revealed that the existing methods are not effective when the attribute bias is too strong, \ie they all have a breaking point, and that there is a continuous connection between their effectiveness and the strength of the attribute bias. However, so far, these observations are limited to the particular Colored MNIST dataset. In this section, we show that this situation is in fact much more general. By deriving an upper bound on the classification performance of any attribute bias removal method in terms of the bias strength, regardless of the dataset and domain, we will elucidate the cause and extent of the limitation we observed in~\cref{fig:EB}.

\input{sections/pictures/CelebA_bounds}

We first need to formalize the notions of performance, attribute bias strength, and attribute bias removal. Let $X$ be a random variable representing the input (\eg images) with support $\mathcal{X}$, $Y$ be a random variable representing the prediction target (\eg hair color) with support $\mathcal{Y}$, and $A$ be a random variable representing the protected attribute (\eg sex). We define the attribute bias removal method as a function $f: \mathcal{X} \rightarrow \mathcal{Z}$ that maps input data to a latent bottleneck feature space $\mathcal{Z}$ inducing the random variable $Z$, and consider the prediction model as a function $g: \mathcal{Z} \rightarrow \mathcal{Y}$ inducing the random variable $\hat{Y}$. According to the information bottleneck theory~\cite{tishby2015deep-bottleneck-theory, shwartz2017opening-blackbox-deep-learning}, the goal of classification can be stated as maximizing the mutual information between prediction and target, namely $I(\hat{Y};Y)$, which is itself bounded by the mutual information between feature and target due to the data processing inequality~\cite{cover2006elements}, \ie $I(\hat{Y};Y) \leq I(Z;Y)$. Intuitively, $I(Z;Y)$ measures how informative the features learnt by the model are of the target, with $I(Z;Y)=0$ indicating completely uninformative learnt features: the best attainable prediction performance is no better than random guess. Therefore, the optimization objective of attribute bias removal methods can be formalized as learning $f$ parameterized by $\theta$ that minimizes mutual information between feature and attribute $I(Z_{\theta};A)$, while maximizing mutual information between feature and target $I(Z_{\theta};Y)$.

Given the above definitions, we can state our goal concretely: to derive a connection between $I(Z;Y)$ (the best attainable performance), $H(Y|A)$ (the attribute bias strength measured by the conditional entropy of target given attribute), and $I(Z;A)$ (the amount of remained attribute bias in the learnt feature). 
Note that smaller $H(Y|A)$ corresponds to stronger attribute bias (\ie the attribute can more certainly predict the target), and vice versa. So the extreme attribute bias happens when $H(Y|A)=0$. In this particular setting, the following proposition shows that no classifier can outperform random guess if the attribute is removed from the feature.

\begin{propos}
\label{th:extreme}
Given random variables $Z, Y, A$, in case of the extreme attribute bias $H(Y|A)=0$, if the attribute is removed from the feature $I(Z;A)=0$, then $I(Z;Y)=0$, \ie no classifier can outperform random guess\footnote{\label{fnote_proof}Proof in~\cref{appsec:proofs}.}.
\end{propos}

This proposition extends and explains the observation on the leftmost location of the x-axis in~\cref{fig:EB}: when the color variance is zero, color is completely predictive of the digit, $H(Y|A)=0$, and removing color from the latent feature, $I(Z;A)=0$, makes the prediction uninformative, $I(Z;Y)=0$. However, \cref{th:extreme} does not explain the rest of the curve beyond just the zero color variance. The following theorem closes this gap by deriving a bound on the performance of attribute bias removal methods in terms of the attribute bias strength, thus providing a more complete picture of the limitation of such methods, and elucidating the connection between performance and bias strength.

\begin{theorem}
\label{th:general}
Given random variables $Z, Y, A$, the following inequality holds without exception\cref{fnote_proof}:
% ~\footnotemark[\value{fnote_proof}]
\begin{align}
\label{eq:genreal}
    0 \leq I(Z;Y) \leq I(Z;A) + H(Y|A)
\end{align}
\end{theorem}

\begin{remark}
In the extreme bias case $H(Y|A)=0$, the bound in~\cref{eq:genreal} shows that the model performance is bounded by the amount of protected attribute information that is retained in the feature, namely $I(Z;Y) \leq I(Z;A)$. This puts the model in a trade-off: the more the attribute bias is removed, the lower the best attainable performance.
\end{remark}

\begin{remark}
When the protected attribute is successfully removed from the feature $I(Z;A)=0$, the bound in~\cref{eq:genreal} shows that the model's performance is bounded by the strength of the attribute bias, namely $I(Z;Y)\leq H(Y|A)$. This explains the gradual decline observed in~\cref{fig:EB} as we move from the moderate to the strong bias region (from right to left towards zero color variance).
\end{remark}

\begin{remark}
When $H(Y|A)=0$ and $I(Z;A)=0$, \cref{eq:genreal} reduces to the result of~\cref{th:extreme}, $I(Z;Y)=0$, hence no classifier can outperform random guess.
\end{remark}

\begin{remark}
\label{th:best_att_perf}
We emphasize that the provided bound is placed on the best attainable performance. So while decreasing the bound will decrease performance, increasing the bound will not necessarily result in an increased performance. For example, consider the baseline classifier: even though there is no attribute bias removal performed and therefore the bound can be arbitrarily large, $I(Z;A) \gg 0$, the model declines in the strong bias region since learning the highly predictive protected attribute is likely in the non-convex optimization.
\end{remark}

To empirically test our theory in a real-world dataset, we compute the terms in~\cref{th:general} for several existing methods in CelebA and plot the results in~\cref{fig:CelebA_bounds}. In these experiments, blond hair is the target $Y$, and sex is the protected attribute $A$. We vary the bias strength $H(Y|A)$ by increasing/decreasing the fraction of bias-conflicting images in the training set (images of females with non-blond hair and males with blond hair) while maintaining the number of biased images in the training set at 89754. Then, we compute $H(Y|A)$ directly and estimate the mutual information terms $I(Z;A)$ and $I(Z;Y)$ using mutual information neural estimator~\cite{MINE}. We observe that the bound holds in accordance with~\cref{th:general} for all methods (additional methods are provided in~\cref{appsec:bounds}). 

So far, we have mathematically and empirically shown the existence of the bound, and we have observed its consequence in the synthetic Colored MNIST dataset in~\cref{fig:EB}. We will further investigate the extent of its consequences for attribute bias removal methods in real-world image and census datasets in~\cref{subsec:comp_extreme,subsec:comp_strong}. Prior to that, in the following two sections, we will investigate whether we can design a method that can mitigate the limitation of removing strong bias.

%% file: sections/pictures/CelebA_bounds.tex
\begin{figure*}[t]
    \centering % <-- added
      \subfloat[EnD~\cite{End}.]{\includegraphics[width=0.26\linewidth]{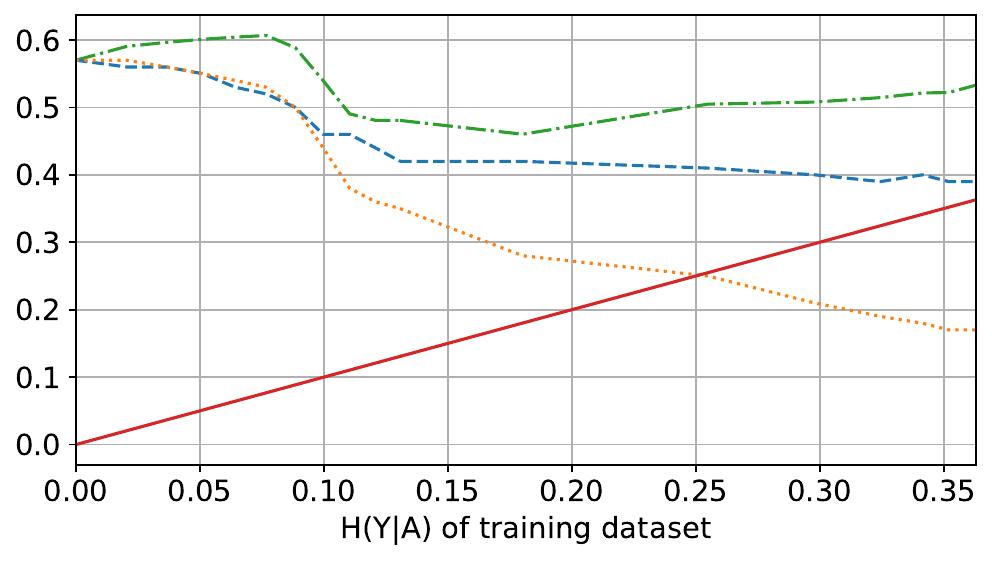}\label{fig:CelebA_EnD_bound}}\quad
      \subfloat[CSAD~\cite{CSAD}.]{\includegraphics[width=0.26\linewidth]{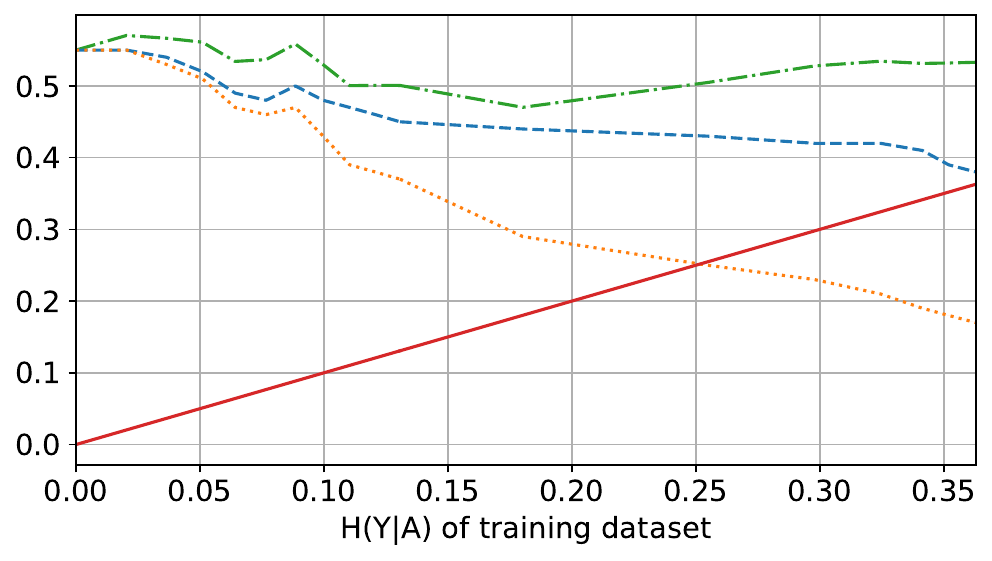}\label{fig:CelebA_CSAD_bound}}\quad
      \subfloat[BCL~\cite{BCL}.]{\includegraphics[width=0.42\linewidth]{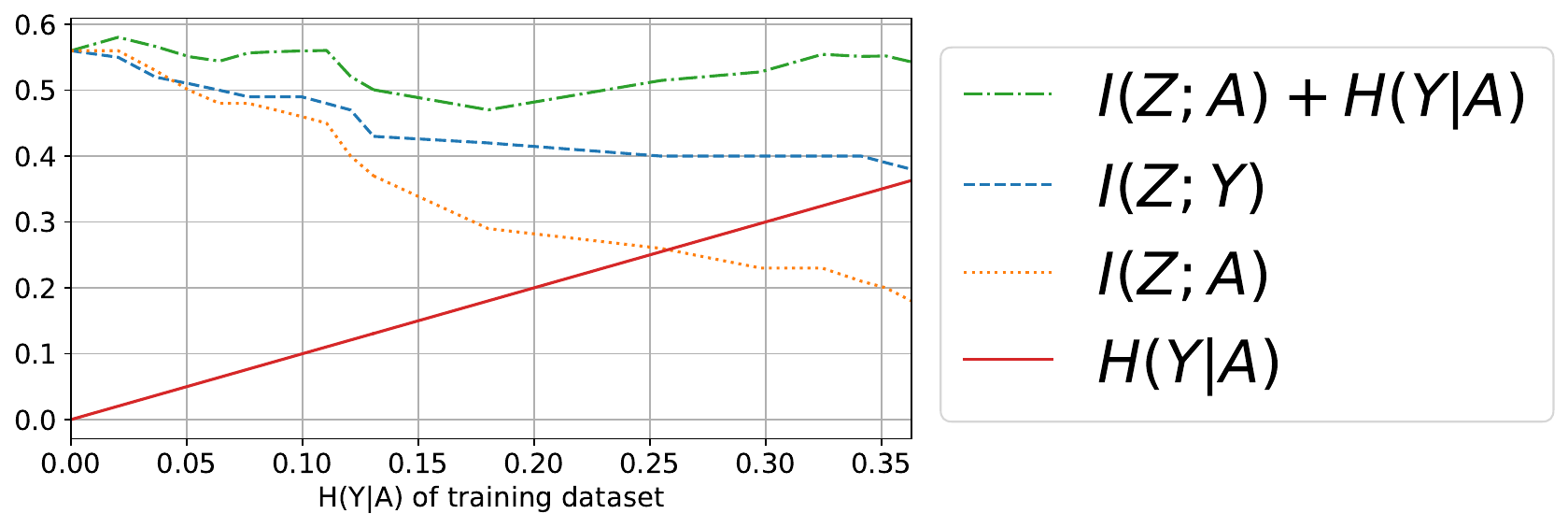}\label{fig:CelebA_BCL_bound}}
\caption{Empirically verifying the bound in~\cref{th:general} for several bias removal methods trained on CelebA. 
The x-axis shows $H(Y|A)$, which we vary directly by adjusting the fraction of bias-conflicting images while ensuring a constant number of biased images in the training set. We empirically compute $H(Y|A)$ based on the distribution of $Y$ and $A$ in the modified training set, and estimate mutual information using~\cite{MINE}. The bound $0 \leq I(Z;Y) \leq I(Z;A) + H(Y|A)$ holds for all methods (results of additional bias removal methods are provided in~\cref{appsec:bounds}).
}
\label{fig:CelebA_bounds}
\end{figure*}

%% file: sections/pictures/fig_idea.tex
\begin{figure*}[htbp]
\begin{center}
  \includegraphics[width=0.8\linewidth]{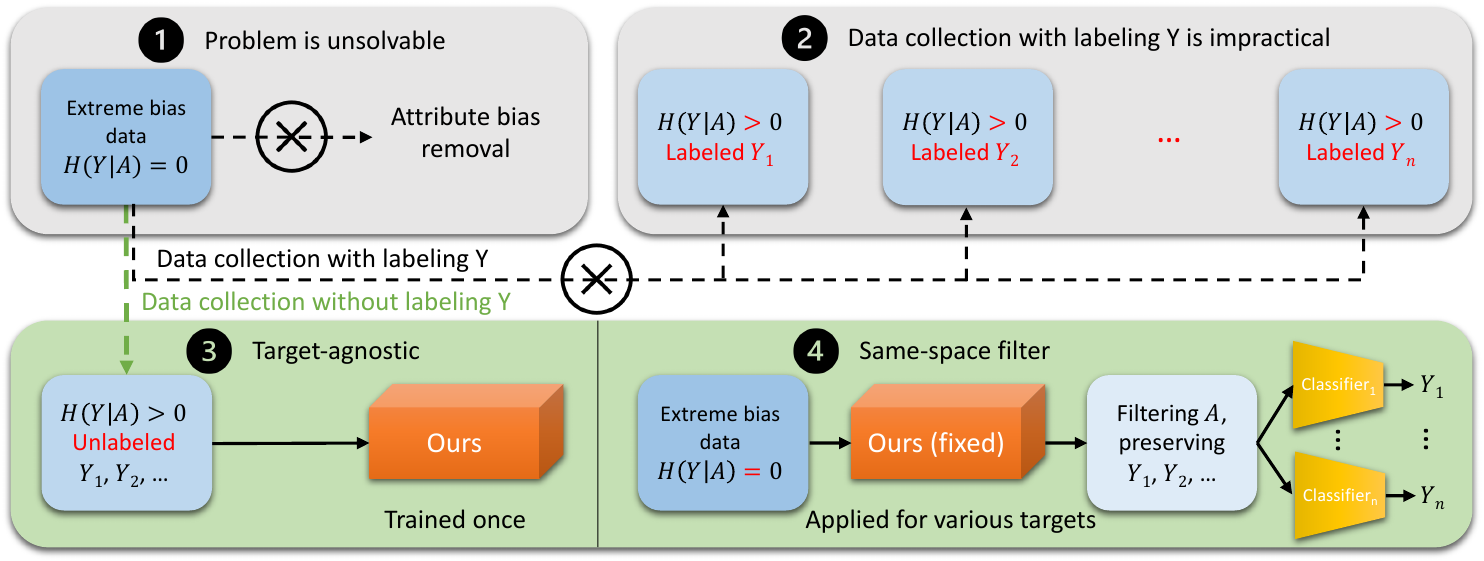}
\end{center}
  \caption{Illustration of extreme bias and the proposed method.
  (1) In extreme bias where $H(Y|A)=0$, no effective attribute bias removal method exists unless it can access a \Univ{} distribution where $H(Y|A)>0$.
(2) It is impractical to collect samples from a \Univ{} distribution with target labels for all potential downstream tasks.
(3) Thus, we propose a target-agnostic method that can utilize a \Univ{} distribution without target labels, \ie a partially observable distribution.
(4) Due to its same-space design, our method can be easily applied as preprocessing in various downstream tasks for removing attribute bias.}
\label{fig:idea}
\end{figure*}

%% file: sections/05_necessary_assumption.tex
\label{sec:necessary}

According to~\cref{th:general}, if a dataset has extreme bias ($H(Y|A)=0$), then the best performance of any attribute bias removal method in learning the latent feature $Z_\theta$ becomes bounded by the amount of attribute bias that remains in the learnt latent feature, \ie $I(Z_\theta;Y) \leq I(Z_\theta;A)$. Therefore, the more attribute bias the method removes, the lower the best performance on predicting the target from learnt feature becomes. In this section, we will derive a necessary condition under which it is possible to avoid this trade-off. Consider again the random variables $X$ (input), $Y$ (target), and $A$ (protected attribute), as defined in~\cref{sec:theory}, with the respective distributions $p_X(x)$, $p_Y(y)$, and $p_A(a)$. Note that while the observed joint distribution $p(x,y,a)$ over these random variables in a given dataset can be such that $H_{p}(Y|A)=0$, \ie having extreme bias, this is not necessarily the only observable joint distribution over these random variables. In other words, there could exist another joint distribution $q(x,y,a)$ over the same three random variables (with the correct marginal distributions) in which $H_{q}(Y|A)>0$, which we denote the \textit{\Univ{} distribution}. If such a distribution exists -- even if only partially observable -- it could help mitigate the limitation in removing extreme bias in the observed distribution. The following corollary of~\cref{th:general} shows that the existence of a \Univ{} distribution is necessary for the existence of a successful attribute bias removal method:
\begin{definition}
\label{th:def} (\Univ{} Distribution). $q: \mathcal{X}\times\mathcal{Y}\times\mathcal{A} \rightarrow \mathbb{R}^{\geq 0}$ such that all the following conditions hold\footnote{We consider all random variables to be discrete, as these are represented by the finite set of floating point numbers in practice.}:
\begin{enumerate}
% \begin{enumerate}
    \item $\sum_{x,y,a}q(x,y,a) = 1$
    \item $\sum_{y,a}q(x,y,a) = p_X(x)$
    \item $\sum_{x,a}q(x,y,a) = p_Y(y)$
    \item $\sum_{x,y}q(x,y,a) = p_A(a)$
    \item $H_q(Y|A) > 0$
\end{enumerate}
\end{definition}

\begin{corollary}
\label{th:necessary} (Necessary Condition). Consider any family of bias removal methods $\Theta$, then there exists a method $\phi \in \Theta$ that simultaneously removes the bias and achieves the best performance, \ie $\phi = \argmin_{\theta \in \Theta} I(Z_{\theta};A) = \argmax_{\theta \in \Theta} I(Z_{\theta};Y)$ only if $~\exists~q(x,y,a):  H_{q}(Y|A) > 0$. 
\end{corollary}

\begin{proof}
By contraposition:
from~\cref{th:general}, we know that when $H(Y|A)=0$, $0 \leq I(Z;Y) \leq I(Z;A)$. Thus, for any method $\phi\in\Theta$ that achieves $I(Z_\phi;A) = \min_{\theta \in \Theta} I(Z_{\theta};A) = 0$, we have $I(Z_\phi;Y) = 0 \neq \argmax_{\theta \in \Theta}I(Z_{\theta};Y)$.
\end{proof}

The existence of a \Univ{} distribution is essentially formalizing the knowledge that the two concepts $A$ and $Y$ are not exactly the same, \ie there exists a distribution where they can be distinguished. However, note that~\cref{th:necessary} does not require this distribution to be fully observable to break the trade-off between performance and bias removal. Therefore, assuming this dataset exists (\ie we can collect samples of $X$ from it), we consider three possibilities regarding the observability of target $Y$ and protected attribute $A$: 1)~\textit{Fully~Observable}: where we can collect both target and protected attribute labels, in which case the existing methods can directly use this distribution to remove bias; 2)~\textit{Partially~Observable}: where we cannot collect target labels, in which case the existing methods cannot use this dataset since they require both target and attribute labels; and 3)~\textit{Non~Observable}: where we can collect neither target nor attribute labels, in which case the existing methods cannot be used.

In practice, as we will show in~\cref{sec:exp}, one can collect samples of $X$ from a \Univ{} distribution by either directly using samples from large-scale web-scraped datasets, or using pretrained generative models. However, in both cases, collecting target labels is particularly expensive, since there are many downstream tasks each requiring its own target label; additionally, we have to relabel a large-scale dataset for every new downstream task that emerges over time. In contrast, collecting protected attribute labels are more feasible since there are only a small number of protected attributes, and once the labels are collected, they can be used with any downstream task\footnote{We further discuss the feasibility of collecting protected attribute labels in~\cref{appsec:feasibility}.}. This motivates the development of attribute bias removal methods that only require access to attribute labels, \ie can utilize a partially observable \Univ{} distribution, which we will consider in the next section. 
We will also consider developing methods that can utilize a non-observable \Univ{} distribution in~\cref{appsec:SSL}.

%% file: sections/pictures/framework+visualization.tex
\begin{figure*}
  \centering
  \includegraphics[width=0.95\linewidth]{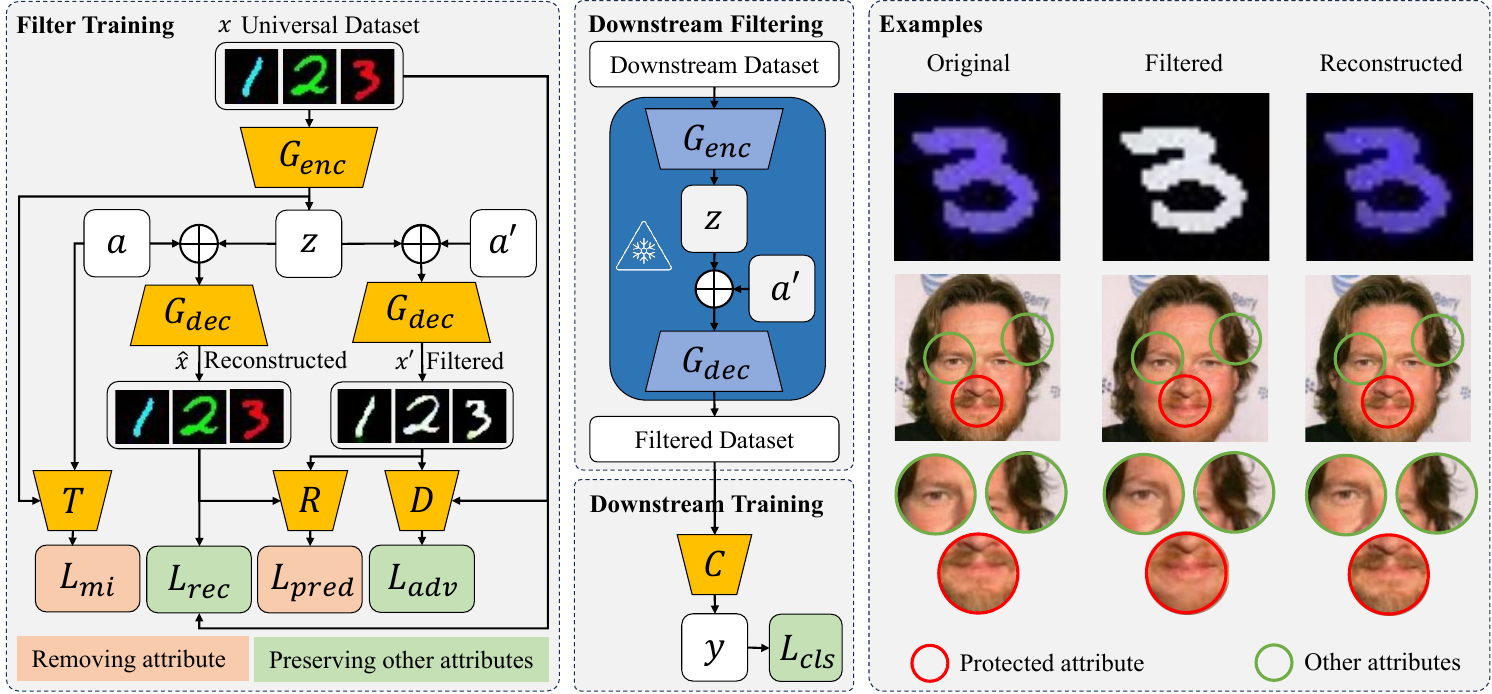}
  \caption{Summary of the proposed method. (Left) shows the training mechanism of the filter ($G_{dec} \circ G_{enc}$) with samples from a~\Univ{} distribution, where the protected attribute is removed and other attributes are preserved. (Middle) shows the use of the filter in a downstream task, where the frozen pretrained filter is first used to remove the protected attribute from the downstream dataset (top), and then the resulting filtered dataset is used to train a classifier $C$ with cross-entropy loss $L_{cls}$ (bottom). (Right) shows the application of the trained filtered on two examples from Colored MNIST and CelebA.}
  \label{fig:method}
\end{figure*}

%% file: sections/06_method.tex
\label{sec:method}

As discussed in~\cref{sec:necessary}, having access to a \Univ{} distribution can help mitigate the limitation of existing methods in removing strong attribute bias. However, when this distribution is only partially observable, \ie target labels are missing and only protected attribute labels are available, the existing bias removal methods cannot directly utilize the \Univ{} distribution. In this section, we aim to propose a method that can utilize a partially observable \Univ{} distribution to better remove strong attribute bias. The intuition behind our method is to develop a same-space filter that removes a protected attribute, while maximally preserving all other attributes, thus being target-agnostic. Being a same-space filter (\ie transforming images into images) allows various downstream tasks to use it as a simple preprocessing mechanism to remove protected attributes without having to adapt their architecture to accept inputs from a particular latent space.

\cref{fig:method} provides a summary of our proposed method. We will assume the inputs are images in explaining our proposed method, which can be readily generalized to other modalities, as shown in the census experiments in~\cref{sec:exp}. Given an image $x\in\mathcal{X}$ with protected attribute $a\in\mathcal{A}$, we use an encoder $G_{enc}: \mathcal{X} \rightarrow \mathcal{Z}$ to map $x$ to latent representation $z = G_{enc}(x)$, and an attribute-conditioned decoder $G_{dec}: \mathcal{Z}\times\mathcal{A} \rightarrow \mathcal{X}$ to reconstruct the original image $\hat{x} = G_{dec}(z, a)$ and produce a corresponding filtered image $x' = G_{dec}(z, a')$, where $a'\in\mathcal{A}$ is a constant value for all input images, representing a \textit{neutral value} of the attribute. For example, in Colored MNIST, we choose $a'$ to be a constant RGB color, and in the case of the discrete attribute in CelebA, $a'$ is the uniform categorical distribution. The target of our optimization is $x'$, where we need protected attribute information to be removed (\ie constant attribute), while all other information about $x$ is preserved.

\noindent
\textbf{Removing Protected Attribute.}
First, to be able to generate the filtered image $x'$ by swapping the attribute $a$ with its neutral value $a'$, we need to ensure the representation $z$ and $a$ are disentangled. To achieve this, we minimize the mutual information between their corresponding random variables $Z$ and $A$, where we adopt MINE~\cite{MINE} and use an auxiliary neural network $T: \mathcal{Z}\times\mathcal{A} \rightarrow \mathbb{R}$ for estimating $I(Z;A)$:
\begin{align}
\label{eq:mi}
	\mathcal{L}^G_{mi} =\max_T \mathbb{E}_{Z,A} T(z,a) - \log\mathbb{E}_{{Z \otimes A}} e^{T(z,a)}
\end{align}
\noindent
where $\mathbb{E}_{Z, A}$ and $\mathbb{E}_{{Z \otimes A}}$ represent the joint and product distribution of latent features and attributes, respectively.
Next, to ensure that the reconstruction $\hat{x}$ and the filtered image $x'$ contain the respective attributes, $a$ and $a'$, we introduce a regressor $R: \mathcal{X} \rightarrow \mathcal{A}$ trained to achieve:
\begin{align}
\label{eq:reg}
    R^* = \argmin_R \mathcal{L}_{reg}(R(\hat{x}), a)
\end{align}
where $\mathcal{L}_{reg}$ is an appropriate regression loss (L2 loss for continuous attributes and cross-entropy loss for discrete attributes). Then, the following loss ensures the generated images contain the respective attributes:
\begin{align}
    \mathcal{L}^G_{pred} = \mathcal{L}_{reg}(R^*(\hat{x}), a) + \mathcal{L}_{reg}(R^*(x'), a')
\end{align}
 
\noindent
\textbf{Preserving Other Attributes.}
To ensure the minimal loss of information and preserve other attributes in the filtered image $x^{\prime}$, we introduce two reconstruction losses inspired by~\cite{attgan}. First, an $L_1$ reconstruction loss on the reconstructed image:
\begin{align}
\label{rec:rec}
	\mathcal{L}^G_{rec} = \mathbb{E}_{X,\hat{X}}\norm{x - \hat{x}}_1
\end{align}

\noindent
Second, since $L_1$ reconstruction is too strict on the filtered image $x'$, we introduce an adversarial loss for matching it with the original image $x$. We follow WGAN~\cite{wgan} with a critic neural network $D: \mathcal{X} \rightarrow \mathbb{R}$ for the following loss:
\begin{align}
\label{eq:adv}
    \mathcal{L}^G_{adv} = \max_{\norm{D}_L = 1} \mathbb{E}_{X} D(x) - \mathbb{E}_{X'} D(x')
\end{align}
\noindent
where the Lipschitz constraint $\norm{D}_L = 1$ on $D$ is enforced through gradient penalty~\cite{wgan_gp}. 
 
% 3. Summary
\noindent
\textbf{Overall.} The overall loss that is minimized over $G:\{G_{enc}, G_{dec}\}$ to train the filter is as follows:
\begin{align}
\label{eq:g_loss}
	\mathcal{L}^{G}_{total} = \mathcal{L}^{G}_{adv} + \lambda_{mi} \mathcal{L}^G_{mi} + \lambda_{pred} \mathcal{L}^G_{pred} + \lambda_{rec} \mathcal{L}^G_{rec}
\end{align}

\noindent
where $\lambda_{mi}$, $\lambda_{pred}$ and $\lambda_{rec}$ are hyper-parameters that balance losses to optimize $G$. In practice, we alternate between optimizing $T, R, D$ under \cref{eq:mi,eq:reg,eq:adv}, respectively, and optimizing $G$ under \cref{eq:g_loss} every other step. After training, the filter $G_{dec}\circ G_{enc}: \mathcal{X} \rightarrow \mathcal{X}$ is directly applied in various downstream tasks to remove the protected attribute.\footnote{\label{fnote_NN}Details of our neutral networks are provided in~\cref{appsec:training_details}.}
In all experiments, we use a two-stage training scheme: First, we train our filter on either the biased training set itself or samples from a \Univ{} distribution (when its availability is assumed); second, we apply the filter to the biased training set and then train the baseline neural network classifier on the filtered samples with cross-entropy loss. 
We provide several ablation studies on the hyperparameters of our method in~\cref{subsec:abalation}.

%% file: sections/tables/CelebA_EB_BlondHair.tex
\begin{table*}[htbp]
\caption{Performance of attribute bias removal methods under extreme bias in CelebA dataset (\emph{TrainEx} training set) to predict \textit{blond hair}. $\Delta$ indicates the difference from baseline, and \textbf{Bold} highlights best results. For our method, we report inside parentheses the partially observable \Univ{} distribution used in addition to \emph{TrainEx} for training its filter. Without a \Univ{} distribution, none of the methods can effectively remove the bias $I(Z;A)$ compared to baseline.
}
\label{tab:CelebA_BlondHair}
\centering
\resizebox{0.66\textwidth}{!}{%

\begin{tabular}{lcccc}
\toprule
\multirow{2}{*}{Method}          & \multicolumn{2}{c}{Test Accuracy}         & \multicolumn{2}{c}{Mutual Information} \\
\cmidrule(lr){2-3}  \cmidrule(lr){4-5} 
                                 & Unbiased ↑          & Bias-conflicting ↑  & $I(Z;A)$ ↓          & $\Delta$ (\%) ↑  \\
                                 \midrule
                             Random guess                          & 50.00  & 50.00 & 0.57 & 0.00             \\    
Baseline                          & 66.11{\scriptsize $\pm$0.32} & 33.89{\scriptsize $\pm$0.45} & 0.57{\scriptsize $\pm$0.01} & 0.00             \\
\midrule
LNL~\cite{learn_not_to_learn_Colored_MNIST}                               & 64.81{\scriptsize $\pm$0.17} & 29.72{\scriptsize $\pm$0.26} & 0.56{\scriptsize $\pm$0.06} & 1.75             \\
DI~\cite{domain_independent_training}                                & 66.83{\scriptsize $\pm$0.44} & 33.94{\scriptsize $\pm$0.65} & 0.55{\scriptsize $\pm$0.02} & 3.51             \\
LfF~\cite{LfF_CelebA_Bias_conflicting}                               & 64.43{\scriptsize $\pm$0.43} & 30.45{\scriptsize $\pm$1.63} & 0.57{\scriptsize $\pm$0.03} & 0.00             \\
EnD~\cite{End}                               & 66.53{\scriptsize $\pm$0.23} & 31.34{\scriptsize $\pm$0.89} & 0.57{\scriptsize $\pm$0.05} & 0.00             \\
CSAD~\cite{CSAD}                              & 63.24{\scriptsize $\pm$2.36} & 29.13{\scriptsize $\pm$1.26} & 0.55{\scriptsize $\pm$0.04} & 3.51             \\
BCL~\cite{BCL}                               & 65.30{\scriptsize $\pm$0.51} & 33.44{\scriptsize $\pm$1.31} & 0.56{\scriptsize $\pm$0.07} & 1.75             \\
% \textbf{SSL}                      & 64.24{\scriptsize $\pm$0.54} & 32.59{\scriptsize $\pm$0.61} & 0.56{\scriptsize $\pm$0.02} & 1.75             \\
Ours                     & 66.31{\scriptsize $\pm$0.26} & 32.22{\scriptsize $\pm$0.43} & 0.55{\scriptsize $\pm$0.01} & 3.51            \\
\midrule
% \textbf{SSL (FFHQ) $\star$}       & 69.02{\scriptsize $\pm$0.47} & 42.75{\scriptsize $\pm$0.83} & 0.51{\scriptsize $\pm$0.02} & 10.50            \\
% \textbf{SSL (Synthetic) $\star$}  & 70.19{\scriptsize $\pm$0.58} & 44.23{\scriptsize $\pm$0.92} & 0.50{\scriptsize $\pm$0.02} & 12.28            \\
Ours (FFHQ)       & \textbf{71.53{\scriptsize $\pm$0.67}} & 47.17{\scriptsize $\pm$0.72} & 0.47{\scriptsize $\pm$0.01} & 17.54            \\
Ours (Synthetic)  & 71.37{\scriptsize $\pm$0.64} & \textbf{48.06{\scriptsize $\pm$0.82}} & \textbf{0.45{\scriptsize $\pm$0.01}}  & \textbf{21.05}  \\
\bottomrule
\end{tabular}

}%
\end{table*}

% \toprule
% \cmidrule(lr){2-3}  \cmidrule(lr){4-5} 
% \midrule
% \bottomrule

%% file: sections/tables/Adult_EB.tex
\begin{table*}[t]
\caption{Performance of attribute bias removal methods under extreme bias in Adult dataset (\emph{TrainEx} training set) to predict \textit{income}. $\Delta$ indicates the difference from baseline, and \textbf{Bold} highlights best results. For our method, we report inside parentheses the partially observable \Univ{} distribution used in addition to \emph{TrainEx} for training its filter. Without a \Univ{} distribution, none of the methods can effectively remove the bias $I(Z;A)$ compared to baseline.
}
       
        \label{tab:Adult_mostEx}
        \centering
        \resizebox{0.66\textwidth}{!}{%
\begin{tabular}{lcccc}
\toprule
\multirow{2}{*}{Method}  & \multicolumn{2}{c}{Test Accuracy}         & \multicolumn{2}{c}{Mutual Information} \\
\cmidrule(lr){2-3}  \cmidrule(lr){4-5} 
                         & Unbiased ↑          & Bias-conflicting ↑  & $I(Z;A)$ ↓       & $\Delta$ (\%) ↑      \\
                         \midrule
Random guess             & 50.00                  & 50.00                   & 0.69             & 0.00                \\
Baseline                 & 50.59{\scriptsize $\pm$0.54}          & 1.19{\scriptsize $\pm$0.83}           & 0.69{\scriptsize $\pm$0.00}        & 0.00                \\
\midrule
% DR                       & 66.38+0.22          & 38.17+1.04          & -1.00            & -1.00               \\
LNL~\cite{learn_not_to_learn_Colored_MNIST}                      & 50.10{\scriptsize $\pm$0.18}          & 0.43{\scriptsize $\pm$0.46}           & 0.69{\scriptsize $\pm$0.01}        & 0.00                \\
DI~\cite{domain_independent_training}                       & 50.61{\scriptsize $\pm$0.28}          & 0.65{\scriptsize $\pm$0.64}           & 0.69{\scriptsize $\pm$0.01}        & 0.00                \\
LfF~\cite{LfF_CelebA_Bias_conflicting}                      & 50.33{\scriptsize $\pm$0.34}          & 0.78{\scriptsize $\pm$0.65}           & 0.69{\scriptsize $\pm$0.01}        & 0.00                \\
EnD~\cite{End}                      & 50.59{\scriptsize $\pm$0.75}          & 1.18{\scriptsize $\pm$0.96}           & 0.69{\scriptsize $\pm$0.00}        & 0.00                \\
CSAD~\cite{CSAD}                     & 50.76{\scriptsize $\pm$2.22}          & 1.43{\scriptsize $\pm$2.46}           & 0.69{\scriptsize $\pm$0.01}        & 0.00                \\
BCL~\cite{BCL}                      & 50.83{\scriptsize $\pm$1.34}          & 0.52{\scriptsize $\pm$0.83}           & 0.69{\scriptsize $\pm$0.00}        & 0.00                \\
Ours                     & 50.09{\scriptsize $\pm$0.81}          & 0.64{\scriptsize $\pm$1.01}           & 0.69{\scriptsize $\pm$0.01}            & 0.00            \\
\midrule
Ours (Universal) & \textbf{74.93{\scriptsize $\pm$0.95}} & \textbf{57.63{\scriptsize $\pm$1.30}} & \textbf{0.45{\scriptsize $\pm$0.00}}   & \textbf{34.78}  \\
\bottomrule
\end{tabular}
        }
\end{table*}

%% file: sections/pictures/CelebA_EB.tex
\begin{figure*}[b!]

    \begin{minipage}{1\textwidth}
        \captionsetup{type=table}
        \caption{Area under the curve (AUC) in the strong bias region of CelebA dataset.}
        \label{tab:CelebA_AUC}
        \centering
\begin{tabular}{lcccc}
\toprule
\multirow{2}{*}{Method} & \multicolumn{2}{c}{AUC of Test Accuracy}  & \multicolumn{2}{c}{AUC of Mutual Information} \\
\cmidrule(lr){2-3}  \cmidrule(lr){4-5} 
                        & Unbiased ↑          & Bias-conflicting ↑  & $I(Z;A)$ ↓              & $\Delta$ (\%) ↑     \\
                        \midrule
                        Random guess & 17.50 & 17.50 & 0.15 & 0.00 \\
Baseline                & 24.67{\scriptsize $\pm$0.72}          & 17.18{\scriptsize $\pm$1.62}          & 0.15{\scriptsize $\pm$0.01}               & 0.00                \\
\midrule
LNL~\cite{learn_not_to_learn_Colored_MNIST}                     & 26.81{\scriptsize $\pm$0.97}          & 21.58{\scriptsize $\pm$0.95}          & 0.12{\scriptsize $\pm$0.03}               & 20.00               \\
DI~\cite{domain_independent_training}                      & 27.53{\scriptsize $\pm$0.92}          & 23.81{\scriptsize $\pm$0.76}          & 0.12{\scriptsize $\pm$0.01}               & 20.00               \\
LfF~\cite{LfF_CelebA_Bias_conflicting}                     & 26.79{\scriptsize $\pm$1.16}          & 23.78{\scriptsize $\pm$1.24}          & 0.11{\scriptsize $\pm$0.01}               & 26.67               \\
EnD~\cite{End}                     & 27.31{\scriptsize $\pm$0.96}          & 21.42{\scriptsize $\pm$0.88}          & 0.12{\scriptsize $\pm$0.03}               & 20.00               \\
CSAD~\cite{CSAD}                    & 27.43{\scriptsize $\pm$1.57}          & 22.06{\scriptsize $\pm$0.97}          & 0.12{\scriptsize $\pm$0.02}               & 20.00               \\
BCL~\cite{BCL}                     & 27.82{\scriptsize $\pm$0.66}          & 23.53{\scriptsize $\pm$1.32}          & 0.12{\scriptsize $\pm$0.03}               & 20.00               \\
\midrule
Ours                    & 28.90{\scriptsize $\pm$0.94}          & 24.61{\scriptsize $\pm$0.79}          & 0.11{\scriptsize $\pm$0.01}               & 26.67               \\
Ours (FFHQ)             & \textbf{30.29{\scriptsize $\pm$0.68}} & 25.83{\scriptsize $\pm$1.00}          & \textbf{0.10{\scriptsize $\pm$0.01}}      & \textbf{33.33}      \\
Ours (Synthetic)        & 30.20{\scriptsize $\pm$0.85}          & \textbf{26.04{\scriptsize $\pm$1.22}} & \textbf{0.10{\scriptsize $\pm$0.01}}      & \textbf{33.33}     \\
\bottomrule
\end{tabular}
\end{minipage}

    % \vspace{0.3cm}

    \begin{minipage}{1\textwidth}
        \centering
      \subfloat[\emph{Unbiased} testing set.]{\includegraphics[width=0.39\linewidth]{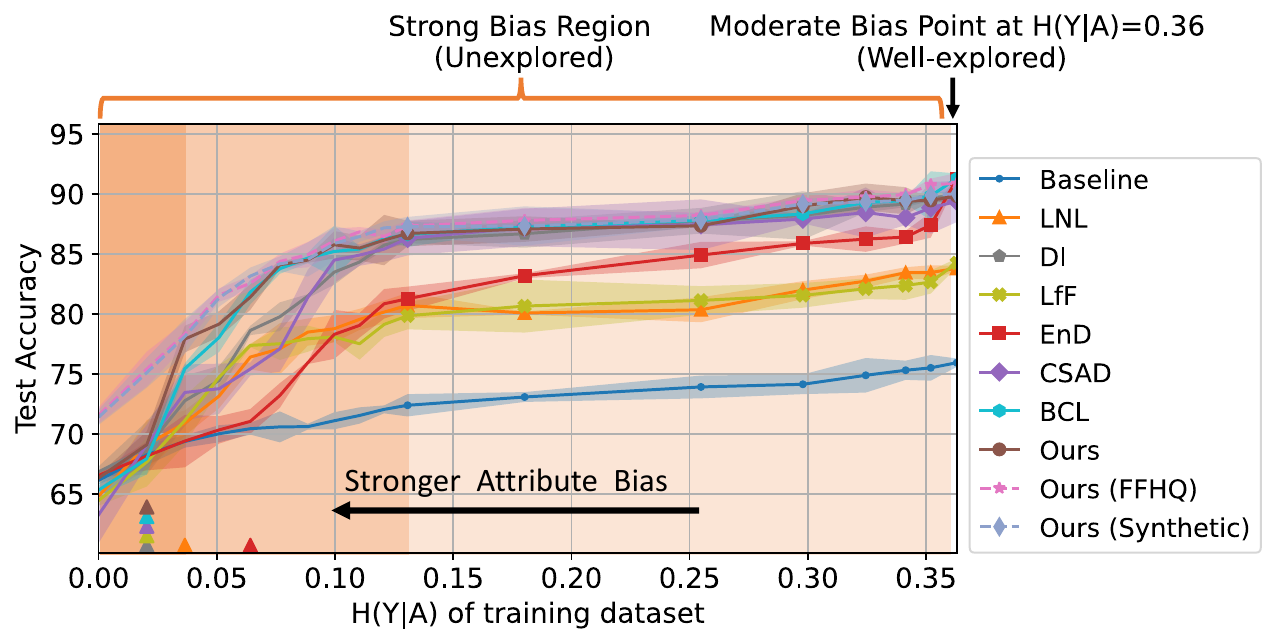}\label{fig:CelebA_Unbiased_Acc}}
      \subfloat[\emph{Bias-conflicting} testing set.]{\includegraphics[width=0.3\linewidth]{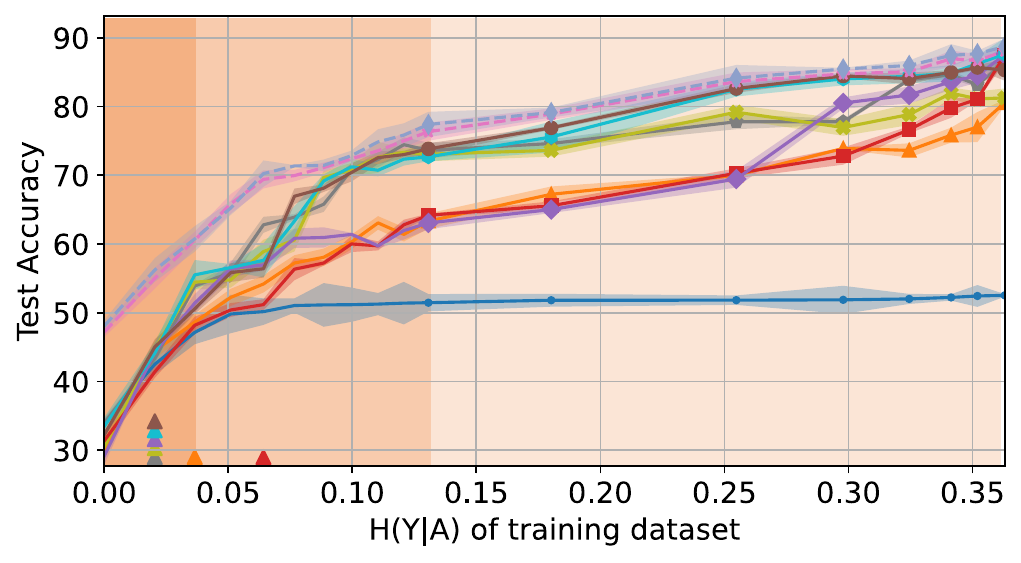}\label{fig:CelebA_conflict_Acc}}
      \subfloat[Attribute bias $I(Z;A)$.]{\includegraphics[width=0.3\linewidth]{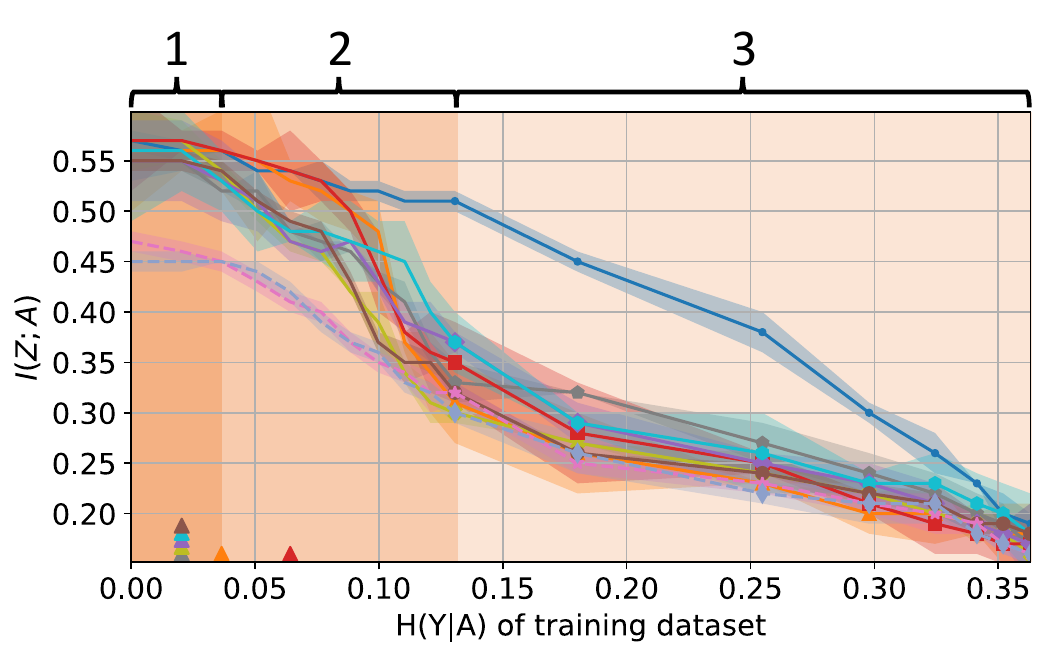}\label{fig:CelebA_IZA}}\quad
  
  \caption{Accuracy and mutual information under different bias strengths in CelebA. As bias strength increases (moving from right to left), the performance of all methods degrades and sharply declines to baseline at the breaking point (labeled by $\blacktriangle$).}
  \label{fig:EB_CelebA}
    \end{minipage}%

\end{figure*}

%% file: sections/pictures/Adult_EB.tex
\begin{figure*}[htbp]

    \begin{minipage}{1\textwidth}
        \captionsetup{type=table}
        \caption{Area under the curve (AUC) in the strong bias region of Adult dataset.}
        \label{tab:Adult_AUC}
        \centering
\begin{tabular}{lcccc}
\toprule
\multirow{2}{*}{Method} & \multicolumn{2}{c}{AUC of Test Accuracy}  & \multicolumn{2}{c}{AUC of Mutual Information} \\
\cmidrule(lr){2-3}  \cmidrule(lr){4-5} 
                        & Unbiased ↑          & Bias-conflicting ↑  & $I(Z;A)$ ↓              & $\Delta$ (\%) ↑      \\
\midrule
Random guess & 34.00 & 34.00 & 0.38 & 0.00 \\
Baseline                & 46.36{\scriptsize $\pm$1.54}          & 27.38{\scriptsize $\pm$4.64}          & 0.38{\scriptsize $\pm$0.02}               & 0.00                \\
\midrule
LNL~\cite{learn_not_to_learn_Colored_MNIST}                     & 48.36{\scriptsize $\pm$0.49}          & 31.41{\scriptsize $\pm$1.25}          & 0.32{\scriptsize $\pm$0.02}               & 15.79               \\
DI~\cite{domain_independent_training}                      & 48.38{\scriptsize $\pm$0.53}          & 31.30{\scriptsize $\pm$1.09}          & 0.34{\scriptsize $\pm$0.01}               & 10.53               \\
LfF~\cite{LfF_CelebA_Bias_conflicting}                     & 48.57{\scriptsize $\pm$0.50}          & 31.64{\scriptsize $\pm$1.17}          & 0.33{\scriptsize $\pm$0.02}               & 13.16               \\
EnD~\cite{End}                     & 48.42{\scriptsize $\pm$0.71}          & 30.10{\scriptsize $\pm$1.08}          & 0.32{\scriptsize $\pm$0.02}               & 15.79               \\
CSAD~\cite{CSAD}                    & 47.58{\scriptsize $\pm$0.69}          & 31.11{\scriptsize $\pm$1.54}          & 0.34{\scriptsize $\pm$0.02}               & 10.53               \\
BCL~\cite{BCL}                     & 48.54{\scriptsize $\pm$0.73}          & 31.20{\scriptsize $\pm$1.17}          & 0.33{\scriptsize $\pm$0.01}               & 13.16               \\
\midrule
Ours                    & 50.29{\scriptsize $\pm$0.44}          & 33.63{\scriptsize $\pm$0.99}          & 0.31{\scriptsize $\pm$0.01}               & 18.42               \\
Ours (Universal)        & \textbf{53.54{\scriptsize $\pm$0.59}} & \textbf{45.16{\scriptsize $\pm$1.18}} & \textbf{0.25{\scriptsize $\pm$0.01}}      & \textbf{34.21}     \\
\bottomrule
\end{tabular}
\end{minipage}

    % \vspace{0.3cm}

    \begin{minipage}{1\textwidth}
        \centering
      \subfloat[\emph{Unbiased} testing set.]{\includegraphics[width=0.39\linewidth]{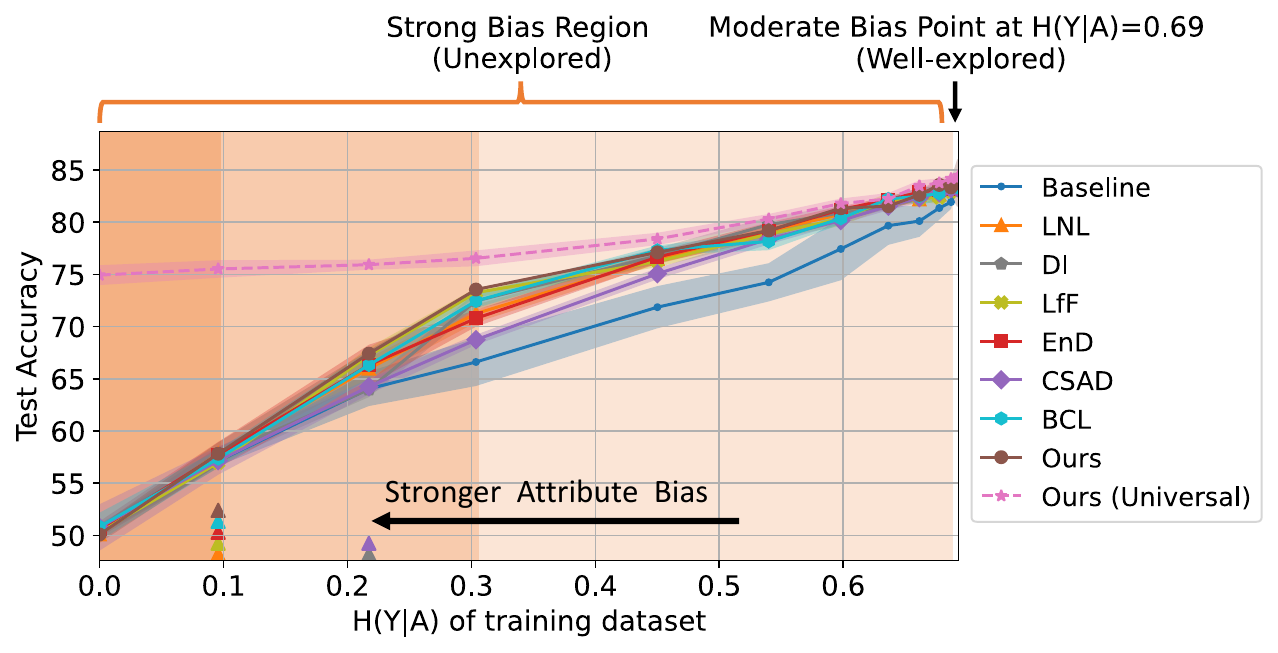}\label{fig:Adult_Unbiased_Acc}}\quad
      \subfloat[\emph{Bias-conflicting} testing set.]{\includegraphics[width=0.285\linewidth]{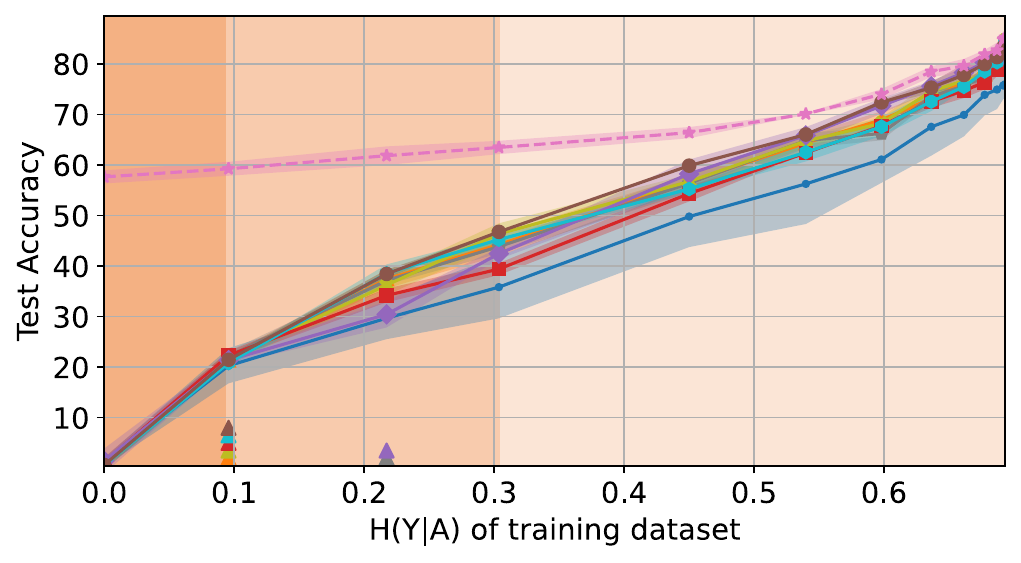}\label{fig:Adult_conflict_Acc}}\quad
      \subfloat[Attribute bias $I(Z;A)$.]{\includegraphics[width=0.285\linewidth]{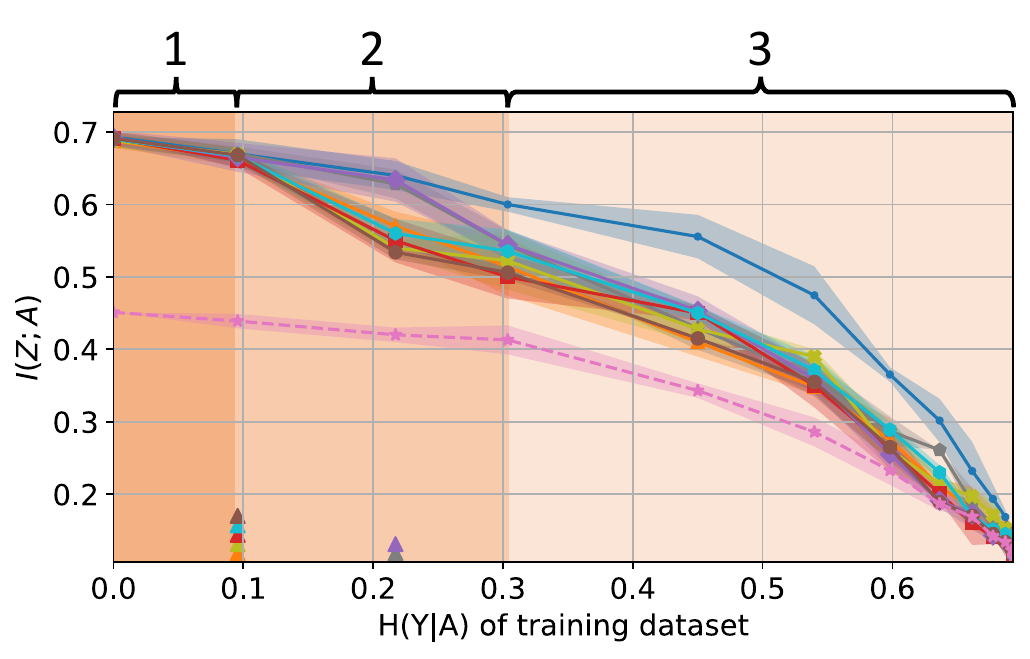}\label{fig:Adult_IZA}}\quad
      
      % \subfloat[\emph{Unbiased} testing set.]{\includegraphics[width=0.35\linewidth]{sections/figures/eb1_balanced-balanced_TestAcc.pdf}\label{fig:Adult_Unbiased_Acc}}\quad
      % \subfloat[\emph{Bias-conflicting} testing set.]{\includegraphics[width=0.26\linewidth]{sections/figures/eb1_balanced-eb2_balanced_TestAcc.pdf}\label{fig:Adult_conflict_Acc}}\quad
      % \subfloat[Attribute bias $I(Z;A)$.]{\includegraphics[width=0.35\linewidth]{sections/figures/eb1_balanced-eb2_balanced_Iza.pdf}\label{fig:Adult_IZA}}\quad

        \caption{Accuracy and mutual information under different bias strengths in Adult. As bias strength increases (moving from right to left), the performance of all methods degrades and sharply declines to baseline at the breaking point (labeled by $\blacktriangle$).}
  \label{fig:EB_Adult}
    \end{minipage}%

\end{figure*}

%% file: sections/07_experiment.tex
\label{sec:exp}

In this section, we conduct experiments with an extensive list of existing state-of-the-art attribute bias removal methods~\cite{learn_not_to_learn_Colored_MNIST,domain_independent_training, LfF_CelebA_Bias_conflicting, End, CSAD, BCL} in Colored MNIST as well as two real-world datasets: CelebA~\cite{CelebA} as an image dataset and Adult~\cite{UCI_ML_Repo} as a census dataset.
We also compare our method with several GAN-based approaches~\cite{CGN,Camel,BiaSwap,GAN_Debiasing_hat_glasses_correlation}. In all experiments, we report average results with one standard deviation over multiple trials (15 trials in Colored MNIST, 5 in CelebA, 25 in Adult). Please see~\cref{appsec:IMDB,appsec:Waterbirds,appsec:modalities} for results on additional datasets (IMDB~\cite{IMDB}, Waterbirds~\cite{DRO}, and CivilComments-WILDS~\cite{CivilComments,Wilds}), and training details. Code for all experiments and models will be publicly released.

\noindent
\textbf{Colored MNIST Dataset}
is an image dataset of handwritten digits, where each digit is assigned a unique RGB color with a certain variance, studied by these methods~\cite{learn_not_to_learn_Colored_MNIST,Back_MI,End,CSAD}. The training set consists of 50000 images and the testing set of 10000 images with uniformly random color assignment. The color is considered the protected attribute $A$ and the digit is the target $Y$. The variance of color in the training set determines the strength of the bias $H(Y|A)$. 
\Univ{} dataset is constructed in a synthetic manner by assigning random colors to digits.
The results on this dataset are reported in~\cref{fig:EB} and explained in~\cref{sec:introduction}.

\noindent
\textbf{CelebA Dataset}~\cite{CelebA} is an image dataset of human faces studied by these methods~\cite{learn_not_to_learn_Colored_MNIST,domain_independent_training,LfF_CelebA_Bias_conflicting,End,CSAD,BCL}. Facial attributes are considered the prediction target $Y$ (\eg blond hair), and sex is the protected attribute $A$. For each target, there is a notion of \textit{biased samples} -- images in which $Y$ is positively correlated with $A$, \eg images of females with blond hair and males without blond hair -- and a notion of \textit{bias-conflicting samples} -- images in which $Y$ is negatively correlated with $A$, \eg images of females without blond hair and males with blond hair. The fraction of bias-conflicting images in the training set determines the strength of the bias $H(Y|A)$. For training, we consider the original training set of CelebA denoted \textit{TrainOri} consisted of 162770 images with $H(Y|A)=0.36$, and an extreme bias version in which the bias-conflicting samples are removed from the original training set denoted \textit{TrainEx} consisted of 89754 images with $H(Y|A)=0$. Additionally, we construct 16 training sets between TrainOri and TrainEx by maintaining the number of biased samples and varying the fraction of bias-conflicting samples. For testing, we consider two versions of the original testing set: 1) \emph{Unbiased} consists of 720 images in which all pairs of target and protected attribute labels have the same number of samples, and (2) \emph{Bias-conflicting} consists of 360 images in which biased samples are excluded from the \emph{Unbiased} dataset (only bias-conflicting samples remain). 
We consider two choices for a \Univ{} dataset: 1) appending \emph{TrainEx} with the FFHQ dataset~\cite{ffhq}, and 2) appending \emph{TrainEx} with a same-sized synthetic dataset where images are randomly generated using~\cite{CAT}.

\noindent
\textbf{Adult Dataset}~\cite{UCI_ML_Repo} is a census dataset of income which is a well-known fairness benchmark. Income is considered the target $Y$ and sex is the protected attribute $A$. To construct training and testing sets, we follow the setup of CelebA explained above, but we further mitigate the effect of data imbalance and the variation in the total number of training samples. For training, we consider the balanced version of the original training set of Adult denoted \textit{TrainOri} consisted of 7076 records with $H(Y|A)=0.69$, and an extreme bias version in which the bias-conflicting samples are removed from TrainOri and the same number of biased samples are appended denoted \textit{TrainEx} with $H(Y|A)=0$ consisted of the same total number (7076) of records as TrainOri. Additionally, we construct 11 training sets in between TrainOri and TrainEx by varying the fraction of biased samples in TrainEx while maintaining the total size of training set. For testing, we consider two versions of the original testing set: 1) \emph{Unbiased} consists of 7076 records in which all pairs of target and protected attribute labels have the same number of samples, and (2) \emph{Bias-conflicting} consists of 3538 records in which biased samples are excluded from the \emph{Unbiased} dataset (only bias-conflicting samples remain). We utilize TrainOri training set, excluding target labels, as a \Univ{} dataset.

\subsection{Analysis of the Extreme Bias Point $H(Y|A)=0$}
\label{subsec:comp_extreme}
In this section, we investigate the consequences of applying attribute bias removal methods at the extreme bias point $H(Y|A)=0$.
We study two aspects of each method, its classification performance (measured by accuracy on Unbiased and Bias-conflicting settings) and its ability to remove bias (measured by estimating $I(Z;A)$ using~\cite{MINE} on the training set). Ideally, a method must achieve on-par or better accuracy than the baseline while learning a representation $Z$ that does not reflect the attribute bias present in the training set, hence successfully removing the bias, \ie $I(Z;A)=0$.
However, as shown in~\cref{tab:CelebA_BlondHair,tab:Adult_mostEx}, without a \Univ{} distribution, none of the bias removal methods can significantly reduce the bias $I(Z;A)$ in the extreme bias setting in either CelebA or Adult datasets. These observations are explained by~\cref{th:extreme} which states that maintaining classification performance above random guess while achieving $I(Z;A)=0$ at $H(Y|A)=0$ is impossible. Note that the methods achieve better than random accuracy because they do not completely remove the bias.

When given access to a \Univ{} distribution, we observe that our proposed method can now significantly improve the performance and the amount of removed bias in both synthetic (Colored MNIST in~\cref{fig:EB}) and real-world datasets (CelebA and Adult in~\cref{tab:CelebA_BlondHair,tab:Adult_mostEx}). Note that none of the existing methods can directly utilize the access to the \Univ{} distribution due to the lack of target labels which is required by these methods. Nonetheless, it is possible to enable all methods to utilize the additional distribution by using pseudo-labeling, which we will explore in~\cref{subsec:all_universal}.

\subsection{Analysis of the Strong Bias Region $H(Y|A)>0$}
\label{subsec:comp_strong}
% Colored MNIST, CelebA
In this section, we go beyond the extreme bias point, and more generally investigate the consequences of applying bias removal methods on the entire range of bias strength, \ie connecting the extreme bias training setting (TrainEx) we studied in~\cref{subsec:comp_extreme} to the moderate bias in the original training setting (TrainOri) commonly studied in existing methods. We again study two aspects of each method, its classification performance (measured by accuracy on Unbiased and Bias-conflicting settings) and its ability to remove bias (measured by estimating $I(Z;A)$ using~\cite{MINE} on the training set).

Without access to a ~\Univ{} distribution, in~\cref{fig:EB_CelebA,fig:EB_Adult}, we observe a decline in the performance of all methods as the bias becomes stronger, in both CelebA and Adult datasets, similar to our prior observation in Colored MNIST in~\cref{fig:EB}. This observation is consistent with~\cref{th:general}, which states that the bias strength determines an upper bound on the best performance of bias removal methods regardless of the dataset and method. Furthermore, in~\cref{fig:CelebA_IZA,fig:Adult_IZA}, we use breaking points (as defined in~\cref{sec:introduction}) to approximately divide the strong bias region into three phases and explain the observed changes in the performance of methods from the perspective of~\cref{th:general}. In phase 1, as $H(Y|A)$ increases from zero to the breaking point (bias strength decreases), we observe that the attribute bias $I(Z;A)$ is not minimized because of the trade-off between best attainable performance $I(Z;Y)$ and attribute bias removal when bias is very strong: the methods choose to increase accuracy towards the best attainable accuracy $I(Z;Y)$ rather than removing attribute bias (this choice is most likely due to the larger weight on the accuracy term in their objectives). Then, in phase 2, as $H(Y|A)$ increases through the breaking point (bias strength decreases further), the methods start to minimize attribute bias $I(Z;A)$ because the upper bound on best attainable performance $I(Z;Y)$ is now large enough to avoid the trade-off between accuracy and attribute bias removal.
Finally, in phase 3, as $H(Y|A)$ further departs from the breaking point, accuracy gradually approaches its best attainable performance, while attribute bias $I(Z;A)$ is minimized further below that of the baseline because the weaker bias strength now allows the model to distinguish $Y$ from $A$ so that minimizing attribute bias and maximizing accuracy do not compete.

To better quantify the performance and compare different methods across the entire strong bias region, in~\cref{tab:CelebA_AUC,tab:Adult_AUC}, we report the area under the curves in~\cref{fig:EB_CelebA,fig:EB_Adult}, respectively. We observe that our proposed method achieves the best performance in both datasets and in all metrics (accuracy and bias removal). In addition, it achieves better or on-par breaking points with existing methods. The same observation holds in the Colored MNIST dataset in~\cref{fig:EB}. This shows that even though we designed our method to be able to utilize a \Univ{} distribution, it can outperform existing methods even without access to such a dataset as well, suggesting that it can be used as a state-of-the-art bias removal method in all settings. We conjecture that this advantage is because we explicitly encourage the filter to maximally preserve information, whereas in other methods the mutual information minimization can remove any information that is not used by the jointly trained classifier, potentially removing too much information early in training when the classifier is relying on only a few features, thus trapping it in local minima. 

When given access to a \Univ{} distribution, we observe that our proposed method can now significantly improve the performance and the amount of removed bias in both synthetic (Colored MNIST in~\cref{fig:EB}) and real-world datasets (CelebA and Adult in~\cref{tab:CelebA_AUC,tab:Adult_AUC}). Note that none of the existing methods can directly utilize the access to the \Univ{} distribution due to the lack of target labels which is required by these methods. This shows that our proposed method can effectively utilize a partially observable ~\Univ{} distribution to improve attribute bias removal.

\input{sections/tables/CelebA_all_universal}
\input{sections/tables/Adult_all_universal}
% \input{sections/tables/CelebA_all_universal_increment}

%%% PSEUDO LABELS
\subsection{Pseudo-Labeling of the \Univ{} Distribution}
\label{subsec:all_universal}
While existing methods cannot directly utilize a partially observable \Univ{} distribution with missing target labels because they require both the protected attribute and the target labels to compute their objectives, it is possible to convert the partially observable \Univ{} distribution to an approximately fully observable distribution using pseudo labels: labels collected using a pretrained classifier. This enables all methods to utilize the additional data available in \Univ{}. To investigate the effectiveness of pseudo-labeling, we first
train a baseline classifier on the observed biased dataset TrainEx -- ResNet18~\cite{ResNet} for CelebA and a three-layer MLP for Adult -- then we use this trained classifier to label samples of the \Univ{} distribution, and finally provide all methods with the original biased dataset extended with the pseudo-labeled samples of the \Univ{} distribution. The results are reported in~\cref{tab:CelebA_all_universal,tab:Adult_all_universal}. We observe that pseudo-labeling improves the performance of all methods (compared to~\cref{tab:CelebA_BlondHair,tab:Adult_mostEx} respectively), and that our proposed method still achieves the best performance in both CelebA and Adult datasets, showing that our proposed method can be used together with pseudo-labeling to provide orthogonal performance gains. We attribute this to the target-agnostic design of our method, which diminishes the reliance on the quality of pseudo-labels.

\input{sections/tables/CelebA_nonsex}
%%% VARIOUS DOWNSTREAM TASKS
\subsection{Application to Various Downstream Tasks}
In this section, we investigate whether our trained filter can be applied to various downstream target prediction tasks, \ie whether it can in fact maximally preserve information while removing the attribute bias. To this end, in~\cref{tab:CelebA_nonsex}, we report the average performance of our method on all 23 non-sex-related downstream tasks in CelebA, in both the extreme and moderate attribute bias settings (sex is considered the protected attribute). Note that the filtering mechanism in the proposed method is only trained once, and then reused in all downstream tasks without retraining. We observe that our proposed method achieves the best performance, even without access to the \Univ{} distribution. The results for individual tasks are reported in~\cref{appsec:CelebA}. This observation suggests that our proposed method can maintain information regarding all other attributes when removing the protected attribute.

\input{sections/tables/CelebA_GAN_based_BH}

\subsection{Comparison with Generative Dataset Augmentation}
\label{subsec:GANs}

To remove the attribute bias, an alternative to our proposed method of filtering the samples in a biased dataset, is to augment the dataset with attribute-flipped samples. In this section, we investigate how our proposed method performs compared to the state-of-the-art GAN-based methods for attribute flipping~\cite{CGN,Camel,BiaSwap,GAN_Debiasing_hat_glasses_correlation}. These methods differ from our method in two main aspects: 1) similar to MI-based methods, they require both the target and attribute labels to apply attribute flipping, thus cannot utilize a partially observable \Univ{} distribution where target labels are missing; 2) they mitigate bias by augmenting the dataset with attribute-flipped samples (rather than filtering the samples), which requires more augmented samples depending on how many different values the protected attribute can take, for example, in CelebA dataset the protected attribute is binary (sex) so they need to increase the dataset size by a factor of two, whereas in Colored MNIST the protected attribute can take ten RGB colors so they need to increase the dataset size 10 times. In~\cref{tab:CelebA_GAN_BlondHair} we report the performance of GAN-based methods. In the moderate bias setting, our method achieves better average accuracy than GAN-based methods, with and without using \Univ{} distribution. In the extreme bias setting, without access to a \Univ{} distribution, none of the methods can outperform the baseline, consistent with our prior observations in~\cref{tab:CelebA_BlondHair,tab:Adult_mostEx}. Given access to a \Univ{} distribution, our method achieves the best average accuracy. These observations provide further evidence that our method is the most effective overall solution for mitigating attribute bias of various strengths, both with and without access to a \Univ{} distribution.

\input{sections/tables/CelebA_nonsex_sex_hp}

\input{sections/pictures/CMNIST_lambda_visualization}
\input{sections/pictures/CMNIST_EB_only}

\subsection{Ablation Studies}
\label{subsec:abalation}

\noindent
\textbf{Removing Protected Attribute.}
Our method achieves this using the mutual information loss ($\mathcal{L}_{mi}$) and attribute prediction loss ($\mathcal{L}_{pred}$), with weight coefficients $\lambda_{mi}$ and $\lambda_{pred}$, respectively. To qualitatively study the importance of each loss, in~\cref{fig:CMNIST_lambda_visualization}, we train our filter on the Colored MNIST dataset with varying coefficients, and observe that if either coefficient is zero, the color is not successfully removed from the digit, thus both $\mathcal{L}_{mi}$ and $\mathcal{L}_{pred}$ are necessary to eliminate the information of protected attributes. Furthermore, to quantitatively measure the importance of each loss in removing the protected attribute, we first train our filter on the CelebA original training set (TrainOri) with varying coefficients, then use it to filter the dataset, and finally measure the attribute prediction accuracy of the baseline classifier trained on the filtered dataset to predict the protected attribute: the lower the attribute prediction accuracy, the better the attribute bias removal. In~\cref{tab:CelebA_sex_hp}, we observe that increasing the coefficients of these two losses reduces the attribute prediction accuracy, thus improving attribute bias removal. Additionally, we observe that increasing the coefficient of the reconstruction loss ($\mathcal{L}_{rec}$) results in weaker attribute bias removal (higher attribute prediction accuracy). The recommended coefficients used in all experiments are displayed in bold.

\noindent
\textbf{Preserving Other Attributes.}
Our method achieves this using the reconstruction loss ($\mathcal{L}_{rec}$) with weight coefficient $\lambda_{mi}$, and the adversarial loss ($\mathcal{L}_{pred}$) with a constant weight coefficient of 1. To quantitatively measure the importance of the reconstruction loss in preserving other attributes, we first train our filter on the CelebA original training set (TrainOri) with varying coefficients, then use it to filter the dataset, and finally measure the average target prediction accuracy of 23 classifiers for each non-sex-related attribute trained on the filtered dataset to predict the 23 non-sex-related targets in CelebA: the higher the target prediction accuracy, the better preserved the other attributes when removing sex. In~\cref{tab:CelebA_sex_hp}, we observe that with a proper choice of $\lambda_{rec}$ their performance on filtered images is consistent with original images, which indicates all relevant facial attributes are preserved. Additionally, we observe that increasing the coefficients of the bias removal losses $\lambda_{mi}, \lambda_{pred}$ improves the target prediction accuracy; we hypothesize that this is because the classifier trained on a biased dataset might employ the protected attribute (\eg sex) as a proxy during training, leading to lower accuracy on datasets without such correlation (Unbiased), and therefore, upon successful removal of sex-related information, an improvement in non-sex-related attribute classification accuracy is observed in~\cref{tab:CelebA_nonsex_hp}. The recommended coefficients used in all experiments are displayed in bold.

\noindent
\textbf{Bias in Universal Distribution.}
We aim to investigate the sensitivity of our filter training to the amount of attribute bias in the \Univ{} distribution itself, namely $H_{q}(Y|A)$. To that end, we consider the extreme bias setting in Colored MNIST dataset -- where no existing method can outperform the baseline except for our method when using the \Univ{} distribution -- and measure how the performance of our method varies when we gradually increase the strength of attribute bias in the \Univ{} distribution (\ie decrease color variance). In~\cref{fig:EB_only}, we observe that our method can outperform the baseline as long as the bias in the \Univ{} distribution ($H_{q}(Y|A)$) is moderately larger than zero. Consistent with this observation, we also observed in CelebA experiments that using an outside image dataset (FFHQ) as samples from a \Univ{} distribution is effective in boosting performance even though we have not explicitly made the dataset unbiased.

%% file: sections/tables/CelebA_all_universal.tex
\begin{table*}[ht!]
\caption{
Effect of pseudo-labeling on attribute bias removal methods in CelebA dataset under extreme bias.
The baseline method trained on the extreme bias dataset (\emph{TrainEx}) is listed for reference.
All other methods are trained on the combination of \emph{TrainEx} and FFHQ pseudo-labeled by a classifier pretrained on \emph{TrainEx}. With pseudo-labeling, all methods outperform the baseline, with our proposed method achieving the best performance.}

% \rev{Performance under extreme bias in CelebA when training all methods on universal dataset to predict \textit{blond hair}.
% The baseline model trained on extreme bias data (\emph{TrainEx}) is listed for reference.
% All other models are trained on \emph{TrainEx} appending FFHQ. To facilitate the usage of the unlabeled FFHQ for methods other than ours, FFHQ is pseudo-labeled by a classifier pretrained on \emph{TrainEx}. Upon training on universal dataset, most methods clearly outperform baseline, with our method performing the best.}}
% This highlights the practicality of pseudo-labeling to facilitate the construction of universal dataset in addressing extreme bias.}}

\label{tab:CelebA_all_universal}
\centering

\resizebox{0.66\textwidth}{!}{%

\begin{tabular}{lcccc}
\toprule
\multirow{2}{*}{Method} & \multicolumn{2}{c}{Test Accuracy}         & \multicolumn{2}{c}{Mutual Information} \\
\cmidrule(lr){2-3}  \cmidrule(lr){4-5} 
                        & Unbiased ↑          & Bias-conflicting ↑  & $I(Z;A)$ ↓          & $\Delta$ (\%) ↑  \\
                        \midrule
Baseline (TrainEx)      & 66.11{\scriptsize $\pm$0.32}          & 33.89{\scriptsize $\pm$0.45}          & 0.57{\scriptsize $\pm$0.01}           & 0.00             \\
Baseline                & 67.02{\scriptsize $\pm$0.78}          & 35.25{\scriptsize $\pm$1.32}          & 0.48{\scriptsize $\pm$0.01}           & 15.79            \\
\midrule
LNL~\cite{learn_not_to_learn_Colored_MNIST}                     & 67.47{\scriptsize $\pm$0.34}          & 40.56{\scriptsize $\pm$1.24}          & 0.43{\scriptsize $\pm$0.04}           & 24.56            \\
DI~\cite{domain_independent_training}                      & 70.61{\scriptsize $\pm$0.58}          & 46.89{\scriptsize $\pm$0.83}          & 0.39{\scriptsize $\pm$0.03}           & 31.58            \\
LfF~\cite{LfF_CelebA_Bias_conflicting}                     & 69.42{\scriptsize $\pm$0.61}          & 45.54{\scriptsize $\pm$1.26}          & 0.41{\scriptsize $\pm$0.04}           & 28.07            \\
EnD~\cite{End}                     & 67.65{\scriptsize $\pm$0.34}          & 42.85{\scriptsize $\pm$0.65}          & 0.42{\scriptsize $\pm$0.01}           & 26.32            \\
CSAD~\cite{CSAD}                    & 68.18{\scriptsize $\pm$0.16}          & 46.51{\scriptsize $\pm$0.81}          & 0.39{\scriptsize $\pm$0.02}           & 31.58            \\
BCL~\cite{BCL}                     & 70.43{\scriptsize $\pm$0.71}          & 46.86{\scriptsize $\pm$1.61}          & 0.39{\scriptsize $\pm$0.03}           & 31.58            \\
\midrule
Ours                    & \textbf{72.05{\scriptsize $\pm$0.86}} & \textbf{48.72{\scriptsize $\pm$0.56}} & \textbf{0.38{\scriptsize $\pm$0.01}}  & \textbf{33.33}  \\
\bottomrule
\end{tabular}

}
\end{table*}

% \cmidrule(lr){2-3}  \cmidrule(lr){4-5} 

%% file: sections/tables/Adult_all_universal.tex
\begin{table*}[ht!]
\caption{
Effect of pseudo-labeling on attribute bias removal methods in Adult dataset under extreme bias.
The baseline method trained on the extreme bias dataset (\emph{TrainEx}) is listed for reference.
All other methods are trained on the combination of \emph{TrainEx} and bias-conflicting samples pseudo-labeled by a classifier pretrained on \emph{TrainEx}. With pseudo-labeling, all methods outperform the baseline, with our proposed method achieving the best performance.}

\label{tab:Adult_all_universal}
\centering

\resizebox{0.66\textwidth}{!}{%

\begin{tabular}{lcccc}
\toprule
\multirow{2}{*}{Method} & \multicolumn{2}{c}{Test Accuracy}         & \multicolumn{2}{c}{Mutual Information} \\
\cmidrule(lr){2-3}  \cmidrule(lr){4-5} 
                        & Unbiased ↑          & Bias-conflicting ↑  & $I(Z;A)$ ↓          & $\Delta$ (\%) ↑  \\
                        \midrule
Baseline (TrainEx)      & 50.59{\scriptsize $\pm$0.54}          & 1.19{\scriptsize $\pm$0.83}           & 0.69{\scriptsize $\pm$0.00}           & 0.00             \\
Baseline                & 60.86{\scriptsize $\pm$0.13}          & 22.21{\scriptsize $\pm$0.42}          & 0.54{\scriptsize $\pm$0.04}           & 21.74            \\
\midrule
LNL~\cite{learn_not_to_learn_Colored_MNIST}                     & 68.46{\scriptsize $\pm$0.43}          & 46.75{\scriptsize $\pm$0.41}          & 0.46{\scriptsize $\pm$0.03}           & 33.33            \\
DI~\cite{domain_independent_training}                      & 73.25{\scriptsize $\pm$0.32}          & 54.14{\scriptsize $\pm$0.62}          & 0.42{\scriptsize $\pm$0.02}           & 39.13            \\
LfF~\cite{LfF_CelebA_Bias_conflicting}                     & 70.86{\scriptsize $\pm$0.72}          & 51.25{\scriptsize $\pm$0.56}          & 0.44{\scriptsize $\pm$0.02}           & 36.23            \\
EnD~\cite{End}                     & 73.78{\scriptsize $\pm$1.21}          & 56.75{\scriptsize $\pm$1.13}          & 0.43{\scriptsize $\pm$0.03}           & 37.68            \\
CSAD~\cite{CSAD}                    & 72.93{\scriptsize $\pm$1.62}          & 56.82{\scriptsize $\pm$1.95}          & 0.42{\scriptsize $\pm$0.03}           & 39.13            \\
BCL~\cite{BCL}                     & 73.75{\scriptsize $\pm$0.63}          & 57.52{\scriptsize $\pm$1.43}          & 0.41{\scriptsize $\pm$0.02}           & 40.58            \\
\midrule
Ours                    & \textbf{76.35{\scriptsize $\pm$0.31}} & \textbf{60.56{\scriptsize $\pm$1.82}} & \textbf{0.39{\scriptsize $\pm$0.00}}  & \textbf{43.48}  \\
\bottomrule
\end{tabular}

}
\end{table*}

%% file: sections/tables/CelebA_nonsex.tex
\begin{table*}
\caption{Accuracy of attribute bias removal methods under extreme bias and moderate bias in all 23 non-sex-related downstream tasks of CelebA dataset. See~\cref{appsec:CelebA} for separate results. Our proposed method achieves the best performance, both with and without access to the \Univ{} distribution, showing that its trained filter has preserved the information of the other 23 attributes while removing the protected attribute (\ie sex in CelebA).}
\label{tab:CelebA_nonsex}
\centering

% \resizebox{0.66\textwidth}{!}{%
\resizebox{0.7\textwidth}{!}{%

\begin{tabular}{lcccc}
\toprule
\multirow{2}{*}{Method}  & \multicolumn{2}{c}{Extreme Bias Training (\emph{TrainEx})} & \multicolumn{2}{c}{Moderate Bias Training (\emph{TrainOri})} \\
\cmidrule(lr){2-3}  \cmidrule(lr){4-5} 
                        & Unbiased ↑            & Bias-conflicting ↑   & Unbiased ↑             & Bias-conflicting ↑     \\
                         \midrule
Baseline                 & 59.03{\scriptsize $\pm$0.96}           & 21.53{\scriptsize $\pm$1.42}          & 78.08{\scriptsize $\pm$0.82}            & 71.85{\scriptsize $\pm$1.04}           \\
LNL~\cite{learn_not_to_learn_Colored_MNIST}                      & 55.84{\scriptsize $\pm$0.31}           & 18.81{\scriptsize $\pm$0.53}          & 78.43{\scriptsize $\pm$0.75}            & 75.03{\scriptsize $\pm$1.27}           \\
DI~\cite{domain_independent_training}                       & 59.73{\scriptsize $\pm$0.43}           & 22.03{\scriptsize $\pm$0.42}          & 80.83{\scriptsize $\pm$0.54}            & 76.45{\scriptsize $\pm$0.42}           \\
LfF~\cite{LfF_CelebA_Bias_conflicting}                      & 56.12{\scriptsize $\pm$0.35}           & 20.45{\scriptsize $\pm$1.54}          & 79.31{\scriptsize $\pm$0.68}            & 75.82{\scriptsize $\pm$1.73}           \\
EnD~\cite{End}                      & 58.32{\scriptsize $\pm$0.47}           & 20.48{\scriptsize $\pm$0.89}          & 81.14{\scriptsize $\pm$1.61}            & 77.03{\scriptsize $\pm$2.73}           \\
CSAD~\cite{CSAD}                     & 54.65{\scriptsize $\pm$1.43}           & 18.93{\scriptsize $\pm$2.07}          & 80.45{\scriptsize $\pm$1.82}            & 76.20{\scriptsize $\pm$2.94}           \\
BCL~\cite{BCL}                      & 59.28{\scriptsize $\pm$0.58}           & 22.16{\scriptsize $\pm$0.53}          & 81.02{\scriptsize $\pm$0.12}            & 77.81{\scriptsize $\pm$1.83}           \\
Ours                     & 60.13{\scriptsize $\pm$0.27}           & 22.45{\scriptsize $\pm$1.52}          & 81.62{\scriptsize $\pm$1.46}            & 78.76{\scriptsize $\pm$2.84}           \\
\midrule
Ours (FFHQ)      & 63.43{\scriptsize $\pm$0.98}           & 34.98{\scriptsize $\pm$1.93}          & 82.62{\scriptsize $\pm$1.12}            & 79.78{\scriptsize $\pm$1.54}           \\
Ours (Synthetic) & \textbf{63.76{\scriptsize $\pm$1.03}}  & \textbf{36.29{\scriptsize $\pm$1.24}} & \textbf{83.24{\scriptsize $\pm$1.03}}   & \textbf{80.23{\scriptsize $\pm$1.84}} \\
\midrule
\end{tabular}

}

\end{table*}

% \cmidrule(lr){2-3}  \cmidrule(lr){4-5} 

%% file: sections/tables/CelebA_GAN_based_BH.tex
\begin{table*}[htbp]
\caption{Accuracy of GAN-based methods under extreme bias and moderate bias in CelebA to predict \textit{blond hair}. For our method, we report inside parentheses the partially observable \Univ{} distribution used in addition to TrainEx for training its filter. Our method performs better or on-par with GAN-based methods with only half the size of classifier training set.}

\label{tab:CelebA_GAN_BlondHair}
\centering
\resizebox{1\textwidth}{!}{%
\begin{tabular}{lcccccc}
\toprule
\multirow{2}{*}{Method}          & \multicolumn{3}{c}{Extreme Bias Training (\emph{TrainEx})}                      & \multicolumn{3}{c}{Moderate Bias Training (\emph{TrainOri})}                    \\
\cmidrule(lr){2-4}  \cmidrule(lr){5-7} 
                                 & Size of Classifier Training Set ↓ & Unbiased ↑          & Bias-conflicting ↑  & Size of Classifier Training Set ↓ & Unbiased ↑          & Bias-conflicting ↑  \\
                                 \midrule
Baseline                         & 89754                      & 66.11{\scriptsize $\pm$0.32}          & 33.89{\scriptsize $\pm$0.45}          & 162770                     & 75.92{\scriptsize $\pm$0.35}          & 52.52{\scriptsize $\pm$0.19}          \\
CGN~\cite{CGN}                              & 89754$\times$2                    & 63.38{\scriptsize $\pm$1.34}          & 31.46{\scriptsize $\pm$1.42}          & 162770$\times$2                   & 82.65{\scriptsize $\pm$1.82}          & 79.81{\scriptsize $\pm$1.80}          \\
CAMEL~\cite{Camel}                            & 89754$\times$2                    & 64.23{\scriptsize $\pm$1.82}          & 32.81{\scriptsize $\pm$1.18}          & 162770$\times$2                   & 86.45{\scriptsize $\pm$1.17}          & 82.67{\scriptsize $\pm$1.47}          \\
BiaSwap~\cite{BiaSwap}                          & 89754$\times$2                    & 65.97{\scriptsize $\pm$1.12}          & 33.67{\scriptsize $\pm$1.65}          & 162770$\times$2                   & 88.83{\scriptsize $\pm$1.61}          & 85.45{\scriptsize $\pm$1.42}          \\
GAN-Debiasing~\cite{GAN_Debiasing_hat_glasses_correlation}                    & 89754$\times$2                    & 66.83{\scriptsize $\pm$1.73}          & 32.18{\scriptsize $\pm$1.38}          & 162770$\times$2                   & 88.34{\scriptsize $\pm$2.05}          & 85.27{\scriptsize $\pm$1.13}          \\
Ours                             & 89754                      & 66.31{\scriptsize $\pm$0.26}          & 32.22{\scriptsize $\pm$0.43}          & 162770                     & 89.81{\scriptsize $\pm$0.45}          & 85.29{\scriptsize $\pm$1.54}          \\
\midrule
Ours (FFHQ)                      & 89754                      & \textbf{71.53{\scriptsize $\pm$0.67}} & 47.17{\scriptsize $\pm$0.72}          & 162770                     & \textbf{90.86{\scriptsize $\pm$0.87}}          & 88.06{\scriptsize $\pm$0.91}          \\
Ours (Synthetic)                 & 89754                      & 71.37{\scriptsize $\pm$0.64}          & \textbf{48.06{\scriptsize $\pm$0.82}} & 162770                     & 90.01{\scriptsize $\pm$0.65} & \textbf{88.72{\scriptsize $\pm$1.16}} \\
\bottomrule
\end{tabular}
}
\end{table*}

%% file: sections/tables/CelebA_nonsex_sex_hp.tex
\begin{figure*}[ht!]
    \begin{minipage}{0.48\textwidth}
        \captionsetup{type=table}

\caption{Accuracy of protected attribute prediction (lower is better) on the \emph{Unbiased} testing set on sex classification of CelebA. Our filter is trained on the original training set. The vanilla baseline performance is 98.25{\scriptsize$\pm$0.13}. Bold shows the fixed hyper-parameters while others vary.}
\label{tab:CelebA_sex_hp}
\centering
\resizebox{1\textwidth}{!}{%

\begin{tabular}{lccccc}
\toprule
$\lambda_{mi}$    & 0                             & 10                            & 25                            & \textbf{50}                          & 60                            \\
\midrule
Ours              & 95.36{\scriptsize $\pm$0.43} & 90.78{\scriptsize $\pm$0.74} & 86.42{\scriptsize $\pm$0.54} & 84.74{\scriptsize $\pm$0.38} & 84.02{\scriptsize $\pm$0.23} \\
\midrule
\midrule
$\lambda_{pred}$  & 0                             & 10                            & 25                            & \textbf{50}                          & 60                            \\
\midrule
Ours              & 97.27{\scriptsize $\pm$0.36} & 91.13{\scriptsize $\pm$0.54} & 87.81{\scriptsize $\pm$0.87} & 84.74{\scriptsize $\pm$0.38} & 83.45{\scriptsize $\pm$0.41} \\
\midrule
\midrule
$\lambda_{rec}$ & 0                             & 10                            & 50                            & \textbf{100}                         & 110                           \\
\midrule
Ours              & 70.89{\scriptsize $\pm$0.27} & 76.21{\scriptsize $\pm$0.83} & 81.48{\scriptsize $\pm$0.61} & 84.74{\scriptsize $\pm$0.38} & 85.09{\scriptsize $\pm$0.86} \\
\bottomrule
\end{tabular}

}
        
    \end{minipage}
    % \hspace{0.3cm}
    \hfill
    \begin{minipage}{0.48\textwidth}
        \captionsetup{type=table}
\caption{Accuracy of target prediction (higher is better) on the \emph{Unbiased} testing set of all 23 non-sex-related downstream tasks of CelebA. Our filter is trained on the original training set. The vanilla baseline performance is 78.08{\scriptsize$\pm$0.82}. Bold shows the fixed hyper-parameters while others vary.}
\label{tab:CelebA_nonsex_hp}
\centering
\resizebox{1\textwidth}{!}{%

\begin{tabular}{lccccc}
\toprule
$\lambda_{mi}$    & 0                                   & 10                                  & 25                                  & \textbf{50}                             & 60                                  \\
\midrule
Ours              & 76.91{\scriptsize $\pm$0.43} & 78.21{\scriptsize $\pm$0.81} & 79.98{\scriptsize $\pm$1.21} & 81.62{\scriptsize $\pm$1.46}    & 80.72{\scriptsize $\pm$0.71} \\
\midrule
\midrule
$\lambda_{pred}$  & 0                                   & 10                                  & 25                                  & \textbf{50}                             & 60                                  \\
\midrule
Ours              & 73.54{\scriptsize $\pm$0.17} & 76.83{\scriptsize $\pm$0.55} & 78.39{\scriptsize $\pm$0.49} & 81.62{\scriptsize $\pm$1.46}    & 79.82{\scriptsize $\pm$0.62} \\
\midrule
\midrule
$\lambda_{rec}$ & 0                                   & 10                                  & 50                                  & \textbf{100}                            & 110                                 \\
\midrule
Ours              & 43.83{\scriptsize $\pm$0.46} & 60.81{\scriptsize $\pm$0.51} & 71.43{\scriptsize $\pm$0.83} & 81.62{\scriptsize $\pm$1.46}    & 81.48{\scriptsize $\pm$0.23} \\
\bottomrule
\end{tabular}
}

    \end{minipage}%
\end{figure*}

%% file: sections/pictures/CMNIST_lambda_visualization.tex
\begin{figure}[t]
\begin{center}
  \includegraphics[width=0.77\linewidth]{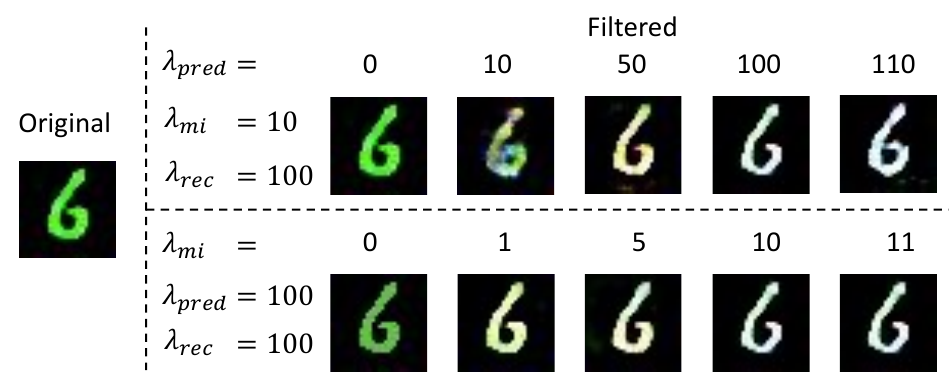}
\end{center}
  \caption{Effectiveness of hyper-parameters to remove protected attribute on Colored MNIST.}
\label{fig:CMNIST_lambda_visualization}
\end{figure}

%% file: sections/pictures/CMNIST_EB_only.tex
\begin{figure}[htbp]
\begin{center}
  \includegraphics[width=0.8\linewidth]{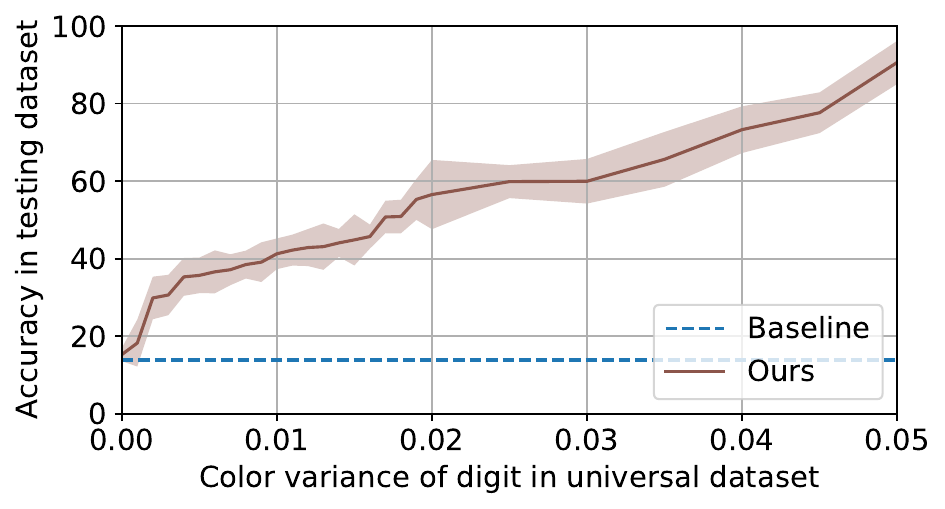}
\end{center}
  \caption{
  The effect of bias strength in the \Univ{} distribution on our method in the extreme bias setting (corresponding to the zero location on the horizontal axis in~\cref{fig:EB}).
  Since the \Univ{} distribution is partially observable (lacks target labels), the baseline classifier and other bias removal methods have a constant performance.
  }
\label{fig:EB_only}
\end{figure}

%% file: sections/08_discussion.tex
\label{sec:discussion}
In this work, we mathematically and empirically showed the sensitivity of the state-of-the-art attribute bias removal methods to the bias strength, revealing a previously overlooked limitation of these methods. 
In particular, we empirically demonstrated that when a protected attribute is strongly predictive of a target, these methods become ineffective. 
To understand the cause and extent of these findings, we derived an information-theoretical upper bound on the performance of any attribute bias removal method, and verified it in experiments on synthetic, image, and census datasets. These findings caution against the use of existing attribute bias removal methods in datasets with potentially strong bias (\eg small datasets).
Next, based on our theoretical analysis, we stated a necessary condition for the existence of any method that can remove the extreme attribute bias (\ie \Univ{} distribution), and then proposed a new method which not only overcomes extreme bias under the necessary condition but also performs better or on-par with the existing state-of-the-art methods even without access to a \Univ{} distribution.

\noindent
\textbf{Limitations and Future Directions.} 
While our method shows promising results, in the ablation studies we found that it is sensitive to the amount of bias in the \Univ{} distribution itself. Thus, an interesting future direction is constructing methods that are less sensitive to the bias in the \Univ{} distribution. 
Another interesting direction is to explore how to construct a \Univ{} distribution more efficiently.
Also, it is important to consider more challenging scenarios where attribute labels are absent~\cite{without_pa1, EIIL} or unknown biases emerge~\cite{unknown_bias}. 
Finally, it is noteworthy that while the ability to effectively remove protected attributes is a valuable tool, removing them will not always result in a fairer decision, as in some cases rewarding a demographic group might be desirable, a complex subject discussed more elaborately in~\cite{protected_attributes}.

\noindent
\textbf{Societal Impact.} We expect our work to promote positive societal impact by addressing strong attribute bias from neural networks which hinders the fairness of machine intelligence.

%% file: sections/acknowledgement.tex
\section*{Acknowledgement}

This research is based upon work supported in part by the Office of the Director of National Intelligence (ODNI), Intelligence Advanced Research Projects Activity (IARPA), via [2022-21102100007]. The views and conclusions contained herein are those of the authors and should not be interpreted as necessarily representing the official policies, either expressed or implied, of ODNI, IARPA, or the U.S. Government. The U.S. Government is authorized to reproduce and distribute reprints for governmental purposes notwithstanding any copyright annotation therein.

%% file: sections/supplement.tex
\section{Proofs}
\label{appsec:proofs}
We consider all random variables to be discrete, as these are represented by floating point numbers in practice.
 
\noindent
\textbf{Proposition 1.} \emph{Given random variables $Z, Y, A$, in case of the extreme attribute bias $H(Y|A)=0$, if the attribute is removed from the feature $I(Z;A)=0$, then $I(Z;Y)=0$, \ie no classifier can outperform random guess.}

\begin{proof}
 
Generally, we have,
\begin{equation}
\begin{split}
	p(z,y) &= \sum_{a} p(z,y,a) \\
	&= \sum_{a} p(y|z,a) p(z,a).
\end{split}
\end{equation}
As $I(Z;A)=0$ and by the property of mutual information that $I(Z;A)=0$ if and only if random variable $Z$ and random variable $A$ are mutually independent, we have,
 
\begin{equation}
\begin{split}
\label{2}
	p(z,y) &= \sum_{a} p(y|z,a) p(z) p(a). \\
\end{split}
\end{equation}
By the property of conditional entropy that $H(Y|A)=0$ if and only if $\exists~\text{a function}~g: \mathcal{A} \to \mathcal{Y}, Y=g(A)$, and the property that if $Y$ is determined by $A$ with a function $g$, $\forall~\text{random variable}~Z, p(y|a,z) = p(y|a)$, \cref{2} goes to,

\begin{equation}
\begin{split}
	p(z,y) &= \sum_{a} p(y|a) p(z) p(a) \\
	&= \sum_{a} p(y,a) p(z) \\
	&= p(y) p(z).
\end{split}
\end{equation}
Thus, $Z$ and $Y$ are mutually independent, and we have,
\begin{equation}
\begin{split}
	I(Z;Y) = 0.
\end{split}
\end{equation}
\end{proof}

\newpage
 
\noindent
\textbf{Theorem 1.}
\emph{Given random variables $Z, Y, A$, the following inequality holds without exception:}
\begin{align}
	0 \leq I(Z;Y) \leq I(Z;A) + H(Y|A)
\end{align}

\begin{proof}
For interaction information $I(Z;Y;A)$, we have,
\begin{equation}
\begin{split}
I(Z;Y;A) &= I(Z;Y) - I(Z;Y|A).
\end{split}
\end{equation}
By the symmetry property of mutual information,
\begin{equation}
\begin{split}
I(Z;Y;A) = I(Z;Y) - I(Y;Z|A).
\end{split}
\end{equation}
By the chain rule of conditional mutual information,
\begin{equation}
\begin{split}
I(Z;Y;A) &= I(Z;Y) - [I(Y;Z,A) - I(Y;A)] \\
&= I(Z;Y) - I(Y;Z,A) + I(Y;A).
\end{split}
\end{equation}
By the property of interaction information $I(Z;Y;A) \leq \min\{I(Z;Y), I(Y;A), I(Z;A)\}$, we have,
\begin{equation}
\begin{split}
I(Z;Y;A) &\leq I(Z;A) \\
I(Z;Y) - I(Y;Z,A) + I(Y;A) &\leq I(Z;A).
\label{4}
\end{split}
\end{equation}
According to the relation of mutual information to conditional and joint entropy, Left Hand Side (LHS) of~\cref{4} goes to,
\begin{equation}
\begin{split}
LHS = &I(Z;Y) - [H(Y)-H(Y|Z,A)] \\
&+ [H(Y) - H(Y|A)] \\
= &I(Z;Y) + H(Y|Z,A) - H(Y|A).
\end{split}
\end{equation}
Then, \cref{4} goes to,
\begin{equation}
\begin{split}
I(Z;Y) + H(Y|Z,A) - H(Y|A) &\leq I(Z;A), \\
I(Z;Y) + H(Y|Z,A) &\leq I(Z;A) + H(Y|A).
\end{split}
\end{equation}
As $H(Y|X,A) \geq 0$ and mutual information is non-negative, we have,
\begin{equation}
\begin{split}
0 \leq I(Z;Y) &\leq I(Z;A) + H(Y|A).
\end{split}
\end{equation}
\end{proof}

\clearpage

\section{Approach to Utilize Non-Observable \Univ{} Distribution} 
\label{appsec:SSL}

In~\cref{sec:necessary} of the main paper, we show that in extreme bias $H(Y|A)=0$, the existence of a \Univ{} distribution is necessary to overcome the trade-off between attribute bias removal $I(Z;A)$ and the best attainable performance $I(Z;Y)$.
Further, given the existence, we present three possible scenarios regarding the observability of target $Y$ and protected attribute $A$.
These scenarios include: (1) \emph{Fully Observable} containing both target labels and protected attribute labels, (2) \emph{Partially Observable} lacking target labels but containing protected attribute, and (3) \emph{Non Observable} lacking any labels.
We have discussed the first two scenarios in the main paper. 
In this section, we mainly discuss the third scenario where \Univ{} distribution does not yield any labels.
Under this weakest annotation possibility, we seek a simple approach based on self-supervised learning (SSL) to demonstrate the potential of overcoming the trade-off between attribute bias removal and the best attainable performance.

The idea of this approach is to utilize \Univ{} dataset to train an encoder that can bring features from the same input closer together while pushing them away from features belonging to other inputs. 
By doing so, the encoder learns to capture the intrinsic information of the input without being influenced by the data with attribute bias.
To this end, we deploy a two-stage training scheme, as shown in~\cref{fig:method_SSL}, 
First, we fine-tune the pre-trained baseline encoder~\cite{SEER,pretrained} on \Univ{} dataset with contrastive loss~\cite{simclr}. Then, we apply the frozen encoder followed by a shallow regressor trained from scratch in the downstream task to address attribute bias. 

Specifically, during the encoder training, given the input $x^U$ from a minibatch of size $N$ in \Univ{} dataset, a set of data augmentation techniques randomly augment $x^U$ to be $x_i$ and $x_j$.
Following~\cite{simclr}, the set of augmentation includes RandomResizedCrop, RandomHorizontalFlip, RandomApply(ColorJitter), and RandomGrayscale.
Then, a pre-trained baseline encoder $E: \mathcal{X} \to \mathcal{Z}$ outputs learned representation $z_i$ and $z_j$, followed by a mapping network $M: \mathcal{Z} \to \mathcal{H}$ to output $h_i$ and $h_j$, respectively. 
To ensure $z_i$ and $z_j$ are closer to each other than other features, we optimize $E$ and $M$ over the following contrastive loss~\cite{simclr} for pair $(i,j)$:
\begin{align}
\label{eq:contrastive}
 \mathcal{L}^{E,M}_{\text{contrastive}}(i,j) = -\log(\frac{\exp(\text{sim}(h_i, h_j)/\tau)}{\sum_{k=1}^{2N} \mathbb{I}(k \neq i) \exp(\text{sim}(h_i, h_k)/\tau)})
\end{align}
where $\text{sim}(h_i, h_j) = \frac{h_i \cdot h_j}{{\|h_i\| \|h_j\|}}$, $\tau$ is the temperature parameter, and $\mathbb{I}(k \neq i) \in \{0,1\}$ is an indicator function that equals 1 if and only if $k \neq i$. The loss is computed on both $(i,j)$ and $(j,i)$ for each input $x^U$.

Further, during applying the trained encoder to the biased dataset of downstream task $\{\mathcal{X}^B, \mathcal{Y}\}$, given the input $x^B \in \mathcal{X}^B$, the frozen encoder is used to output learned representation $z$. 
Then, a regressor $R: \mathcal{Z} \to \mathcal{Y}$ is introduced with the objective:
\begin{align}
\label{eq:reg_supp}
    R^* = \argmin_R \mathcal{L}_{reg}(R(z), y)
\end{align}
\noindent
where $\mathcal{L}_{reg}$ is an appropriate regression loss (L2 loss for continuous attributes and cross-entropy loss for discrete attributes). 

The whole framework can be found in~\cref{fig:method_SSL}. The network architecture is shown in~\cref{tab:CelebA_SSL_network}, and the hyper-parameters are shown in~\cref{tab:hp_SSL}.

\noindent
\textbf{Results.}
Following the experiment setup of CelebA in~\cref{sec:exp} of the main paper to predict \textit{BlondHair}, we conduct the same experiment on the SSL-based approach.
As shown in~\cref{tab:CelebA_SSL}, compared with its relatively worse performance than baseline without \Univ{} dataset, this approach outperforms baseline when getting access to \Univ{} dataset, which highlights the possibility to utilize \Univ{} dataset to escape the trade-off between attribute bias removal and achieving good accuracy even in the absence of any strengthened supervision signals.

\input{sections/pictures/framework_SSL}
\input{sections/tables/CelebA_SSL}

\section{Statistics of Attribute Bias in Various Real-World Datasets}
\label{appsec:stats}

In this section, we present detailed statistics of attribute bias in various real-world datasets to illustrate that the presence of strong bias is a common phenomenon in real-world datasets. 
For instance, in the healthcare domain, a publicly available diabetes dataset, Pima Indians Diabetes Dataset~\cite{diabetes_dataset}, collected by the National Institute of Diabetes and Digestive and Kidney Diseases, contains a large portion of negative samples for young individuals (352) and a small portion of negative samples for old individuals (148), which leads age to be strongly predictive of diabetes diagnoses~\cite{diabetes_chapter}.\footnote{Following~\cite{BlindEye_IMDB_eb}, we choose age around 30 (33) as the boundary to transform the origin continuous age label to be discrete (\ie old versus young).}
Moreover, diagnosing Human Immunodeficiency Virus (HIV) from Magnetic Resonance Imaging (MRI) images, by statistics in~\cite{dataset_vs_task}, age of control group is 45{\scriptsize$\pm$17.0} while the age of HIV subjects is 51{\scriptsize$\pm$8.3}, making age a very strong attribute bias for this task.
Furthermore, in CelebA dataset, \textit{HeavyMakeup}, which is a widely used attribute to study attribute bias in~\cite{LfF_CelebA_Bias_conflicting, CSAD, End} is strongly predictive of sex.  
This is evident from the fact that the positive rate of HeavyMakeup in males is only $0.28\%$ (234 out of 84434), which is close to zero, whereas the positive rate of HeavyMakeup in females is $66\%$ (78156 out of 118165). Therefore, the classification of HeavyMakeup is highly influenced by the spurious association between HeavyMakeup and sex. Similarly, other non-sex-related attributes in CelebA also exhibit a significant disparity in the positive rate between males and females. For example, \textit{BlondHair} (2\%, 24\%), \textit{WavyHair} (14\%, 45\%), \textit{HighCheekbones} (31\%, 56\%), \textit{BagsUnderEyes} (35\%, 10\%), \textit{BigNose} (42\%, 10\%), and \textit{PointyNose} (16\%, 36\%).
All of these statistics illustrate that strong bias is a common problem in the real world.

\section{Feasibility of Collecting Protected Attribute Labels}
\label{appsec:feasibility}
In~\cref{sec:necessary} of the paper, we assume that \Univ{} distribution can yield no target labels and only protected attribute labels in the partially observable scenario since collecting protected attributes is generally more practical than collecting multiple target attributes due to differences in volume.
In this section, we further clarify the feasibility of collecting protected attribute labels.
Notably, protected attribute is commonly collected as demographic statistics for various purposes.
For instance, in medical applications such as drug design, protected attributes are commonly collected and utilized to design personalized medications that are more effective for particular patient cohorts. 
Furthermore, despite the challenges posed by privacy and legal issues in collecting protected attributes, it remains feasible to collect them in a responsible and ethical manner. 
For instance, for well-defined and limited protected attributes (\eg sex), leveraging pre-trained classifiers to label face datasets has shown promise such as sex classification applied through transfer learning using commercial services trained on balanced datasets (FairFace)~\cite{Fairface}.

\section{Empirical Bounds}
\label{appsec:bounds}
In~\cref{sec:theory} of the main paper, to empirically validate our theory, we compute the terms in~\cref{th:general} for several attribute bias removal methods on CelebA and plot partial results in~\cref{fig:CelebA_bounds} of the main paper.
Specifically, we use mutual information neural estimator~\cite{MINE} to estimate $I(Z;A)$ and $I(Z;Y)$, and empirically compute $H(Y|A)$.
To complement~\cref{fig:CelebA_bounds} of the main paper, we present results for all methods in~\cref{fig:CelebA_bounds_supp}. 
We observe that the bounds remain valid for all models.

\input{sections/pictures/CelebA_bounds_supp}

\section{Distinction and Connection with ``Impossibility results for fair representations''}
\label{appsec:impossibility}
Lechner \etal~\cite{impossibility_for_fair_repesentations} claim that the non-trivial fair representations in terms of Demographic Parity and Odds Equality are unattainable.
Demographic Parity (DP), which requires independence between model predictions $\hat{Y}$ and attributes $A$ ($I(\hat{Y};A)=0)$, aligns with the goal of our work.
Thus, we mainly clarify the difference with DP.

Specifically, Claim 1 in~\cite{impossibility_for_fair_repesentations} states that any representation that accommodates a non-constant classifier cannot ensure DP-fairness for all potential tasks involving the same set of individuals.
This claim can be proved by finding a task for which no representation can guarantee DP-fairness.
For example, they consider the task to predict attribute $A$ itself, rendering the fair representation with respect to DP for this particular task cannot exist.
However, it is important to note that the claim argues that DP-fair representation cannot exist for all potential tasks but it does not negate the feasibility of the specific task in general attribute bias scenarios.

Furthermore, the task they provide, where the target $Y$ is the attribute $A$ itself, aligns with the extreme bias case where $H(Y|A)=0$ in our paper.
In this scenario, as elaborated in~\cref{sec:necessary}, we also claim that if $H(Y|A)=0$, the successful attribute bias removal method does not exist.
However, they do not specify the condition that the existence of a fair representation necessitates.
In contrast, our investigation further extends to state that a successful attribute bias removal method exists only if the \Univ{} dataset $q$ where $H_{q}(Y|A)>0$ exists.
In other words, fairness with respect to DP can be achieved only when this particular condition is satisfied.

In summary, the claim in~\cite{impossibility_for_fair_repesentations} does not contradict our claim.

\section{Breaking Point of Attribute Bias Removal Methods on Colored MNIST}
\label{appsec:breaking_point}

In this section, we present the details regarding breaking point in~\cref{fig:EB} of the main paper and present additional visualizations of Colored MNIST in~\cref{fig:CMNIST_visualization}.

As mentioned in~\cref{sec:introduction} of the main paper, digit classification on Colored MNIST, assuming color as a protected attribute to be removed, is a popular controlled experiment used by several attribute bias removal methods~\cite{learn_not_to_learn_Colored_MNIST, Back_MI,CSAD} to measure their effectiveness. More concretely, this is a ten-class classification problem given an image as input, where in the training set each digit is assigned a unique RGB color with a fixed variance across all digits, and then to measure how much the trained model is relying on color to predict digit, its accuracy is reported on a testing set where colors are assigned to each sample at random. The higher the accuracy, the less the model has learned to rely on the color -- the protected attribute in training -- for predicting digits. The experiment is repeated for different values of the color variance, and the results are plotted against color variance to show the sensitivity of a method to different levels of bias strength. While existing methods are effective in this experiment -- with a minor sensitivity to bias strength -- the performance is always only reported in the color variance range of $[0.02, 0.05]$, without any explicit justification on the choice of this range, thus leaving the question of whether these methods are effective outside of this range open. 

To empirically address this question, we extend the Colored MNIST experiment in two ways. 
First, we extend the horizontal axis to zero to include the color variance region $[0,0.02]$ where color becomes very predictive of the digit, hence the attribute bias strength is increased to its maximum. We denote this range the \textit{strong bias} region, with extreme bias happening at zero. \cref{fig:EB} of the main paper shows that unlike in the moderate bias region, the performance of all existing methods becomes very sensitive to the bias strength in the strong bias region. Moreover, we see that in extreme bias, all methods perform close to chance. Second, we repeat the experiment at each color variance 15 times and compute error bars for the whole range. This enables us to not only visually judge the significance of differences in performance, but perhaps more interestingly, be able to compute a \textit{breaking point} for each method. Formally, we define the breaking point of a method as the largest color variance (smallest bias strength) where its performance is no longer significantly better than a naive baseline. The naive baseline only uses the typical cross-entropy loss, without any implicit or explicit mutual information minimization term, and otherwise has the exact same classification network structure as all the other methods as presented in~\cref{appsec:training_details}. To compute the breaking point, we conduct hypothesis tests based on a two-sample one-way Kolmogorov–Smirnov test at different color variances and set the null hypothesis to that model performance is better than baseline performance: the largest color variance at which the null hypothesis is rejected with a $p$-value that is less than or equal to the significance level of 0.05 is considered as the breaking point. We provide the $p$-values of each method over different color variances in~\cref{tab:Pvalue}.
In~\cref{fig:EB} of the main paper, the breaking point of each method is illustrated with a $\blacktriangle$ on the x-axis. Interestingly, we observe that different methods have different breaking points, showing that methods that might appear to perform very closely in the moderate bias region (\eg compare LNL and EnD), can perform very differently in the strong bias region.

These empirical findings reveal that the existing methods are only effective under the previously unstated assumption that the attribute bias is not too strong, an important consideration for the use of these methods in practice. But equally important, we observe a clear connection between the effectiveness of existing methods and the attribute bias strength. 
Then, in~\cref{sec:theory} of the main paper, we formulate this connection.

\input{sections/tables/Pvalue}

\input{sections/pictures/CelebA_visualizaton_supp}
\clearpage
\input{sections/pictures/IMDB_visualization}
\clearpage
\input{sections/pictures/CMNIST_visualization}
\clearpage

\section{Details of Experiments on CelebA}
\label{appsec:CelebA}
In this section, we provide the detailed setup used to construct strong bias region. 
Besides, we conduct a comparative analysis of training our filter using real data and semi-synthetic data. 
Furthermore, we supplement the separate results on all 23 non-sex-related facial attributes of CelebA in addition to the average results in~\cref{tab:CelebA_nonsex} of the main paper.
Moreover, we include additional visualizations in~\cref{fig:CelebA_visualization_supp}.

\subsection{Strong Bias Region}

To study strong bias region of CelebA, we select 18 different points in bias strength (represented by $H(Y|A)$) to plot~\cref{fig:EB_CelebA} in the main paper.
These points can be categorized into two groups based on different setups with $H(Y|A)=0.13$ as the dividing point.
Initially, the first 11 points cover the range of $H(Y|A)$ from extreme bias 0 to 0.13. In this range, we ensure an even distribution of bias-conflicting samples across sex while gradually varying the percentage of such samples from 0 to 1, using a step size of 0.1.
Afterward, the next 7 points span a range of $H(Y|A)$ from 0.13 to moderate bias point 0.36. In this range, we incorporate all bias-conflicting samples from the original training set without preserving the even number of these samples across sex, and adjust the percentage of bias-conflicting samples from 0.1 to 1 with a step size of 0.1.
Moreover, within this range, the percentages (0.7, 0.8, and 0.9) are omitted from consideration due to their proximity in terms of $H(Y|A)$.
Thus, in general, we construct 18 training sets with different bias strengths.

\subsection{Comparison between Real \Univ{} Dataset and Semi-Synthetic \Univ{} Dataset}
In general, there are two ways one could gain access to \Univ{} dataset in practice (if it exists): either by collecting real data from different tasks or generating synthetic data through generative models. 
In~\cref{fig:EB_CelebA} of the main paper, for clarification, \Univ{} dataset is constructed by appending \emph{TrainEx} with FFHQ dataset~\cite{ffhq}. 
Moreover, we also construct the semi-synthetic \Univ{} dataset by combining \emph{TrainEx} with a same-sized synthetic dataset where images are randomly generated from random seeds using~\cite{CAT}.
In this section, we empirically compare these two ways following the setup of CelebA in~\cref{sec:exp} of the main paper.

In~\cref{fig:CelebA_ucomparison}, we can observe that when comparing the filter trained on the real \Univ{} dataset with the filter trained on the semi-synthetic \Univ{} dataset, the filter trained on real data exhibits slightly better performance in terms of test accuracy of \emph{Unbiased} testing set (the average accuracy on both bias-aligned samples and bias-conflicting samples). However, it performs slightly worse in terms of test accuracy of \emph{Bias-conflicting} testing set.
We conjecture that this difference arises from the fact that the real-world dataset FFHQ may exhibit similar attribute bias as CelebA by long-tail distribution~\cite{long_tail}, thereby leading to an accuracy improvement among bias-aligned samples. 
Conversely, the synthetic data contains few samples involving the same attribute bias due to its randomness in the generation process, thereby further enhancing the accuracy of bias-conflicting samples. 
Nevertheless, due to the comparatively lower image quality of synthetic data than real data, the filter trained on the synthetic data yields a comparatively lower average accuracy.

In~\cref{tab:CelebA_supp_append_only}, we present the performance of our filter trained exclusively on FFHQ and synthetic dataset for reference to supplement~\cref{tab:CelebA_nonsex} of the main paper.
It can be observed that the performance of the filter trained solely on FFHQ or synthetic dataset is not as good as the filter trained on \Univ{} dataset consisting of the original CelebA dataset and appended dataset in~\cref{tab:CelebA_nonsex} of the main paper. 
We conjecture that this is because of domain bias~\cite{domain_bias} in FFHQ and the comparatively lower image quality of synthetic data compared with the original CelebA dataset.

\input{sections/pictures/CelebA_universal_comparison}
\input{sections/tables/CelebA_supp_append_only}

\subsection{Separate Results on All Non-Sex-Related Attributes under Extreme Bias and Moderate Bias}
\label{appsec:CelebA_separate}

To demonstrate the generality of our filter on different downstream tasks, we provide separate results on all 23 non-sex-related downstream tasks of CelebA in~\cref{tab:CelebA_supp,tab:CelebA_supp_moderate}, in addition to the average results in~\cref{tab:CelebA_nonsex} of the main paper.
All 23 non-sex-related attributes include \textit{BagsUnderEyes}, \textit{Bangs}, \textit{BigLips}, \textit{BigNose}, \textit{BlackHair}, \textit{BlondHair}, \textit{Blurry}, \textit{BrownHair}, \textit{Chubby}, \textit{DoubleChin}, \textit{Eyeglasses}, \textit{GrayHair}, \textit{HeavyMakeup}, \textit{HighCheekbones}, \textit{MouthSlightlyOpen}, \textit{NarrowEyes}, \textit{PaleSkin}, \textit{PointyNose}, \textit{Smiling}, \textit{StraightHair}, \textit{WavyHair}, \textit{WearingHat} and \textit{Young}. 

As shown in~\cref{tab:CelebA_supp}, our filter consistently outperforms other methods in extreme bias.
Furthermore, in~\cref{tab:CelebA_supp_moderate}, we also provide detailed results in moderate bias point to supplement the CelebA experiments and fully compare with other methods~\cite{LfF_CelebA_Bias_conflicting, CSAD, End}.
These experiments are conducted following the setup of CelebA in~\cref{sec:exp} of the main paper and the classification models are trained on \emph{TrainOri} training set.
As shown in~\cref{tab:CelebA_supp_moderate}, our model yields better average results than other methods in moderate bias point. 
We conjecture that this is because other methods focus on achieving better classification performance \wrt target labels while simultaneously addressing attribute bias \wrt protected attribute labels.
However, maximizing classification accuracy competes with minimizing attribute bias due to the association between target labels and protected attribute labels.
In contrast, our approach initially trains a target-agnostic filter that specializes in debiasing by solely removing protected attributes while preserving other attributes. 
Subsequently, these filtered images are utilized for prediction, ensuring that all target information remains unaltered. 
As a result, our filter performs comparably or even outperforms other methods in various downstream tasks.

\subsection{Pseudo-Labeling of the \Univ{} Distribution}

In~\cref{subsec:all_universal} of the main paper, we conduct the comparison across different methods when training all methods on \Univ{} dataset with pseudo-labels of targets. In that section, we assume that the dataset with both target and protected attribute labels yields extreme bias so that the model used for pseudo-labeling is also trained on the extreme bias data.
For a more comprehensive comparison, in this section, we provide the results on CelebA dataset where the model used for pseudo-labeling is trained on the original CelebA dataset (\emph{TrainOri}). As shown in~\cref{tab:Adult_all_universal_ori}, our method outperforms other methods.

Furthermore, as shown in~\cref{tab:CelebA_all_universal_ori_incremental,tab:CelebA_all_universal_incremental}, our method benefits most from appending FFHQ. Furthermore, the accuracy improvement is aligned with the breaking point of each method, which differs across different methods. Specifically, methods with lower breaking points perform better.

\input{sections/tables/CelebA_all_universal_ori}
\input{sections/tables/CelebA_all_universal_ori_increment}
\input{sections/tables/CelebA_all_universal_increment}

\subsection{Comparison with Generative Dataset Augmentation}
\label{appsec:GAN-based}
In~\cref{subsec:GANs} of the main paper, we compare our filter with the methods based on generative adversarial network (GAN) from two key aspects: framework difference and experimental evaluation. In this section, we supplement more analysis on the difference between other GAN-based methods with our method.

First, from the perspective of framework, one notable distinction between our approach and other GAN-based methods~\cite{CGN, Camel, BiaSwap, GAN_Debiasing_hat_glasses_correlation} is that we use generative models to directly modify the input image by eliminating the protected attribute information while preserving other information, and only utilize the filtered images to train the downstream task classifier, rather than augment the original dataset with additional images. This sets our approach apart from other GAN-based methods that rely on using both the original and augmented data for classifier training, with the goal of maintaining performance. Our use of generative models to edit images in place is a unique feature of our filter, making it a new approach to address this problem. Furthermore, our filter is target-agnostic and only needs to be trained once when applying to different targets of downstream tasks, but existing methods~\cite{CGN, Camel, BiaSwap, GAN_Debiasing_hat_glasses_correlation} require target labels for training and need to be retrained every time during applying.
Specifically, CGN~\cite{CGN} is designed with three predefined independent mechanisms, namely shape, texture, and background. This design is effective for digit classification on MNIST~\cite{mnist} and object classification on ImageNet~\cite{imagenet}, but may not be suitable for face dataset since many facial attributes (\eg \textit{Chubby}) cannot be well extracted by these three predefined mechanisms. In contrast, our filter is more general without the need for prior knowledge about the type of bias attribute. Also, CGN is trained on an even number of real and counterfactual images in each batch, and as authors mentioned in~\cite{CGN} the performance of CGN may suffer if there are fewer real images than counterfactual images. In contrast, the classification task in use of our filter solely relies on synthetic data without any real data. According to the comparable performance between CGN and our filter, as shown in~\cref{tab:CelebA_GAN_BlondHair,tab:CelebA_GAN_HeavyMakeup}, the quality of the generated images from our filter may be better than that of CGN. Furthermore, the technique of other GAN-based methods~\cite{Camel,BiaSwap,GAN_Debiasing_hat_glasses_correlation} may be ineffective in the extreme bias setup. For instance, in the case of CAMEL~\cite{Camel}, inter-subgroup augmentation cannot be normally deployed in the first stage of the two-stage training process if there are few bias-conflicting samples in the extreme bias case. Similarly, BiaSwap~\cite{BiaSwap} relies on identifying bias-aligned samples and bias-conflicting samples, but in the extreme bias case, there may be no bias-conflicting samples to train on.

\input{sections/tables/CelebA_GAN_based_HM}

\section{Details of Experiments on Adult}
\label{appsec:Adult}
In this section, we provide additional information about the experiments conducted on Adult dataset. 
For the census data in Adult dataset, we perform data preprocessing where continuous attributes are normalized to the range of $[0,1]$, and discrete attributes are encoded as one-hot vectors.
To plot~\cref{fig:EB_Adult} in the main paper, we construct several training sets with different $H(Y|A)$ by gradually incorporating bias-conflicting samples into \emph{TrainEx} training set. 
Specifically, we adjust the percentage of bias-conflicting samples from 0 to 1 with 0.1 as the step size.

% \subsection{Comparison when the Model Used for Pseudo-Labeling \Univ{} Dataset is Trained on Original Dataset}
\subsection{Pseudo-Labeling of the \Univ{} Distribution}

In~\cref{subsec:all_universal} of the main paper, we conduct the comparison across different methods when training all methods on \Univ{} dataset with pseudo-labels of targets. In that section, we assume that the dataset with both target and protected attribute labels yields extreme bias so that the model used for pseudo-labeling is also trained on the extreme bias data.
For a more comprehensive comparison, in this section, we provide the results on Adult dataset where the model used for pseudo-labeling is trained on the original Adult dataset. As shown in~\cref{tab:Adult_all_universal_ori}, our method outperforms other methods.

\input{sections/tables/Adult_all_universal_ori}

\input{sections/tables/CelebA_supp}
\input{sections/tables/CelebA_supp_moderate}

\section{Experimental Evaluation on IMDB}
\label{appsec:IMDB}
In this section, we present the complementary results on IMDB dataset~\cite{IMDB} in~\cref{tab:IMDB}, and the visualization of the filtered images on IMDB dataset in~\cref{fig:IMDB_visualization_supp}.
Furthermore, we conduct a comprehensive analysis of the results, both quantitatively and qualitatively.

\subsection{Experiment Setup}

Following the setup in~\cite{learn_not_to_learn_Colored_MNIST}, we conduct experiments on IMDB with age as the protected attribute.
It is worth noting that age is a twelve-class random variable quantified with five-year steps in IMDB, which is different from the binary protected attribute label (sex) in CelebA.
To facilitate our analysis, following data split in~\cite{learn_not_to_learn_Colored_MNIST}, we introduce three sub-datasets: (1) \emph{EB1} consisting of women aged 0-29 and men aged 40+, (2) \emph{EB2} consisting of women aged 40+ and men aged 0-29, and (3) \emph{Test} which contains 20\% of all cleaned samples aged 0-29 and 40+. 
Besides, we construct a \Univ{} dataset by appending the biased dataset with the synthetic dataset generated by the generative model~\cite{CAT}.

\input{sections/tables/IMDB}

\subsection{Quantitatively Analysis}
Compared with CelebA results in~\cref{tab:CelebA_BlondHair} of the main paper, we observe that all models perform well on IMDB, as shown in~\cref{tab:IMDB}. To better analyze it, we present detailed statistics in~\cref{fig:statistics_IMDB,fig:boxplot_IMDB}.
According to the age label division shown in~\cref{tab:IMDB_age_label}, we present boxplots of age labels for females and males in~\cref{fig:boxplot_IMDB}. 
These boxplots show that \emph{Test} yields the same attribute bias~as \emph{EB1} so that the accuracy is high when trained on \emph{EB1} and tested on \emph{Test}.
Moreover, we observe that the sex distribution of \emph{EB2} is severely imbalanced according to the number of female samples (1744) and male samples (15056) as shown in~\cref{fig:statistics_IMDB}.
Besides, the difference between samples with closed age is not visually clear when comparing the younger females or younger males with the older females or older males respectively in~\cref{fig:IMDB_EB_comparison}.
Furthermore, the majority age group of females is 25-29 in \emph{EB1} and 40-44 in \emph{EB2}, and the majority age group of males is 40-44 in \emph{EB1} and 25-29 in \emph{EB2}. Thus, the dominated groups for each sex category in \emph{EB1} and \emph{EB2} are visually indistinguishable. 
Therefore, the models may compromise the classification performance with attribute bias removal \wrt age since age information is hard to capture by models, thereby resulting in a good performance for all methods.

\subsection{Qualitatively Analysis}
As shown in~\cref{fig:IMDB_visualization_supp}, our model filters out age information by blurring the image and adding wrinkles to images of young individuals, which further closes the appearance gap between different age groups.
Notably, there are several outliers in the training set hindering the generation of high-quality images. As shown in~\cref{fig:IMDB_long_distance}, some samples do not contain an obvious face. Moreover, in~\cref{fig:IMDB_two}, some samples contain more than one face. Besides, as identified in~\cite{BlindEye_IMDB_eb}, the age labels of IMDB are noisy. Overall, these obstacles negatively affect the generation of high-quality filtered images for our approach.
 
\section{Experimental Evaluation on Waterbirds}
\label{appsec:Waterbirds}
We supplement the experiments on Waterbirds~\cite{DRO}, where the target is waterbird or landbird, and the attribute is water background or land background. Following Group DRO~\cite{DRO}, the training set contains more instances of waterbirds with water backgrounds and landbirds with land backgrounds compared to other combinations, whereas such a correlation does not exist in the testing set. 
The networks of the filter and the assigned hyper-parameters are the same as the experiments on CelebA, as shown in~\cref{tab:CelebA_filter_network} and~\cref{tab:hp_filter}, respectively.
To construct the \Univ{} dataset, we ensure an even number of landbirds and waterbirds on both land and water backgrounds by utilizing provided pixel-level segmentation masks to extract each bird from its original background and then placing it onto water background or land background sourced from the Places dataset~\cite{Places}. Please see the qualitative results of our method in~\cref{fig:waterbirds} and quantitative comparison in~\cref{tab:Waterbirds&NLP}.

\section{Application on More Modalities (CivilComment-WILDS)}
\label{appsec:modalities}
In the main paper, our experiments include modalities beyond images, \eg Adult dataset, which consists of census data. For NLP, following JTT~\cite{JTT}, we supplement the results on CivilComments-WILDS, an NLP task that classifies online comments as toxic or non-toxic, where target labels are spuriously correlated with mentions of certain demographic identities. 
We assign hyper-parameters of the filter ($\lambda_{mi}$, $\lambda_{pred}$, and $\lambda_{rec}$) to be 5, 5, and 10, respectively. Other hyper-parameters are configured following~\cite{Wilds}. Consequently, we set a maximum of 300 tokens for the token-type embedding of the original text and ensure the same number of tokens for both reconstructed text and filtered text. 
The filter is designed to remove words associated with specific demographic identities. 
The filter networks remain consistent with the Adult experiments, with the exception of adjusting the input dimension to 300, as shown in~\cref{tab:Adult_filter_network}.
To construct the \Univ{} dataset, we build a subset by including the samples from all 16 groups categorized by target (toxic or non-toxic) and all 8 demographic identities. Please see~\cref{tab:Waterbirds&NLP} for the results.

\input{sections/pictures/Waterbirds_tab_fig}

\section{Comparison with Group DRO and DFA}
\label{appsec:DRO_DFA}
Group DRO~\cite{DRO} emphasizes the role of regularization in distributionally robust optimization to prevent overfitting and enhance worst-group performance. Our approach surpasses Group DRO by not demanding target labels during training, which becomes valuable when obtaining target labels is challenging. Thus, we believe our approach is more applicable to address attribute bias.
Besides, DFA~\cite{DFA} introduces feature-level data augmentation to generate diverse bias-conflicting samples by disentangling intrinsic and bias attributes from bias-aligned samples. In contrast, our approach conducts in-place image editing, eliminating the need for augmenting the dataset with additional bias-conflicting samples. This property is advantageous especially for non-binary bias attributes since generating bias-conflicting samples with combinations of target and all other bias attributes, can be time-consuming. 

We supplement the empirical comparison with them on several datasets in~\cref{tab:CelebA_DFA&DRO}. The implementation follows the same guidelines as our main paper. As shown in~\cref{tab:CelebA_DFA&DRO}, Group DRO trade-off average performance with enhanced worst-group performance. One possible reason is that the regularization in Group DRO may overemphasize the worst group during training in extreme bias where the worst group performance is notably compromised.

\input{sections/tables/CelebA_DFA_DRO}

\section{Related Work about Domain-Invariant Representation Learning Approaches}
\label{appsec:domain_invariant}
In the main paper, we provide related work about mutual information-based methods~\cite{learn_not_to_learn_Colored_MNIST,Back_MI,CSAD}, which provides a good basis to draw conclusions about domain shift works. In this section, we cover a more comprehensive list of related work about domain-invariant representation learning approaches.
Ganin \etal~\cite{ganin2016domain} propose Domain-Adversarial Neural Networks (DANN) for domain adaptation, which are trained to be discriminative for the primary task in the source domain and domain-indiscriminate to handle domain shifts. Zhao \etal~\cite{zhao2019learning} introduce an upper bound for conditional shift (the variation in class-conditional distributions of input features between the source and target domains). Albuquerque \etal~\cite{albuquerque2019generalizing} present an adversarial domain generalization method to minimize risk across unseen domains by mapping data to a space where training distributions are indistinguishable while retaining task-relevant information.

\input{sections/tables/IMDB_age_label}
\input{sections/pictures/IMDB_boxplot}
\input{sections/pictures/IMDB_statistics}
\input{sections/pictures/IMDB_EB_comparison+noisy}

\section{Training Details}
\label{appsec:training_details}
In this section, we present the implementation details of our filter and the downstream task classifier on Colored MNIST, CelebA, Adult, and IMDB.
The experiments are conducted on one NVIDIA 1080 Ti GPU. The reference time spent for each experiment is shown in~\cref{tab:time}.

\subsection{Filter Training}
The network architectures of our filter on Colored MNIST are shown in~\cref{tab:CMNIST_filter_network}, the architectures on CelebA and IMDB are shown in~\cref{tab:CelebA_filter_network}, and the architecture on Adult is shown in~\cref{tab:Adult_filter_network}.  
We optimize our filter by Adam optimizer~\cite{Adam} with default settings $(\beta_1=0, \beta_2=0.9)$ in WGAN-GP~\cite{wgan_gp}. The hyper-parameters are shown in~\cref{tab:hp_filter}.

\input{sections/tables/time_statistics}
\input{sections/tables/hp_filter}
\input{sections/tables/hp_task}
\input{sections/tables/hp_SSL}

% Training details.
% Unet.
 
\subsection{Downstream Task Classifier Training}
\label{sec:classifier_training}
For digit classification on Colored MNIST, we use Lenet-5~\cite{lenet} as the backbone classifier.
For facial attribute classification on CelebA and binary sex classification on IMDB, following CSAD~\cite{CSAD}, we use ResNet-18~\cite{ResNet} as the backbone classifier. 
For income prediction on Adult dataset, we use a three-layer MLP as the backbone classifier. The hyper-parameters are shown in~\cref{tab:hp_task}.

\clearpage
\input{sections/tables/CMNIST_filter_network}
\input{sections/tables/Adult_filter_network}
\input{sections/tables/CelebA_filter_network}
\input{sections/tables/CelebA_SSL_network}

%% file: sections/pictures/framework_SSL.tex
\begin{figure}[t]
\begin{center}
  \includegraphics[width=0.65\linewidth]{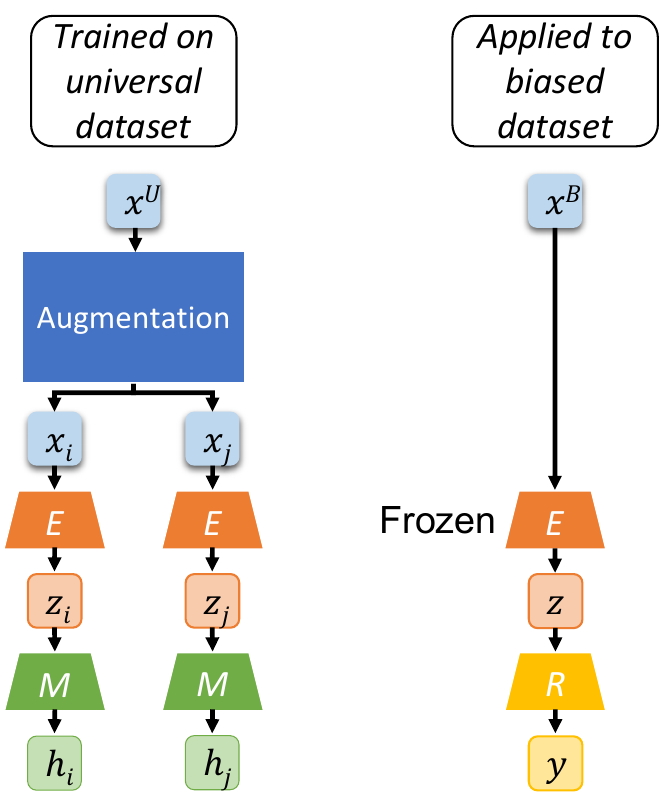}
\end{center}
  \caption{Framework of the proposed approach based on self-supervised learning to utilize universal dataset without any labels. First, the encoder is trained on universal dataset without any labels using contrastive loss. Then, the trained encoder can be applied to the biased dataset for many downstream tasks.} 
  \label{fig:method_SSL}
\end{figure}

%% file: sections/tables/CelebA_SSL.tex
\begin{table}[htbp]
\caption{Performance of the proposed approach based on self-supervised learning (SSL) under extreme bias in CelebA (\emph{TrainEx} training set) to predict \textit{blond hair}. SSL-based approach is trained on \emph{TrainEx} appending the dataset shown in parenthesis as universal dataset. Without universal dataset, the performance of this approach is worse than baseline. With universal dataset, its performance is boosted. \textbf{Bold} for the best results.}
\label{tab:CelebA_SSL}
\centering
\resizebox{0.45\textwidth}{!}{%
\begin{tabular}{lcccc}
\toprule
\multirow{2}{*}{Method} & \multicolumn{2}{c}{Test Accuracy}         & \multicolumn{2}{c}{Mutual Information} \\
\cmidrule(lr){2-3}  \cmidrule(lr){4-5} 
                        & Unbiased ↑          & Bias-conflicting ↑  & $I(Z;A)$ ↓          & $\Delta$ (\%) ↑  \\
                        \midrule
Baseline                & 66.11{\scriptsize $\pm$0.32}          & 33.89{\scriptsize $\pm$0.45}          & 0.57{\scriptsize $\pm$0.01}           & 0.00             \\
SSL                     & 64.24{\scriptsize $\pm$0.54}          & 32.59{\scriptsize $\pm$0.61}          & 0.56{\scriptsize $\pm$0.02}           & 1.75             \\
\midrule
SSL (FFHQ)              & 69.02{\scriptsize $\pm$0.47}          & 42.75{\scriptsize $\pm$0.83}          & 0.51{\scriptsize $\pm$0.02}           & 10.50            \\
SSL (Synthetic)         & \textbf{70.19{\scriptsize $\pm$0.58}} & \textbf{44.23{\scriptsize $\pm$0.92}} & \textbf{0.50{\scriptsize $\pm$0.02}}  & \textbf{12.28}  \\
\bottomrule
\end{tabular}
}
\end{table}

%% file: sections/pictures/CelebA_bounds_supp.tex
\begin{figure*}[htbp]
    \centering % <-- added

      \subfloat[Baseline.]{\includegraphics[width=0.32\linewidth]{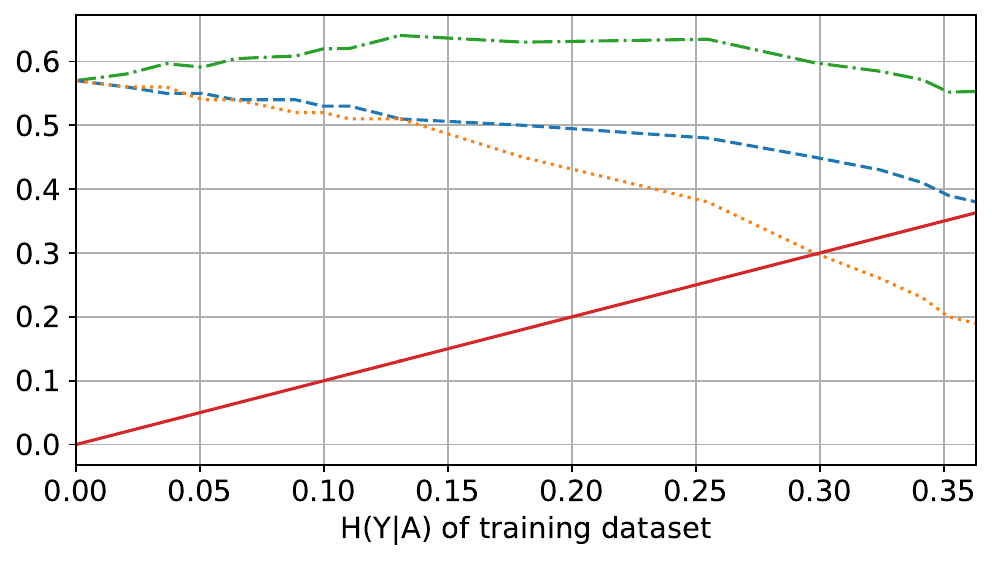}\label{fig:CelebA_baseline_bound_supp}}\quad
      \subfloat[LNL~\cite{learn_not_to_learn_Colored_MNIST}.]{\includegraphics[width=0.32\linewidth]{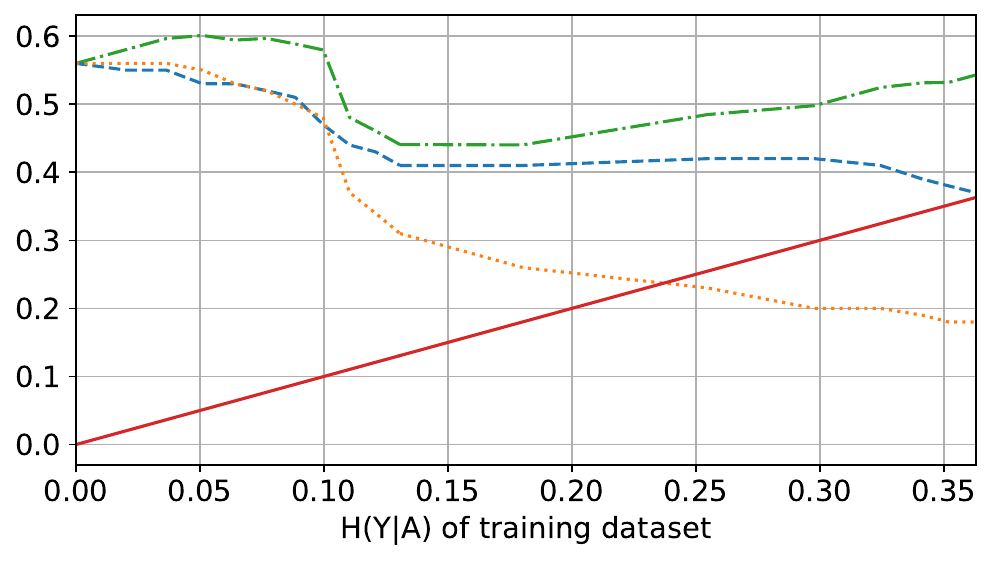}\label{fig:CelebA_LNL_bound_supp}}\quad
      \subfloat[DI~\cite{domain_independent_training}.]{\includegraphics[width=0.32\linewidth]{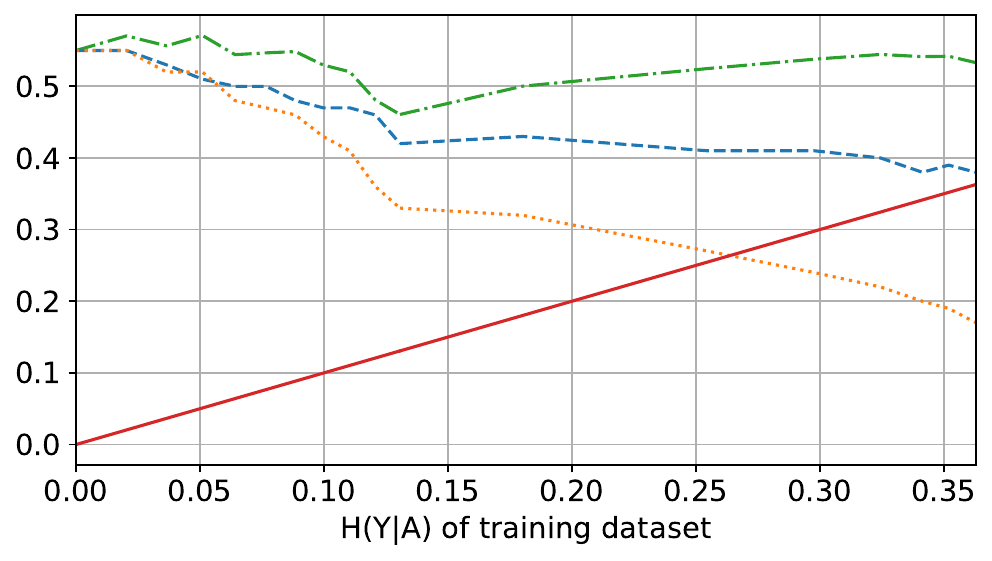}\label{fig:CelebA_DI_bound_supp}}

      \subfloat[LfF~\cite{LfF_CelebA_Bias_conflicting}.]{\includegraphics[width=0.32\linewidth]{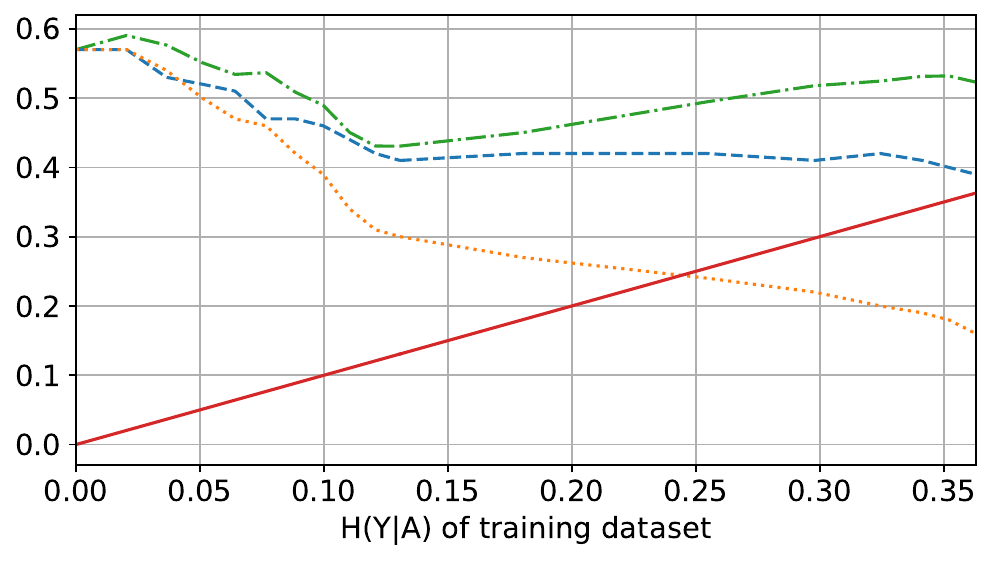}\label{fig:CelebA_LfF_bound_supp}}\quad
      \subfloat[EnD~\cite{End}.]{\includegraphics[width=0.32\linewidth]{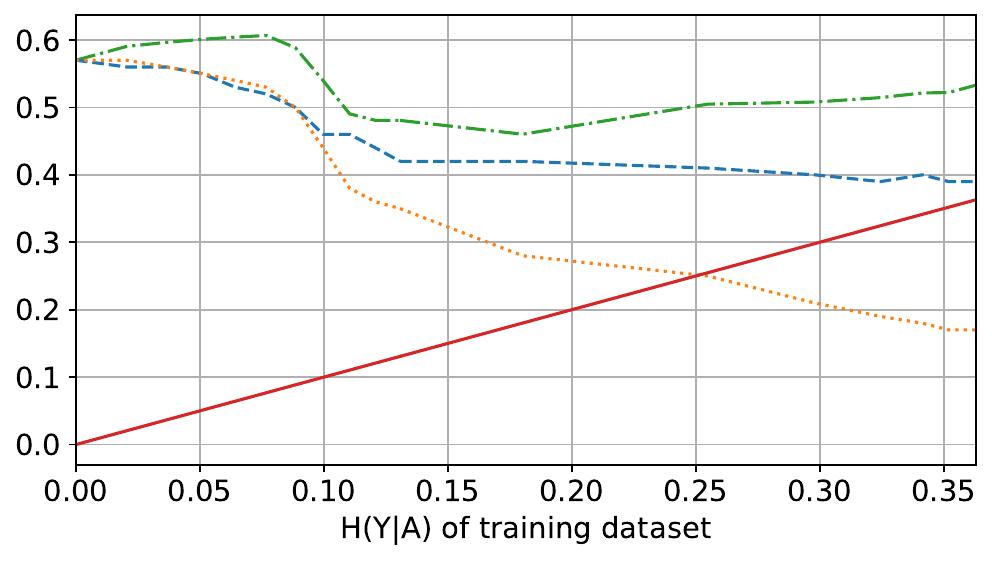}\label{fig:CelebA_EnD_bound_supp}}\quad
      \subfloat[CSAD~\cite{CSAD}.]{\includegraphics[width=0.32\linewidth]{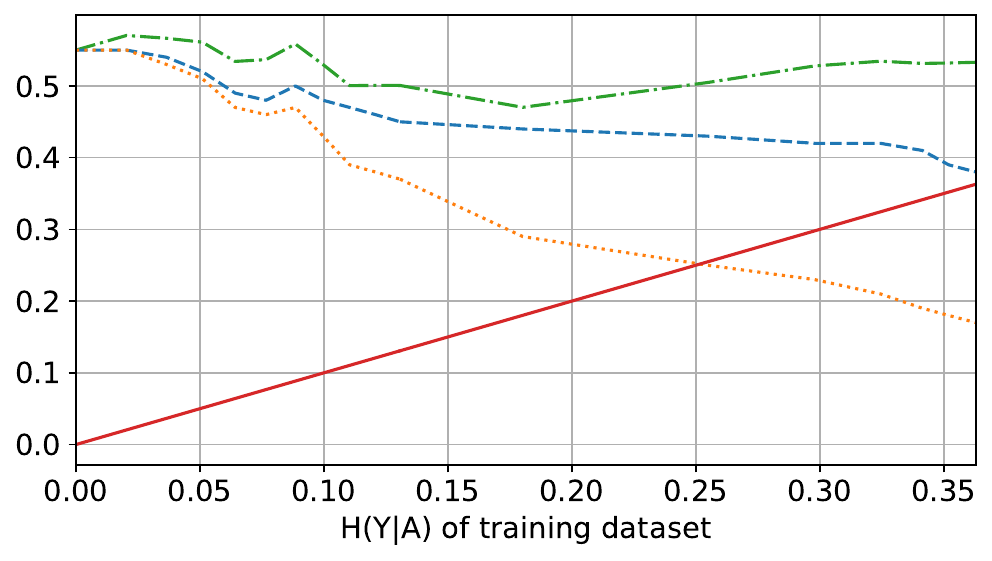}\label{fig:CelebA_CSAD_bound_supp}}

      \subfloat[BCL~\cite{BCL}.]{\includegraphics[width=0.32\linewidth]{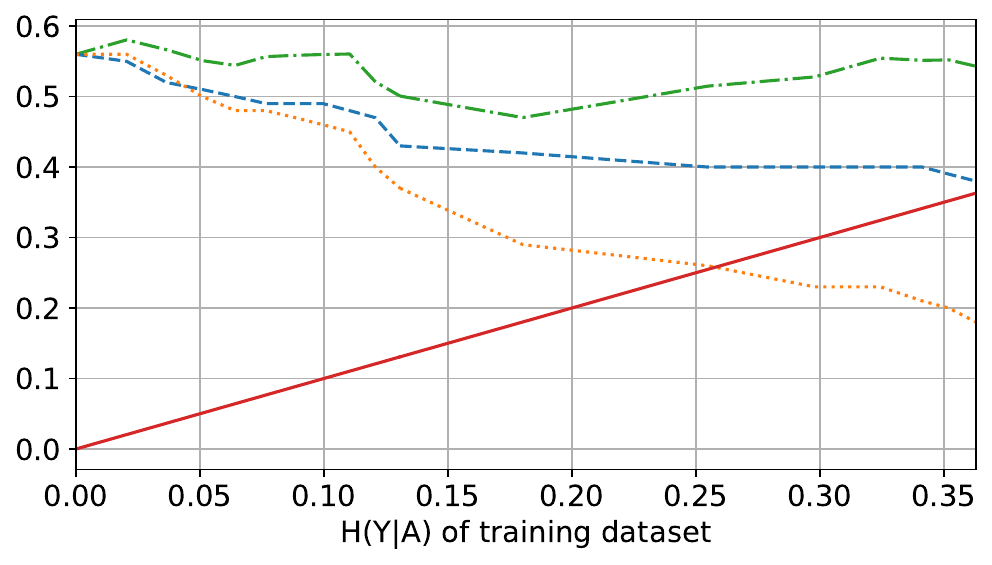}\label{fig:CelebA_BCL_bound_supp}}\quad
      \subfloat[Ours.]{\includegraphics[width=0.32\linewidth]{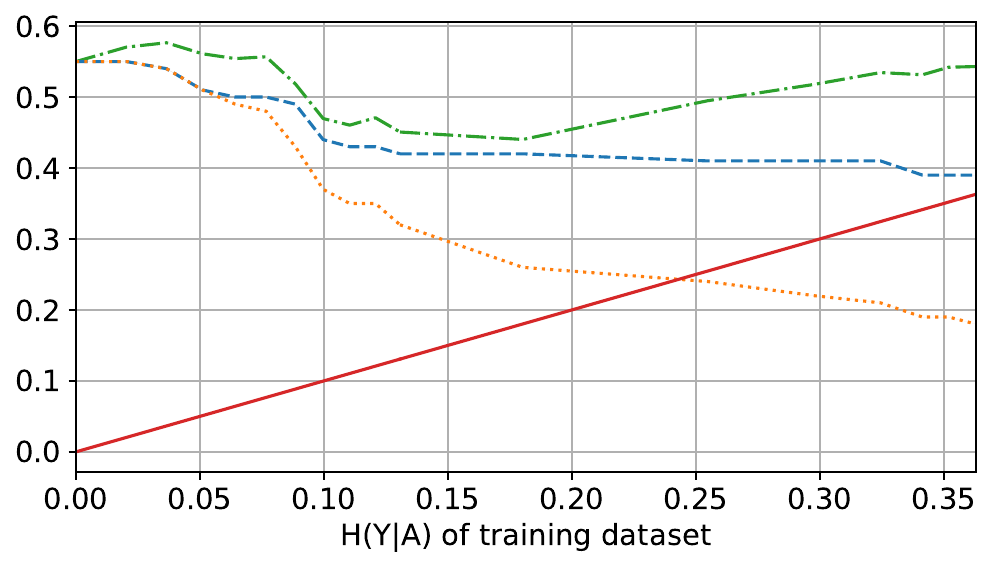}\label{fig:CelebA_Ours_bound_supp}}\quad
      \subfloat[Ours (FFHQ).]{\includegraphics[width=0.32\linewidth]{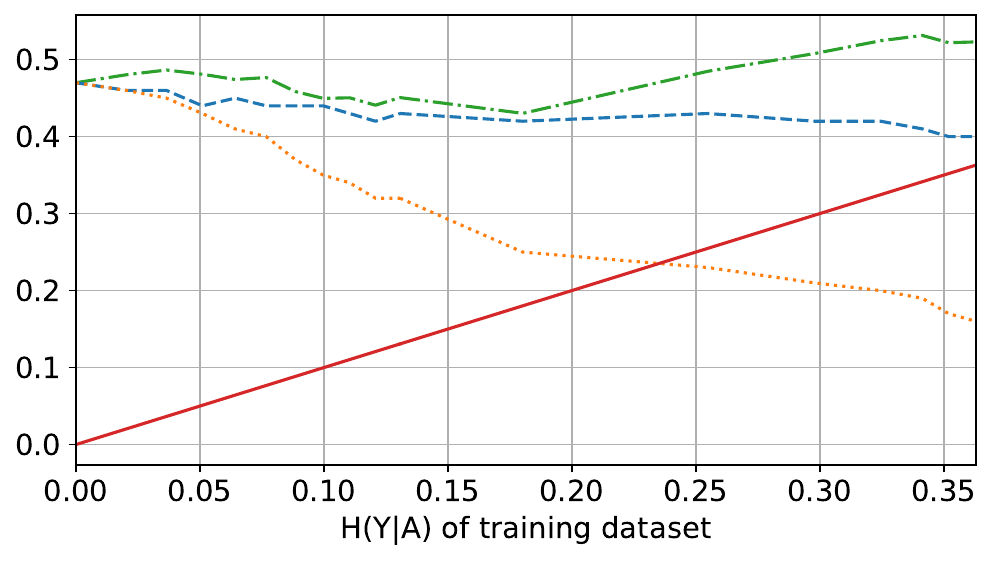}\label{fig:CelebA_OursF_bound_supp}}

      \subfloat[Ours (Synthetic).]{\includegraphics[width=0.54\linewidth]{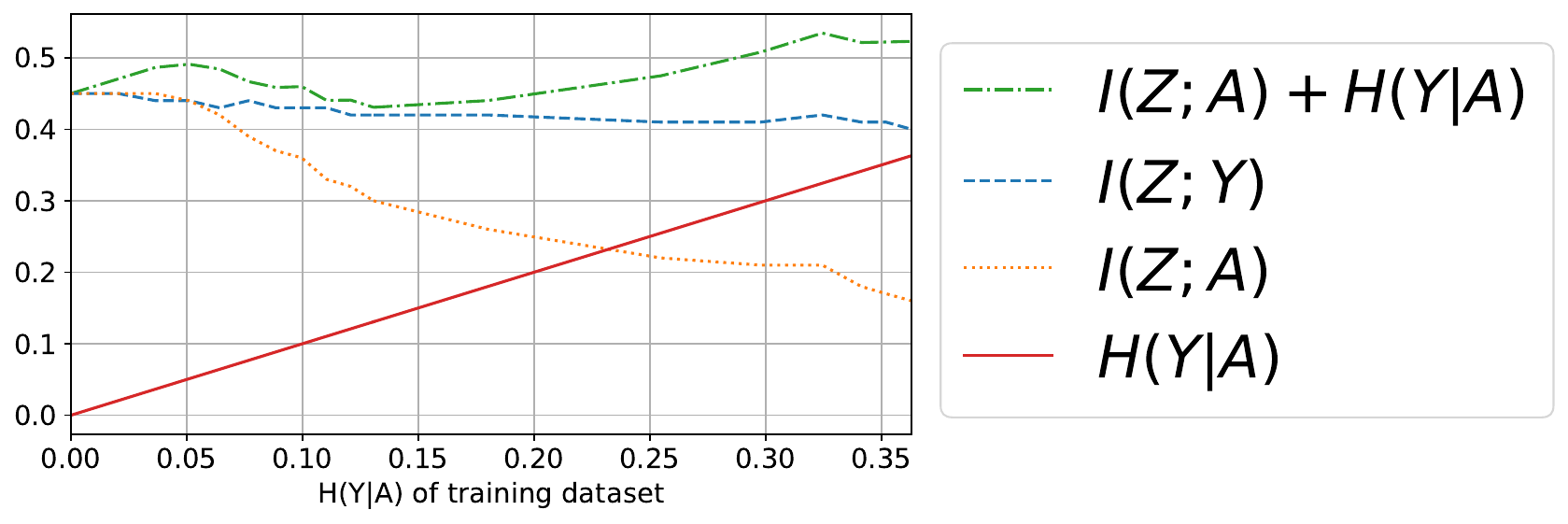}\label{fig:CelebA_OursS_bound_supp}}

\caption{Verifying the bounds in Theorem 1 on CelebA, using mutual information estimation.}
\label{fig:CelebA_bounds_supp}
\end{figure*}

%% file: sections/tables/Pvalue.tex
\begin{table*}[htbp]
% \begin{table*}[h!]
\caption{$p$-value of hypothesis tests over different color variances on Colored MNIST. The significant level of hypothesis tests is 0.05. Color variance is consecutive in its three rows. \textbf{Bold} for the $p$-value on the breaking point of each method. While existing methods are effective over baseline in moderate bias region where color variance exceeds 0.02, they perform close to or worse than baseline in strong bias region and own different breaking points.}

\label{tab:Pvalue}
\centering

\begin{tabular}{lrrrrrrrrr}
\toprule
Color variance           & 0.000                & 0.001                & 0.002                & 0.003                & 0.004                & 0.005                & 0.006                & 0.007                & 0.008                \\
\midrule
LNL~\cite{learn_not_to_learn_Colored_MNIST}          & 0.025                & 0.044                & 0.000                & 0.000                & 0.048                & 0.040                & 0.032                & 0.025                & 0.040                \\
BackMI~\cite{Back_MI}        & 0.034                & 0.020                & 0.043                & 0.010                & 0.023                & 0.037                & 0.042                & \textbf{0.036}       & 0.540                \\
EnD~\cite{End}         & 0.043                & 0.000                & 0.042                & 0.016                & 0.017                & 0.048                & 0.025                & 0.016                & 0.018                \\
CSAD~\cite{CSAD}          & 0.038                & 0.048                & \textbf{0.040}       & 0.723                & 1.000                & 1.000                & 1.000                & 1.000                & 1.000                \\
Ours  & 0.044                & \textbf{0.030}       & 0.684                & 1.000                & 1.000                & 1.000                & 1.000                & 1.000                & 1.000                \\
\midrule
            %   & \multicolumn{1}{l}{} & \multicolumn{1}{l}{} & \multicolumn{1}{l}{} & \multicolumn{1}{l}{} & \multicolumn{1}{l}{} & \multicolumn{1}{l}{} & \multicolumn{1}{l}{} & \multicolumn{1}{l}{} & \multicolumn{1}{l}{} \\
\midrule
Color variance           & 0.009                & 0.010                & 0.011                & 0.012                & 0.013                & 0.014                & 0.015                & 0.016                & 0.017                \\
\midrule
LNL~\cite{learn_not_to_learn_Colored_MNIST}          & \textbf{0.045}       & 0.436                & 0.990                & 1.000                & 0.970                & 0.970                & 1.000                & 1.000                & 1.000                \\
BackMI~\cite{Back_MI}       & 0.960                & 1.000                & 1.000                & 0.990                & 0.960                & 0.990                & 1.000                & 0.980                & 1.000                \\
EnD~\cite{End}          & 0.044                & 0.034                & 0.018                & \textbf{0.040}       & 0.999                & 0.999                & 0.996                & 0.998                & 0.999                \\
CSAD~\cite{CSAD}          & 1.000                & 1.000                & 1.000                & 1.000                & 1.000                & 1.000                & 1.000                & 1.000                & 1.000                \\
Ours & 1.000                & 1.000                & 1.000                & 1.000                & 1.000                & 1.000                & 1.000                & 1.000                & 1.000                \\
\midrule
            %   & \multicolumn{1}{l}{} & \multicolumn{1}{l}{} & \multicolumn{1}{l}{} & \multicolumn{1}{l}{} & \multicolumn{1}{l}{} & \multicolumn{1}{l}{} & \multicolumn{1}{l}{} & \multicolumn{1}{l}{} & \multicolumn{1}{l}{} \\
\midrule
Color variance           & 0.018                & 0.019                & 0.020                & 0.025                & 0.030                & 0.035                & 0.040                & 0.045                & 0.050                \\
\midrule
LNL~\cite{learn_not_to_learn_Colored_MNIST}          & 0.999                & 1.000                & 1.000                & 0.999                & 1.000                & 1.000                & 1.000                & 1.000                & 0.998                \\
BackMI~\cite{Back_MI}       & 0.996                & 1.000                & 1.000                & 1.000                & 1.000                & 1.000                & 1.000                & 1.000                & 1.000                \\
EnD~\cite{End}          & 0.994                & 1.000                & 1.000                & 1.000                & 1.000                & 1.000                & 1.000                & 1.000                & 0.999                \\
CSAD~\cite{CSAD}         & 1.000                & 1.000                & 1.000                & 1.000                & 1.000                & 1.000                & 1.000                & 1.000                & 1.000                \\
Ours  & 1.000                & 1.000                & 1.000                & 1.000                & 1.000                & 1.000                & 1.000                & 1.000                & 0.998  \\             
\bottomrule
\end{tabular}

\end{table*}

%% file: sections/pictures/CelebA_visualizaton_supp.tex
\begin{figure*}[htbp]
\begin{center}
  \includegraphics[width=1\linewidth]{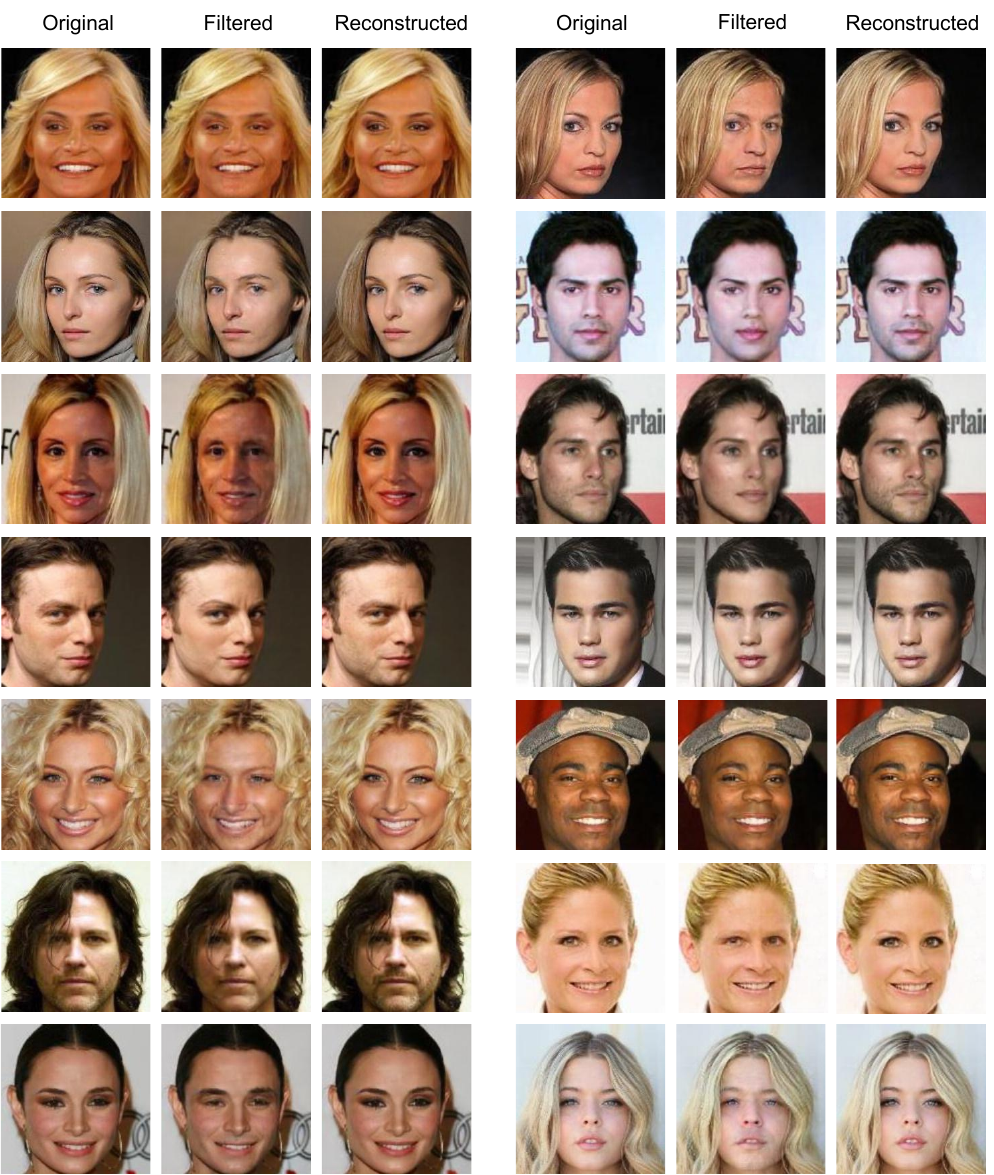}
\end{center}
  \caption{Examples of the filtered images from CelebA. Filtered images, where \emph{sex} is eliminated while other attributes are preserved, can be applied for many downstream tasks.}
\label{fig:CelebA_visualization_supp}
\end{figure*}

%% file: sections/pictures/IMDB_visualization.tex
\begin{figure*}[htbp]
\begin{center}
  \includegraphics[width=1\linewidth]{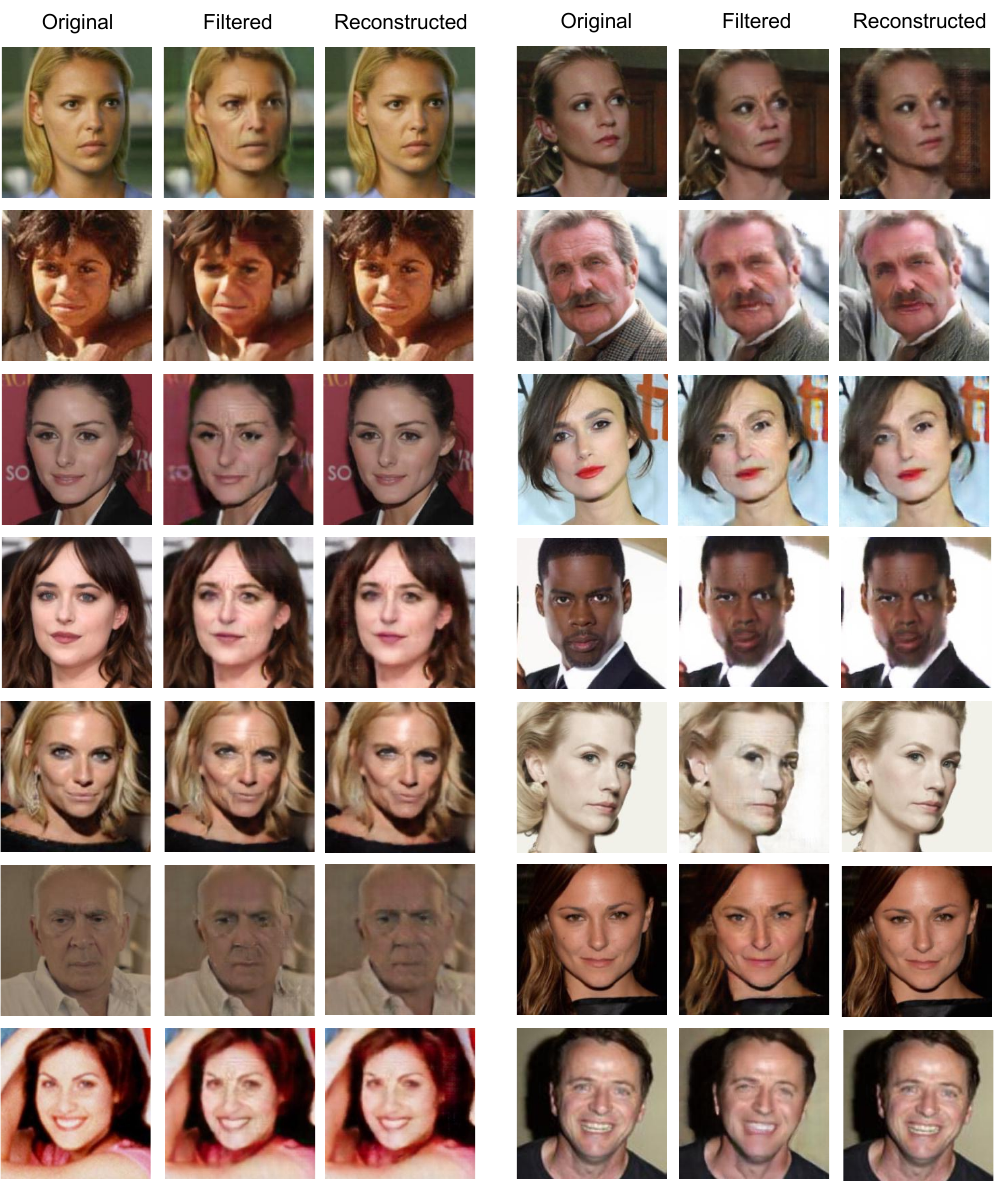}
\end{center}
  \caption{Examples of the filtered images from IMDB. Filtered images, where \emph{age} is eliminated while other attributes are preserved, can be applied for many downstream tasks.}
\label{fig:IMDB_visualization_supp}
\end{figure*}

%% file: sections/pictures/CMNIST_visualization.tex
% \begin{figure}[htbp]
\begin{figure*}[t]
\begin{center}
  \includegraphics[width=0.35\linewidth]{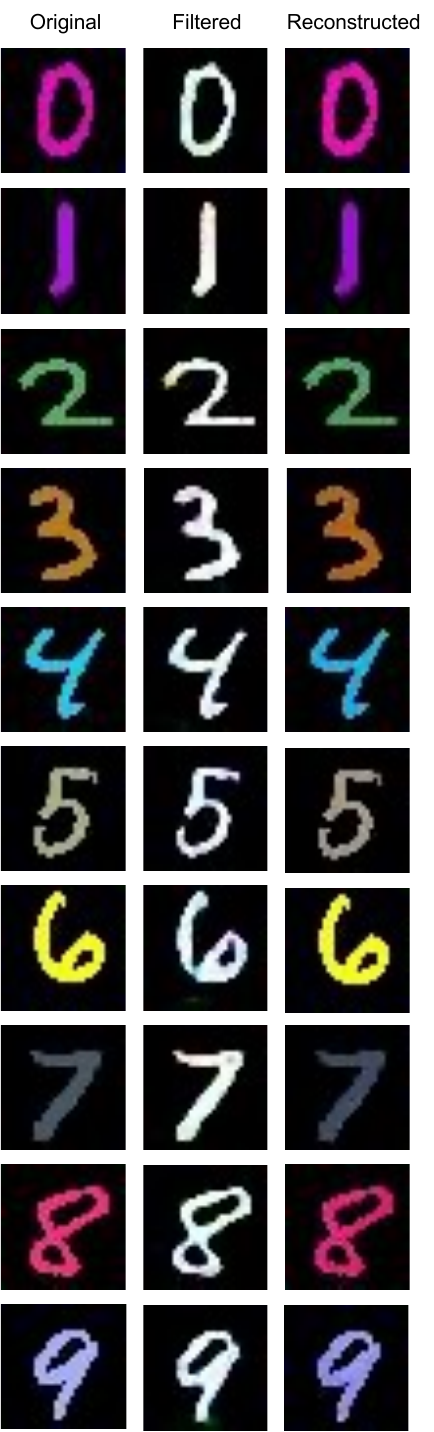}
\end{center}
  \caption{Examples of the filtered images from Colored MNIST.}
\label{fig:CMNIST_visualization}
\end{figure*}

%% file: sections/pictures/CelebA_universal_comparison.tex
\begin{figure*}[htbp]
  \centering

      \subfloat[Accuracy of \emph{Unbiased} testing set.]{\includegraphics[width=0.55\linewidth]{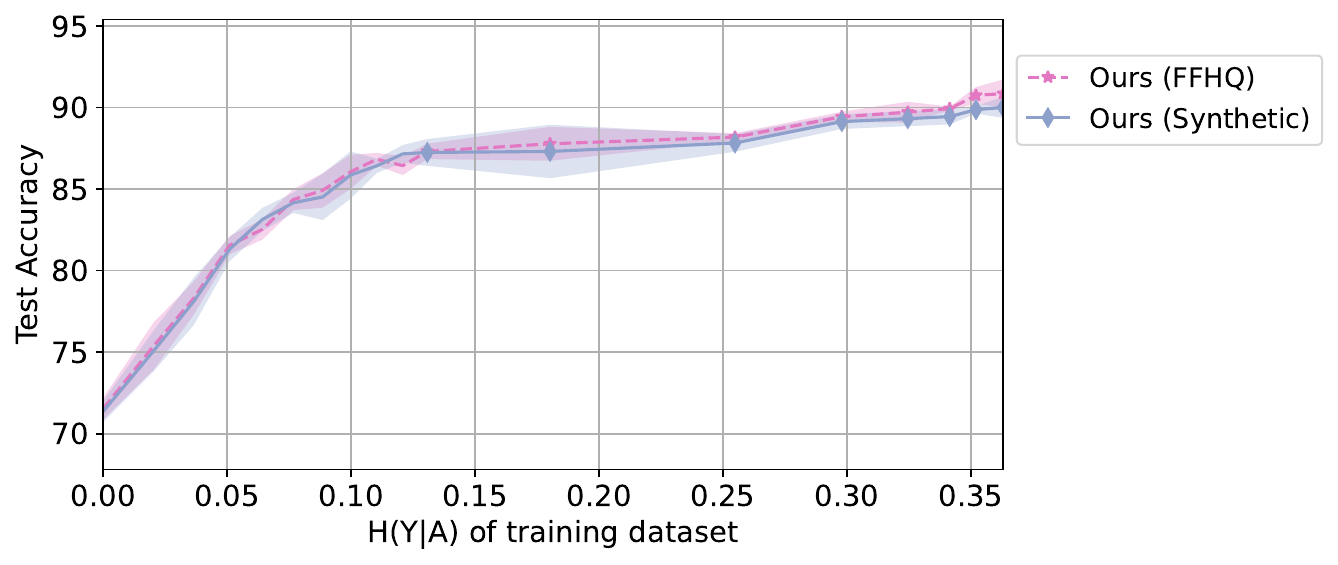} \label{fig:CelebA_Unbiased_Acc_ucomparison}}\quad
      \subfloat[Accuracy of \emph{Bias-conflicting} testing set.]{\includegraphics[width=0.42\linewidth]{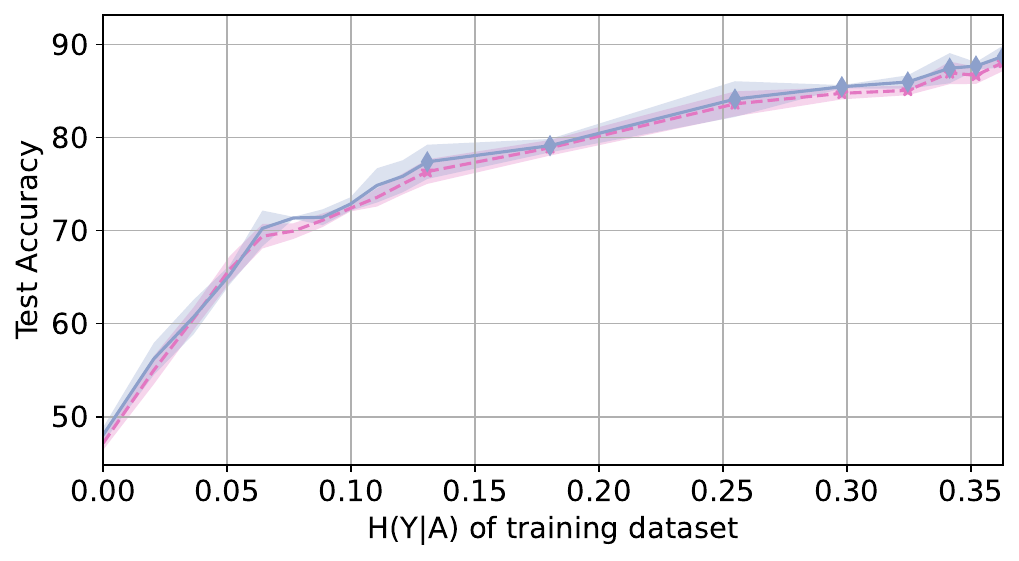}\label{fig:CelebA_Conflicting_Acc_ucomparison}}\quad
      
      \subfloat[Attribute bias $I(Z;A)$ of training set.]{\includegraphics[width=0.42\linewidth]{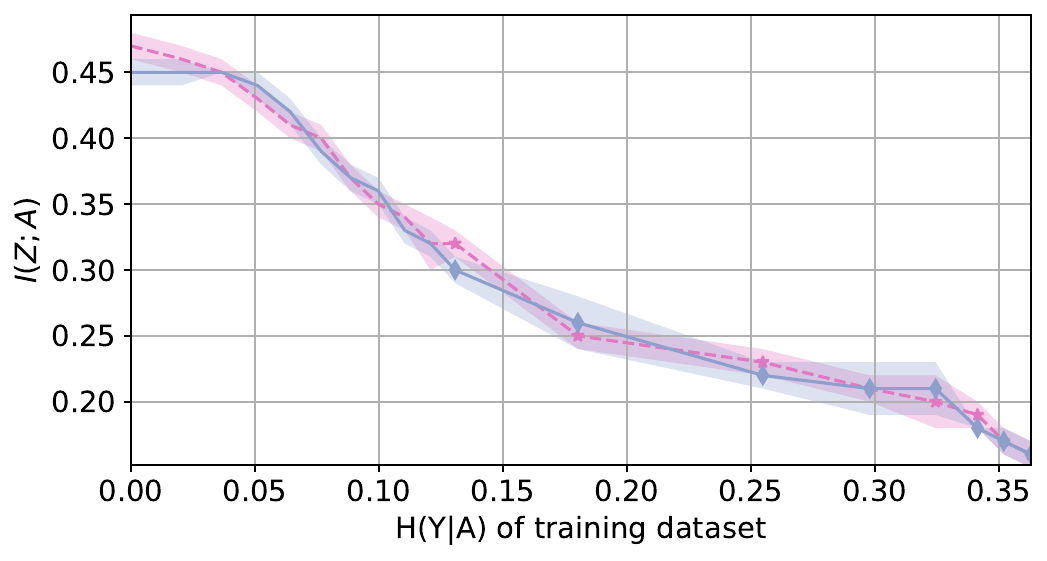}\label{fig:CelebA_IZA_ucomparison}}

  \caption{Performance comparison between our filter trained on the real universal dataset and semi-synthetic universal dataset. Our filter is trained on \emph{TrainEx} appending the dataset shown in parenthesis as universal dataset.}
  \label{fig:CelebA_ucomparison}
\end{figure*}

%% file: sections/tables/CelebA_supp_append_only.tex
\begin{table*}
\caption{Accuracy of our filter trained exclusively on FFHQ and synthetic dataset under extreme bias and moderate bias on all 23 non-sex-related downstream tasks of CelebA.}
\label{tab:CelebA_supp_append_only}
\centering

% \resizebox{0.66\textwidth}{!}{%

\begin{tabular}{lcccc}
\toprule
Trained on  & \multicolumn{2}{c}{Extreme Bias (\emph{TrainEx})} & \multicolumn{2}{c}{Moderate Bias (\emph{TrainOri})} \\
\cmidrule(lr){2-3}  \cmidrule(lr){4-5} 
         Tested on                & Unbiased             & Bias-conflicting    & Unbiased              & Bias-conflicting     \\
                         \midrule
Baseline                & 59.03{\scriptsize $\pm$0.96}        & 21.53{\scriptsize $\pm$1.42}             & 78.08{\scriptsize $\pm$0.82}         & 71.85{\scriptsize $\pm$1.04}              \\
Ours (FFHQ only)        & 61.65{\scriptsize $\pm$1.84}        & 26.84{\scriptsize $\pm$1.76}             & 80.01{\scriptsize $\pm$2.79}         & 78.31{\scriptsize $\pm$3.14}              \\
Ours (Synthetic only)   & 59.93{\scriptsize $\pm$1.43}        & 24.65{\scriptsize $\pm$1.87}             & 79.12{\scriptsize $\pm$1.98}         & 75.72{\scriptsize $\pm$1.62}    \\
\bottomrule
\end{tabular}

% }

\end{table*}

% \cmidrule(lr){2-3}  \cmidrule(lr){4-5} 

%% file: sections/tables/CelebA_all_universal_ori.tex
\begin{table*}[ht!]
% \caption{\rev{Performance under extreme bias in CelebA when training all methods on universal dataset to predict \textit{blond hair}.
% The baseline model trained on extreme bias data (\emph{TrainEx}) is listed for reference.
% All other models are trained on \emph{TrainEx} appending FFHQ. To facilitate the usage of the unlabeled FFHQ for methods other than ours, FFHQ is pseudo-labeled by a classifier pretrained on the original CelebA dataset. Upon training on universal dataset, most methods clearly outperform baseline, with our method performing the best. This highlights the practicality of pseudo-labeling to facilitate the construction of universal dataset in addressing extreme bias.}}

\caption{Effect of pseudo-labeling (trained by the original CelebA dataset) on attribute bias removal methods in CelebA dataset under extreme bias. 
The baseline method trained on the extreme bias dataset (\emph{TrainEx}) is listed for reference. All other methods are trained on the combination of the moderate bias dataset (\emph{TrainOri}) and FFHQ pseudo-labeled by a classifier pretrained on TrainOri. With pseudo-labeling, all methods outperform the baseline, with our proposed method achieving the best performance.}

\label{tab:CelebA_all_universal_ori}
\centering

% \resizebox{0.6\textwidth}{!}{%

\begin{tabular}{lcccc}
\toprule
\multirow{2}{*}{Method} & \multicolumn{2}{c}{Test Accuracy}         & \multicolumn{2}{c}{Mutual Information} \\
\cmidrule(lr){2-3}  \cmidrule(lr){4-5} 
                        & Unbiased ↑          & Bias-conflicting ↑  & $I(Z;A)$ ↓          & $\Delta$ (\%) ↑  \\
                        \midrule
Baseline (TrainEx)      & 66.11{\scriptsize $\pm$0.32}          & 33.89{\scriptsize $\pm$0.45}          & 0.57{\scriptsize $\pm$0.01}           & 0.00             \\
Baseline                & 67.21{\scriptsize $\pm$0.48}          & 39.43{\scriptsize $\pm$0.48}          & 0.41{\scriptsize $\pm$0.00}           & 39.02            \\
\midrule
LNL~\cite{learn_not_to_learn_Colored_MNIST}                     & 72.84{\scriptsize $\pm$0.62}          & 44.85{\scriptsize $\pm$0.78}          & 0.38{\scriptsize $\pm$0.02}           & 33.33            \\
DI~\cite{domain_independent_training}                      & 76.82{\scriptsize $\pm$0.84}          & 50.13{\scriptsize $\pm$0.91}          & 0.33{\scriptsize $\pm$0.01}           & 42.11            \\
LfF~\cite{LfF_CelebA_Bias_conflicting}                     & 75.59{\scriptsize $\pm$0.15}          & 48.21{\scriptsize $\pm$1.71}          & 0.35{\scriptsize $\pm$0.02}           & 38.60            \\
EnD~\cite{End}                     & 72.92{\scriptsize $\pm$0.34}          & 48.52{\scriptsize $\pm$0.68}          & 0.37{\scriptsize $\pm$0.02}           & 35.09            \\
CSAD~\cite{CSAD}                    & 75.48{\scriptsize $\pm$0.63}          & 49.42{\scriptsize $\pm$1.38}          & 0.33{\scriptsize $\pm$0.01}           & 42.11            \\
BCL~\cite{BCL}                     & 77.61{\scriptsize $\pm$0.28}          & 50.34{\scriptsize $\pm$1.25}          & 0.34{\scriptsize $\pm$0.02}           & 40.35            \\
\midrule
Ours                    & \textbf{78.79{\scriptsize $\pm$0.39}} & \textbf{52.72{\scriptsize $\pm$0.56}} & \textbf{0.32{\scriptsize $\pm$0.01}}  & \textbf{43.86}  \\
\bottomrule
\end{tabular}

% }
\end{table*}

%% file: sections/tables/CelebA_all_universal_ori_increment.tex
\begin{table*}[ht!]
\caption{Comparison between training all methods on extreme bias data and universal dataset (pseudo-labeled by the classifier pretrained on \emph{TrainOri}) in CelebA. Universal dataset is constructed by appending extreme bias data (\emph{TrainEx}) in CelebA with FFHQ. The accuracy on \emph{Unbiased} testing set to predict \emph{blond hair} is reported. FFHQ is pseudo-labeled by the classifier pretrained on TrainOri to facilitate the usage of other methods. Breaking point is represented with $H(Y|A)$. $\Delta$ indicates the difference in testing accuracy when trained on TrainEx alone versus TrainEx appending FFHQ. Our method yields the lowest breaking point and outperforms other methods.}
\label{tab:CelebA_all_universal_ori_incremental}
\centering

% \resizebox{0.6\textwidth}{!}{%

\begin{tabular}{lcccc}
\toprule
Method    & \emph{TrainEx} ↑    & \emph{TrainEx+FFHQ} ↑ & $\Delta$ (Acc.) ↑ & \multicolumn{1}{c}{Breaking point ($H(Y|A)$) ↓} \\
\midrule
Baseline  & 66.11{\scriptsize $\pm$0.32}          & 67.21{\scriptsize $\pm$0.48}            & 1.10              & -                                               \\
% \midrule
LNL~\cite{learn_not_to_learn_Colored_MNIST}  & 64.81{\scriptsize $\pm$0.17}          & 72.84{\scriptsize $\pm$0.62}            & 8.03              & 0.0365                                          \\
DI~\cite{domain_independent_training}   & 66.83{\scriptsize $\pm$0.44}          & 76.82{\scriptsize $\pm$0.84}            & 9.99              & 0.0204                                          \\
LfF~\cite{LfF_CelebA_Bias_conflicting}  & 64.43{\scriptsize $\pm$0.43}          & 75.59{\scriptsize $\pm$0.15}            & 11.16             & 0.0204                                          \\
EnD~\cite{End}  & 66.53{\scriptsize $\pm$0.23}          & 72.92{\scriptsize $\pm$0.34}            & 6.39              & 0.0642                                          \\
CSAD~\cite{CSAD} & 63.24{\scriptsize $\pm$2.36}          & 75.48{\scriptsize $\pm$0.63}            & 12.24             & 0.0204                                          \\
BCL~\cite{BCL}  & 65.30{\scriptsize $\pm$0.51}          & 77.61{\scriptsize $\pm$0.28}            & 12.31             & 0.0204                                          \\
Ours      & \textbf{66.31{\scriptsize $\pm$0.26}} & \textbf{78.79{\scriptsize $\pm$0.39}}   & \textbf{12.48}    & 0.0204                                         \\
\bottomrule
\end{tabular}

% }
\end{table*}

%% file: sections/tables/CelebA_all_universal_increment.tex
\begin{table*}[ht!]
\caption{Comparison between training all methods on extreme bias data and universal dataset (pseudo-labeled by the
classifier pretrained on \emph{TrainEx}) in CelebA. Universal dataset is constructed by appending extreme bias data (\emph{TrainEx}) in CelebA with FFHQ. The accuracy on \emph{Unbiased} testing set to predict \emph{blond hair} is reported. FFHQ is pseudo-labeled by the classifier pretrained on TrainEx to facilitate the usage of other methods. Breaking point is represented with $H(Y|A)$. $\Delta$ indicates the difference in testing accuracy when trained on TrainEx alone versus TrainEx appending FFHQ. Our method yields the lowest breaking point and outperforms other methods.}
\label{tab:CelebA_all_universal_incremental}
\centering

% \resizebox{0.6\textwidth}{!}{%

\begin{tabular}{lcccc}
\toprule
Method   & \emph{TrainEx} ↑    & \emph{TrainEx+FFHQ} ↑ & $\Delta$ (Acc.) ↑ & \multicolumn{1}{c}{Breaking point ($H(Y|A)$) ↓} \\
\midrule
Baseline & 66.11{\scriptsize $\pm$0.32}          & 67.02{\scriptsize $\pm$0.78}            & 0.91              & -                                               \\
% \midrule
LNL~\cite{learn_not_to_learn_Colored_MNIST}      & 64.81{\scriptsize $\pm$0.17}          & 67.47{\scriptsize $\pm$0.34}            & 2.66              & 0.0365                                          \\
DI~\cite{domain_independent_training}       & 66.83{\scriptsize $\pm$0.44}          & 70.61{\scriptsize $\pm$0.58}            & 3.78              & 0.0204                                          \\
LfF~\cite{LfF_CelebA_Bias_conflicting}      & 64.43{\scriptsize $\pm$0.43}          & 69.42{\scriptsize $\pm$0.61}            & 4.99              & 0.0204                                          \\
EnD~\cite{End}      & 66.53{\scriptsize $\pm$0.23}          & 67.65{\scriptsize $\pm$0.34}            & 1.12              & 0.0642                                          \\
CSAD~\cite{CSAD}     & 63.24{\scriptsize $\pm$2.36}          & 68.18{\scriptsize $\pm$0.16}            & 4.94              & 0.0204                                          \\
BCL~\cite{BCL}      & 65.30{\scriptsize $\pm$0.51}          & 70.43{\scriptsize $\pm$0.71}            & 5.13              & 0.0204                                          \\
\midrule
Ours     & \textbf{66.31{\scriptsize $\pm$0.26}} & \textbf{72.05{\scriptsize $\pm$0.86}}   & \textbf{5.74}     & 0.0204                                          \\
\bottomrule
\end{tabular}

% }
\end{table*}

%% file: sections/tables/CelebA_GAN_based_HM.tex
\begin{table*}[htbp]
% \caption{Performance of GAN-based methods under extreme bias and moderate bias in CelebA to predict \textit{heavy makeup}. Our filter is trained on \emph{TrainEx} appending the dataset shown in parenthesis as universal dataset. The proposed method, despite being trained on only half the size of the training set compared to other methods, performs better or on par with them.}

\caption{Accuracy of GAN-based methods under extreme bias and moderate bias in CelebA to predict \textit{heavy makeup}. For our method, we report inside parentheses the partially observable Universal distribution used in addition to TrainEx for training its filter. Our method performs better or on-par with GAN-based methods with only half the size of classifier training set.}

\label{tab:CelebA_GAN_HeavyMakeup}
\centering
\resizebox{1\textwidth}{!}{%
\begin{tabular}{lcccccc}
\toprule
\multirow{2}{*}{Method}  & \multicolumn{3}{c}{Extreme Bias Training Set (\emph{TrainEx})}              & \multicolumn{3}{c}{Moderate Bias Training Set (\emph{TrainOri})}         \\
\cmidrule(lr){2-4}  \cmidrule(lr){5-7} 
                         & Classifier training size ↓ & Unbiased ↑           & Bias-conflicting ↑   & Classifier training size ↓ & Unbiased ↑           & Bias-conflicting ↑   \\
                         \midrule
Baseline                 & 130410                   & 59.50{\scriptsize $\pm$0.12}          & 30.09{\scriptsize $\pm$0.53}          & 162770                   & 62.00{\scriptsize $\pm$0.02}          & 33.75{\scriptsize $\pm$0.28}          \\
CGN~\cite{CGN}                      & 130410$\times$2                 & 58.85{\scriptsize $\pm$1.56}          & 29.57{\scriptsize $\pm$1.34}          & 162770$\times$2                 & 70.98{\scriptsize $\pm$0.76}          & 48.12{\scriptsize $\pm$1.72}          \\
CAMEL~\cite{Camel}                    & 130410$\times$2                 & 63.34{\scriptsize $\pm$1.23}          & 30.74{\scriptsize $\pm$1.53}          & 162770$\times$2                 & 74.37{\scriptsize $\pm$1.32}          & 53.53{\scriptsize $\pm$1.25}          \\
BiaSwap~\cite{BiaSwap}                  & 130410$\times$2                 & 64.93{\scriptsize $\pm$0.82}          & 30.88{\scriptsize $\pm$1.96}          & 162770$\times$2                 & 77.31{\scriptsize $\pm$1.27}          & 55.51{\scriptsize $\pm$1.63}          \\

GAN-Debiasing~\cite{GAN_Debiasing_hat_glasses_correlation}            & 130410$\times$2                 & 62.98{\scriptsize $\pm$1.46}          & 30.15{\scriptsize $\pm$1.78}          & 162770$\times$2                 & 75.64{\scriptsize $\pm$1.72}          & 54.47{\scriptsize $\pm$1.52}          \\

Ours           & 130410                   & 64.91{\scriptsize $\pm$1.32}          & 31.36{\scriptsize $\pm$2.43}          & 162770                   & 76.82{\scriptsize $\pm$1.83}          & 54.09{\scriptsize $\pm$3.04}          \\
\midrule
Ours (FFHQ)      & 130410                   & \textbf{71.59{\scriptsize $\pm$0.81}} & 38.64{\scriptsize $\pm$1.57}          & 162770                   & 79.59{\scriptsize $\pm$1.73}          & 57.58{\scriptsize $\pm$1.83}          \\
Ours (Synthetic) & 130410                   & 69.32{\scriptsize $\pm$0.95}          & \textbf{40.91{\scriptsize $\pm$1.54}} & 162770                   & \textbf{80.73{\scriptsize $\pm$0.71}} & \textbf{58.27{\scriptsize $\pm$1.67}} \\
\bottomrule
\end{tabular}
}
\end{table*}

%% file: sections/tables/Adult_all_universal_ori.tex
\begin{table*}[ht!]

\caption{Effect of pseudo-labeling (trained by the original Adult dataset) on attribute bias removal methods in Adult dataset under extreme bias. 
The baseline method trained on the extreme bias dataset (\emph{TrainEx}) is listed for reference. All other methods are trained on the combination of the moderate bias dataset (\emph{TrainOri}) and bias-conflicting samples pseudo-labeled by a classifier pretrained on TrainOri. With pseudo-labeling, all methods outperform the baseline, with our proposed method achieving the best performance.}

\label{tab:Adult_all_universal_ori}
\centering

% \resizebox{0.6\textwidth}{!}{%

\begin{tabular}{lcccc}
\toprule
\multirow{2}{*}{Method} & \multicolumn{2}{c}{Test Accuracy}         & \multicolumn{2}{c}{Mutual Information} \\
\cmidrule(lr){2-3}  \cmidrule(lr){4-5} 
                        & Unbiased ↑          & Bias-conflicting ↑  & $I(Z;A)$ ↓          & $\Delta$ (\%) ↑  \\
                        \midrule
Baseline (TrainEx)      & 50.59{\scriptsize $\pm$0.54}          & 1.19{\scriptsize $\pm$0.83}           & 0.69{\scriptsize $\pm$0.00}           & 0.00             \\
Baseline                & 62.12{\scriptsize $\pm$0.41}          & 28.82{\scriptsize $\pm$0.91}          & 0.49{\scriptsize $\pm$0.00}           & 28.99            \\
\midrule
LNL~\cite{learn_not_to_learn_Colored_MNIST}                     & 77.32{\scriptsize $\pm$0.29}          & 65.38{\scriptsize $\pm$0.60}          & 0.34{\scriptsize $\pm$0.02}           & 50.72            \\
DI~\cite{domain_independent_training}                      & 78.20{\scriptsize $\pm$0.93}          & 67.12{\scriptsize $\pm$0.81}          & 0.28{\scriptsize $\pm$0.01}           & 59.42            \\
LfF~\cite{LfF_CelebA_Bias_conflicting}                     & 77.83{\scriptsize $\pm$0.46}          & 66.91{\scriptsize $\pm$0.77}          & 0.31{\scriptsize $\pm$0.01}           & 55.07            \\
EnD~\cite{End}                     & 79.38{\scriptsize $\pm$1.31}          & 68.89{\scriptsize $\pm$1.29}          & 0.25{\scriptsize $\pm$0.01}           & 63.77            \\
\bottomrule
\end{tabular}

% }
\end{table*}

%% file: sections/tables/CelebA_supp.tex
\begin{table*}
\caption{Performance of attribute bias removal methods under \emph{extreme} attribute bias in CelebA (\emph{TrainEx} training set). \textbf{Bold} for the best results.}
\label{tab:CelebA_supp}
\centering
\resizebox{0.89\textwidth}{!}{%

\begin{tabular}{lcccccccc}
\toprule
\multirow{2}{*}{Method}  & \multicolumn{2}{c}{BagsUnderEyes}                    & \multicolumn{2}{c}{Bangs}                   & \multicolumn{2}{c}{BigLips}                 & \multicolumn{2}{c}{BigNose}                 \\
 \cmidrule(lr){2-3}  \cmidrule(lr){4-5}  \cmidrule(lr){6-7} \cmidrule(lr){8-9}
                         & Unbiased             & Bias-conflicting     & Unbiased             & Bias-conflicting     & Unbiased             & Bias-conflicting     & Unbiased             & Bias-conflicting     \\
\midrule
Baseline                 & 52.22{\scriptsize $\pm$1.06}           & 6.65{\scriptsize $\pm$1.37}            & 63.75{\scriptsize $\pm$1.04}           & 27.61{\scriptsize $\pm$1.41}           & 51.78{\scriptsize $\pm$0.97}           & 4.92{\scriptsize $\pm$1.35}            & 52.40{\scriptsize $\pm$0.96}           & 5.44{\scriptsize $\pm$1.42}            \\
LNL~\cite{learn_not_to_learn_Colored_MNIST}                      & 49.05{\scriptsize $\pm$0.31}           & 3.94{\scriptsize $\pm$0.53}            & 60.41{\scriptsize $\pm$0.27}           & 28.02{\scriptsize $\pm$0.60}           & 48.51{\scriptsize $\pm$0.21}           & 2.26{\scriptsize $\pm$0.58}            & 48.36{\scriptsize $\pm$0.32}           & 2.54{\scriptsize $\pm$0.52}            \\
DI~\cite{domain_independent_training}                       & 53.01{\scriptsize $\pm$0.37}           & 7.14{\scriptsize $\pm$0.34}            & 64.35{\scriptsize $\pm$0.49}           & 31.18{\scriptsize $\pm$0.52}           & 52.55{\scriptsize $\pm$0.44}           & 5.36{\scriptsize $\pm$0.46}            & 53.23{\scriptsize $\pm$0.44}           & 5.91{\scriptsize $\pm$0.39}            \\
\bottomrule

\end{tabular}

}%
\end{table*}

%% file: sections/tables/CelebA_supp_moderate.tex
\begin{table*}
\caption{Performance of attribute bias removal methods under \emph{moderate} attribute bias in CelebA (\emph{TrainOri} training set). \textbf{Bold} for the best results.}
\label{tab:CelebA_supp_moderate}
\centering
\resizebox{0.89\textwidth}{!}{%

\begin{tabular}{lcccccccc}
\toprule
\multirow{2}{*}{Method}   & \multicolumn{2}{c}{BagsUnderEyes}                    & \multicolumn{2}{c}{Bangs}                   & \multicolumn{2}{c}{BigLips}                 & \multicolumn{2}{c}{BigNose}                 \\
 \cmidrule(lr){2-3}  \cmidrule(lr){4-5}  \cmidrule(lr){6-7} \cmidrule(lr){8-9}
                          & Unbiased             & Bias-conflicting     & Unbiased             & Bias-conflicting     & Unbiased             & Bias-conflicting     & Unbiased             & Bias-conflicting     \\
                          \midrule
Baseline                  & 76.87{\scriptsize $\pm$0.80}           & 71.97{\scriptsize $\pm$0.99}           & 91.06{\scriptsize $\pm$0.78}           & 91.39{\scriptsize $\pm$1.11}           & 62.00{\scriptsize $\pm$0.78}           & 58.22{\scriptsize $\pm$1.02}           & 81.66{\scriptsize $\pm$0.73}           & 70.04{\scriptsize $\pm$0.98}           \\
LNL~\cite{learn_not_to_learn_Colored_MNIST}                       & 71.85{\scriptsize $\pm$0.66}           & 76.85{\scriptsize $\pm$1.17}           & 92.28{\scriptsize $\pm$0.82}           & 90.33{\scriptsize $\pm$1.27}           & 63.16{\scriptsize $\pm$0.80}           & 58.13{\scriptsize $\pm$1.29}           & 81.22{\scriptsize $\pm$0.78}           & 71.68{\scriptsize $\pm$1.21}           \\
DI~\cite{domain_independent_training}                        & 74.24{\scriptsize $\pm$0.55}           & \textbf{78.27{\scriptsize $\pm$0.52}}           & 94.72{\scriptsize $\pm$0.55}           & 91.78{\scriptsize $\pm$0.44}           & 65.57{\scriptsize $\pm$0.57}           & 59.58{\scriptsize $\pm$0.35}           & 83.64{\scriptsize $\pm$0.45}           & 73.09{\scriptsize $\pm$0.45}           \\
\bottomrule

\end{tabular}

%  and \underline{underline} for the second best results.

}%
\end{table*}

%% file: sections/tables/IMDB.tex
\begin{table}[ht]
\caption{Performance of attribute bias removal methods in IMDB. \textbf{Bold} for the best results.}
\label{tab:IMDB}
\centering

% \resizebox{0.7\textwidth}{!}{%

\begin{tabular}{lcccc}
\toprule
Trained on & \multicolumn{2}{c}{EB1}    & \multicolumn{2}{c}{EB2}    \\
\cmidrule(lr){2-3}  \cmidrule(lr){4-5} 
                        Tested on & EB2               & Test              & EB1               & Test              \\
                        \midrule
Baseline                & 59.9{\scriptsize $\pm$0.0}          & 84.4{\scriptsize $\pm$0.1}          & 57.8{\scriptsize $\pm$0.0}          & 69.8{\scriptsize $\pm$0.0}          \\
BlindEye~\cite{BlindEye_IMDB_eb}                & 63.7{\scriptsize $\pm$1.8}          & 85.6{\scriptsize $\pm$1.1}          & 57.3{\scriptsize $\pm$1.6}          & 69.9{\scriptsize $\pm$1.3}          \\
LNL~\cite{learn_not_to_learn_Colored_MNIST}                     & 68.0{\scriptsize $\pm$2.0}          & 86.7{\scriptsize $\pm$0.9}          & 64.2{\scriptsize $\pm$1.1}          & 74.5{\scriptsize $\pm$1.1}          \\
BackMI~\cite{Back_MI}                  & 69.1{\scriptsize $\pm$0.0}          & 87.6{\scriptsize $\pm$0.0}          & 65.1{\scriptsize $\pm$0.0}          & 76.2{\scriptsize $\pm$0.0}          \\
EnD~\cite{End}                     & 65.5{\scriptsize $\pm$0.8}          & 87.2{\scriptsize $\pm$0.3}          & 69.4{\scriptsize $\pm$2.0}          & 78.2{\scriptsize $\pm$1.2}          \\
CSAD~\cite{CSAD}                    & 70.4{\scriptsize $\pm$2.8}          & 87.0{\scriptsize $\pm$1.1}          & 68.1{\scriptsize $\pm$1.8}          & 78.7{\scriptsize $\pm$1.9}          \\
Ours             & 73.1{\scriptsize $\pm$2.3}          & 87.6{\scriptsize $\pm$0.3}          & 69.4{\scriptsize $\pm$1.2}          & 78.6{\scriptsize $\pm$0.7}          \\
\midrule
Ours (Universal)   & \textbf{75.7{\scriptsize $\pm$1.8}} & \textbf{88.3{\scriptsize $\pm$0.5}} & \textbf{70.3{\scriptsize $\pm$2.3}} & \textbf{79.2{\scriptsize $\pm$1.0}} \\
\bottomrule
\end{tabular}
% }

\end{table}

%% file: sections/pictures/Waterbirds_tab_fig.tex
\begin{figure*}[ht!]
% \begin{figure*}[H]
% \begin{figure}[h!]
    \begin{minipage}{0.48\textwidth}
        \captionsetup{type=table}
\caption{Average and worst-group test accuracies on Waterbirds and CivilComments-WILDS.}
\label{tab:Waterbirds&NLP}
\centering
\resizebox{1\textwidth}{!}{%
\begin{tabular}{lcccc}
\toprule
\multirow{2}{*}{Model} & \multicolumn{2}{c}{Waterbirds}   & \multicolumn{2}{c}{CivilComments-WILDS} \\
\cmidrule(lr){2-3}  \cmidrule(lr){4-5} 
                       & Avg        & Worst-group         & Avg            & Worst-group            \\
                       \midrule
Baseline               & 97.26{\scriptsize $\pm$0.20} & 72.60{\scriptsize $\pm$1.90}          & 92.64{\scriptsize $\pm$0.52}     & 57.42{\scriptsize $\pm$0.73}             \\
LfF~\cite{LfF_CelebA_Bias_conflicting}                    & 91.22{\scriptsize $\pm$0.85} & 78.04{\scriptsize $\pm$1.83}          & 92.52{\scriptsize $\pm$0.91}     & 58.81{\scriptsize $\pm$1.23}             \\
Group DRO~\cite{DRO}                    & 93.50{\scriptsize $\pm$0.30} & 91.40{\scriptsize $\pm$1.10}          & 88.91{\scriptsize $\pm$0.28}     & 69.94{\scriptsize $\pm$0.52}             \\
JTT~\cite{JTT}                    & 93.34{\scriptsize $\pm$0.56} & 86.72{\scriptsize $\pm$0.97}          & 91.14{\scriptsize $\pm$0.34}     & 69.36{\scriptsize $\pm$0.89}             \\
\midrule
Ours                   & 93.37{\scriptsize $\pm$0.81} & 92.06{\scriptsize $\pm$1.58}          & 91.26{\scriptsize $\pm$0.95}     & 69.51{\scriptsize $\pm$0.71}             \\
Ours (Universal)       & 94.24{\scriptsize $\pm$0.92} & \textbf{93.21{\scriptsize $\pm$1.43}} & 92.42{\scriptsize $\pm$1.43}     & \textbf{70.25{\scriptsize $\pm$0.56}}    \\
\bottomrule
\end{tabular}

}
    \end{minipage}
    % \hspace{0.3cm}
    \hfill
    \begin{minipage}{0.48\textwidth}
\centering
  \includegraphics[width=0.7\linewidth]{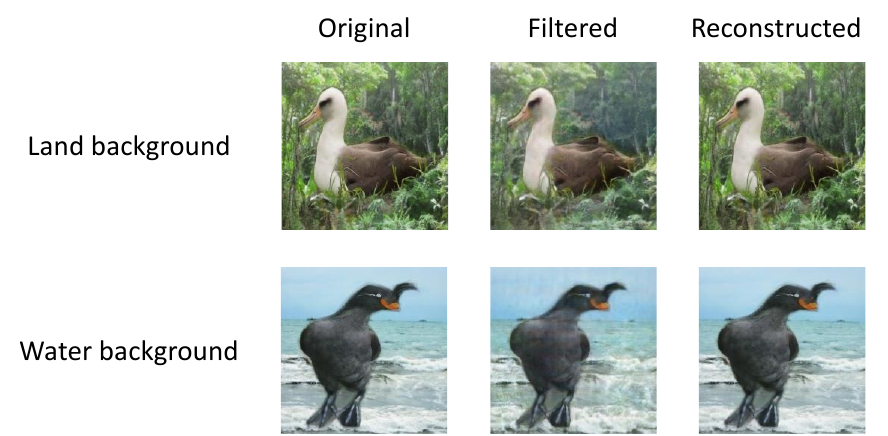}
  \caption{For the origin image with a land background, the background of the filtered image tends to shift towards a blue hue, bridging the gap between land and water backgrounds, since blue is the predominant color for water backgrounds. 
  On the other hand, for the original image with a water background, the background of the filtered image tends to become blurred and yield the texture of leaves.}
\label{fig:waterbirds}
    \end{minipage}%
\end{figure*}
% \end{figure}

%% file: sections/tables/CelebA_DFA_DRO.tex
% \begin{table*}[H]
\begin{table*}[ht!]
\caption{Comparison under extreme bias between DFA, Group DRO, GAN-Debiasing, and our method based on testing accuracy in CelebA.}
\label{tab:CelebA_DFA&DRO}
\centering

\resizebox{0.62\textwidth}{!}{%
\begin{tabular}{lcccc}
\toprule
\multirow{2}{*}{Method}          & \multicolumn{2}{c}{BlondHair}             & \multicolumn{2}{c}{HeavyMakeup}           \\
\cmidrule(lr){2-3}  \cmidrule(lr){4-5} 
                                 & Unbiased ↑          & Bias-conflicting ↑  & Unbiased ↑          & Bias-conflicting ↑  \\
                                 \midrule
Baseline                         & 66.11{\scriptsize $\pm$0.32}          & 33.89{\scriptsize $\pm$0.45}          & 59.50{\scriptsize $\pm$0.12}          & 30.09{\scriptsize $\pm$0.53}          \\
DFA~\cite{DFA}                              & 65.93{\scriptsize $\pm$0.43}          & 32.18{\scriptsize $\pm$0.63}          & 63.74{\scriptsize $\pm$0.43}          & 30.69{\scriptsize $\pm$0.97}          \\
Group DRO~\cite{DRO}                        & 63.82{\scriptsize $\pm$0.81}          & 34.23{\scriptsize $\pm$0.94}          & 61.43{\scriptsize $\pm$1.82}          & 30.92{\scriptsize $\pm$1.62}          \\
% \midrule
% GAN-Debiasing                    & 65.37{\scriptsize $\pm$0.75}          & 33.04{\scriptsize $\pm$0.96}          & 62.98{\scriptsize $\pm$1.46}          & 30.15{\scriptsize $\pm$1.78}          \\
% GAN-Debiasing (no target labels) & 61.84{\scriptsize $\pm$0.37}          & 31.84{\scriptsize $\pm$0.68}          & 58.91{\scriptsize $\pm$0.82}          & 29.75{\scriptsize $\pm$1.39}          \\
\midrule
Ours                             & 66.31{\scriptsize $\pm$0.26}          & 32.22{\scriptsize $\pm$0.43}          & 64.91{\scriptsize $\pm$1.32}          & 31.36{\scriptsize $\pm$2.43}          \\
Ours (FFHQ)                      & \textbf{71.53{\scriptsize $\pm$0.67}} & 47.17{\scriptsize $\pm$0.72}          & \textbf{71.59{\scriptsize $\pm$0.81}} & 38.64{\scriptsize $\pm$1.57}          \\
Ours (Synthetic)                 & 71.37{\scriptsize $\pm$0.64}          & \textbf{48.06{\scriptsize $\pm$0.82}} & 69.32{\scriptsize $\pm$0.95}          & \textbf{40.91{\scriptsize $\pm$1.54}} \\
\bottomrule
\end{tabular}

}
\end{table*}

%% file: sections/tables/IMDB_age_label.tex
\begin{table*}[h!]
\caption{Division of the age label in IMDB.}
\label{tab:IMDB_age_label}
\centering

\resizebox{1\textwidth}{!}{%

\begin{tabular}{lcccccccccccc}
\toprule
Age   & 0-19 & 20-24 & 25-29 & 30-34 & 35-39 & 40-44 & 45-49 & 50-54 & 55-59 & 60-64 & 65-60 & 70+ \\
\midrule
Label & 0    & 1     & 2     & 3     & 4     & 5     & 6     & 7     & 8     & 9     & 10    & 11 \\
\bottomrule
\end{tabular}

}
\end{table*}

%% file: sections/pictures/IMDB_boxplot.tex
\begin{figure*}[htbp]
    \centering % <-- added
    
      \subfloat[EB1.]{\includegraphics[width=0.3\linewidth]{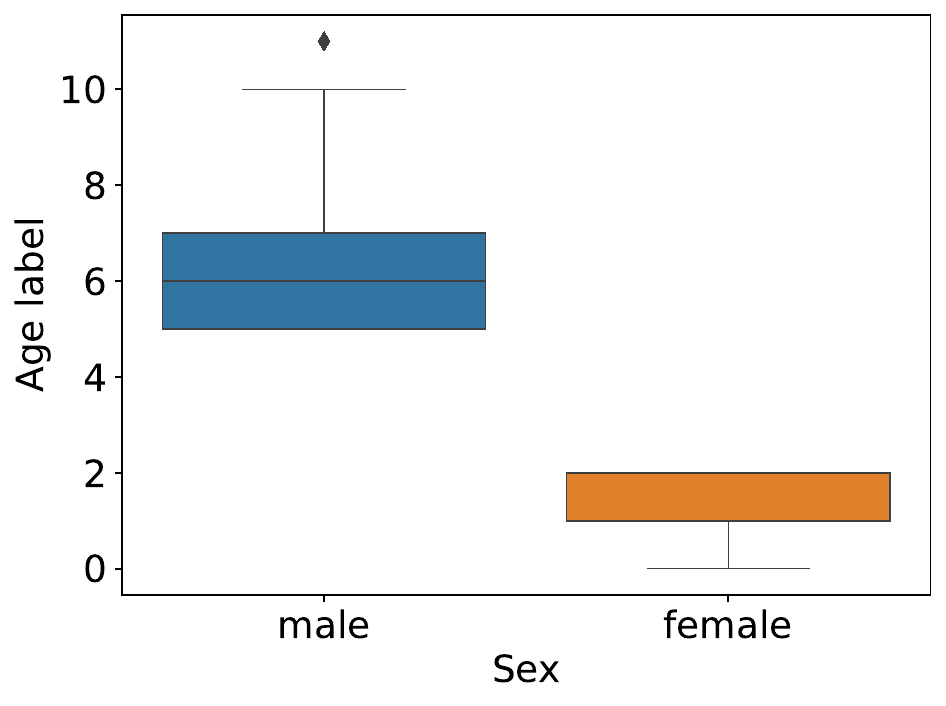}\label{fig:eb1}}\quad
      \subfloat[EB2.]{\includegraphics[width=0.3\linewidth]{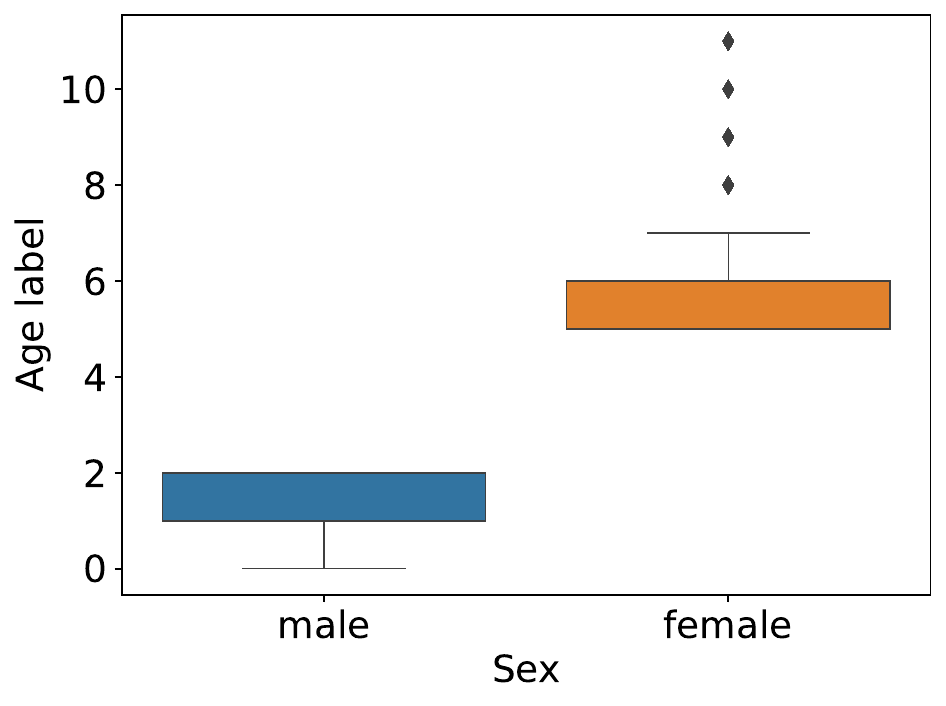}\label{fig:eb2}}\quad
      \subfloat[Test.]{\includegraphics[width=0.3\linewidth]{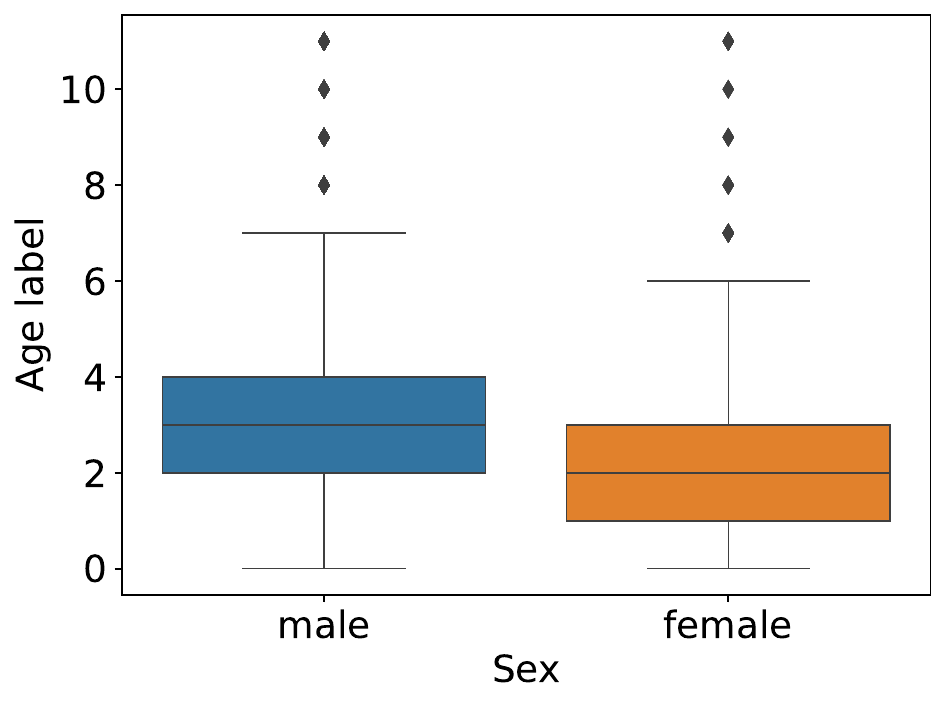}\label{fig:unbiased}}

\caption{Boxplots of age label distribution across sex in IMDB.}
\label{fig:boxplot_IMDB}
\end{figure*}

%% file: sections/pictures/IMDB_statistics.tex
\begin{figure*}[htbp]
    \centering % <-- added

      \subfloat[EB1 with 24997 females and 11007 males.]{\includegraphics[width=0.3\linewidth]{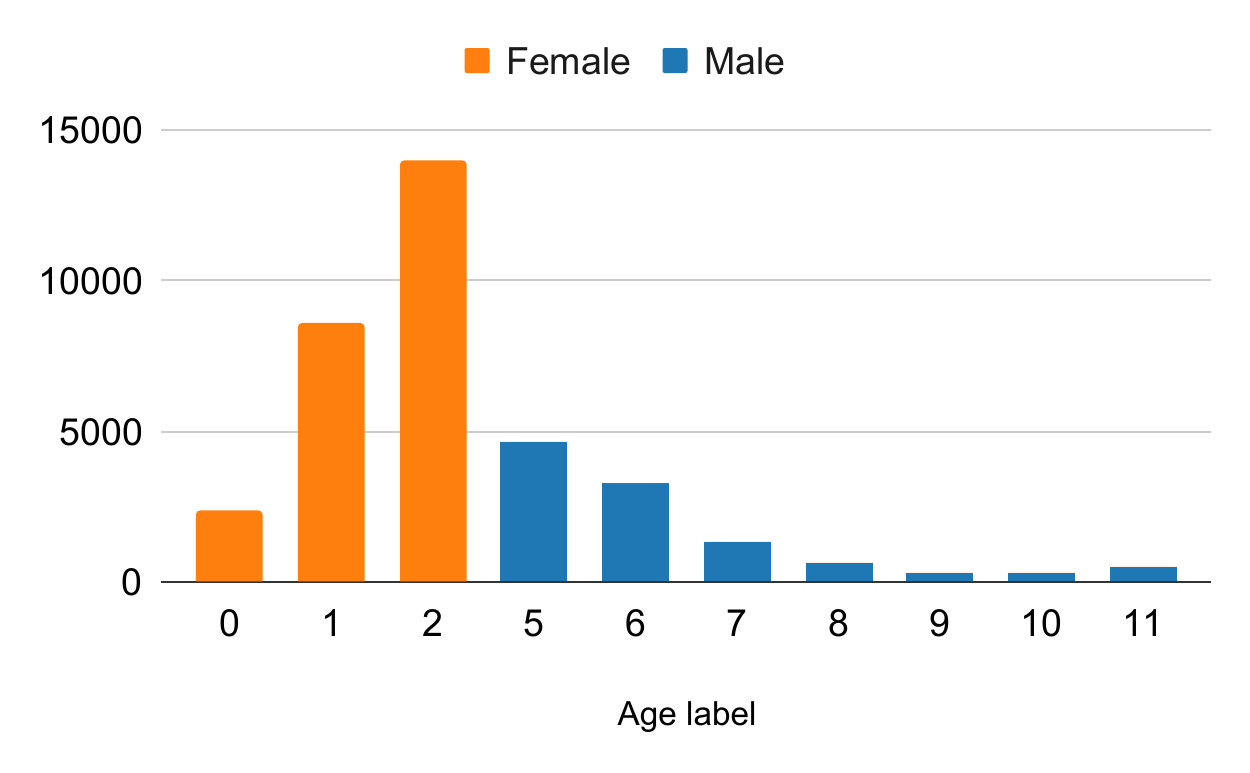}\label{fig:eb1_s}}\quad
      \subfloat[EB2 with 1744 females and 15056 males.]{\includegraphics[width=0.3\linewidth]{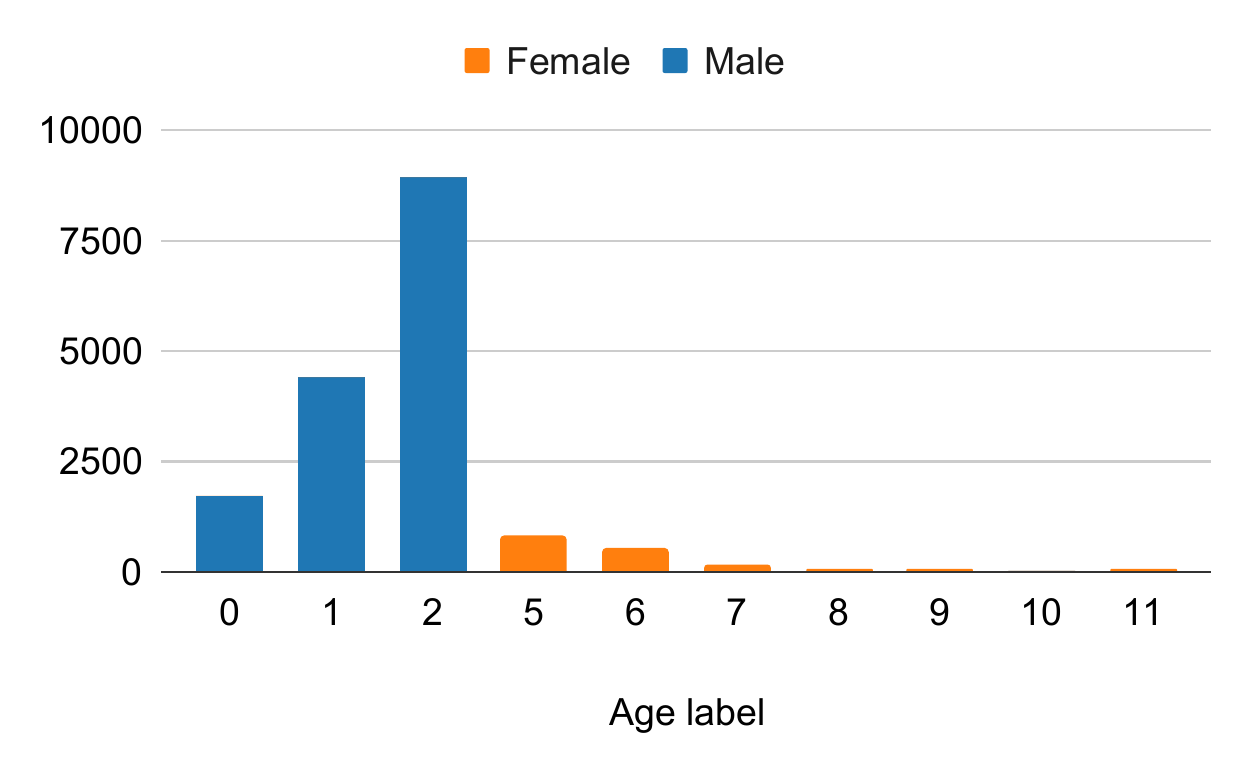}\label{fig:eb2_s}}\quad
      \subfloat[Test with 10753 females and 11715 males.]{\includegraphics[width=0.3\linewidth]{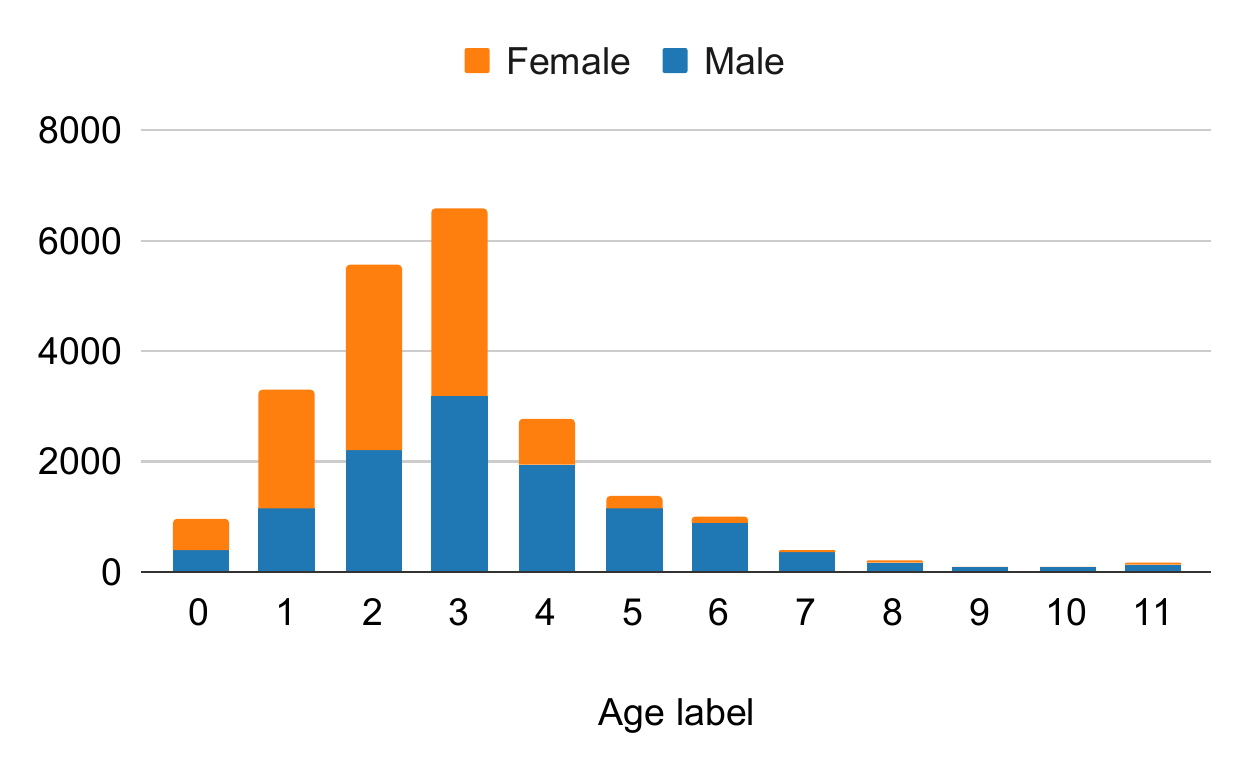}\label{fig:unbiased_s}}

\caption{Statistics for the number of females and males in different age categories of IMDB.}
\label{fig:statistics_IMDB}
\end{figure*}

%% file: sections/pictures/IMDB_EB_comparison+noisy.tex
\begin{figure*}[ht!]
    \begin{minipage}{0.48\textwidth}

\centering
      \subfloat[EB1 dominated by females aged 25-29 and males aged 40-44.]{\includegraphics[width=1\linewidth]{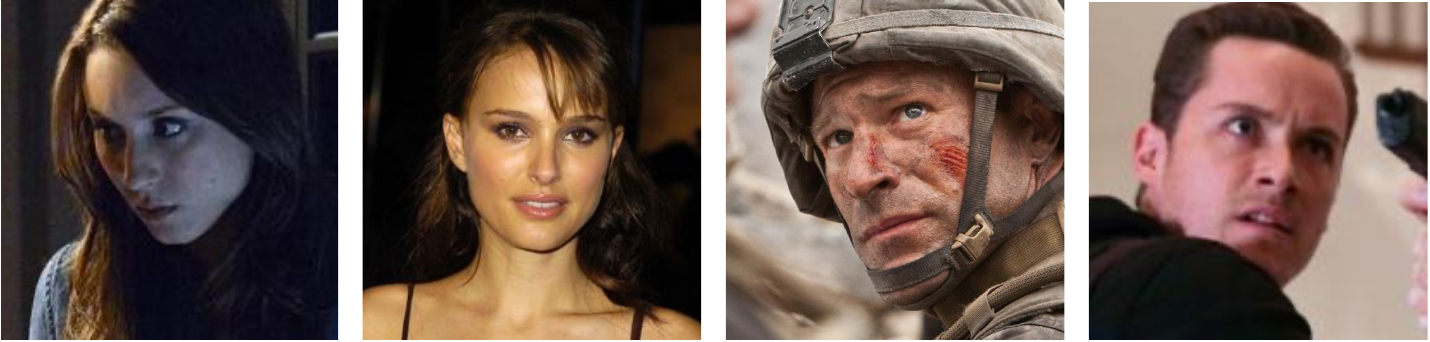}\label{fig:IMDB_comparison_eb1}}\quad
      
      \subfloat[EB2 dominated by females aged 40-44 and males aged 25-29.]{\includegraphics[width=1\linewidth]{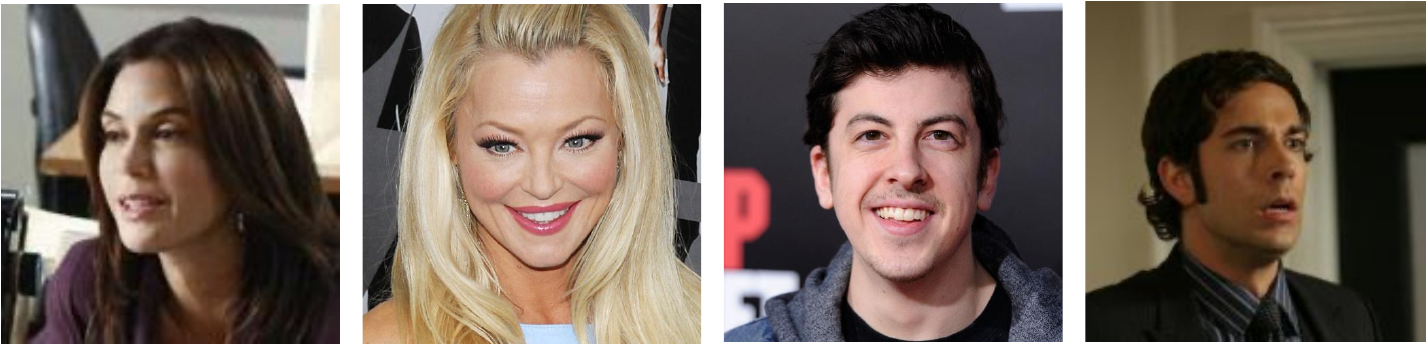}\label{fig:IMDB_comparison_eb2}}\quad

\caption{Comparison of dominated groups in EB1 and EB2.}
\label{fig:IMDB_EB_comparison}
    
    \end{minipage}
    \hfill
    \begin{minipage}{0.48\textwidth}
\centering
         % \vspace{1cm} 
      \subfloat[Examples taken from distance view.]{\includegraphics[width=1\linewidth]{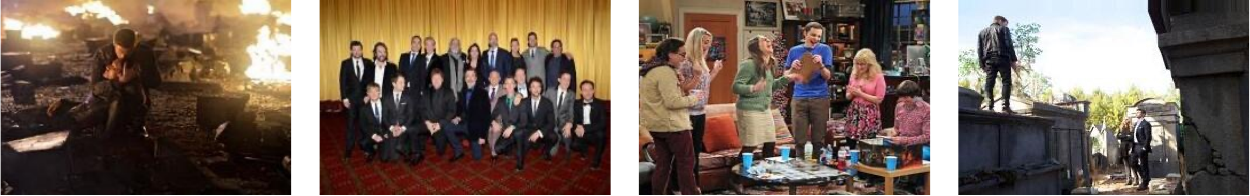}\label{fig:IMDB_long_distance}}\quad
      
         % \vspace{0.5cm} 
      \subfloat[Examples with two faces.]{\includegraphics[width=1\linewidth]{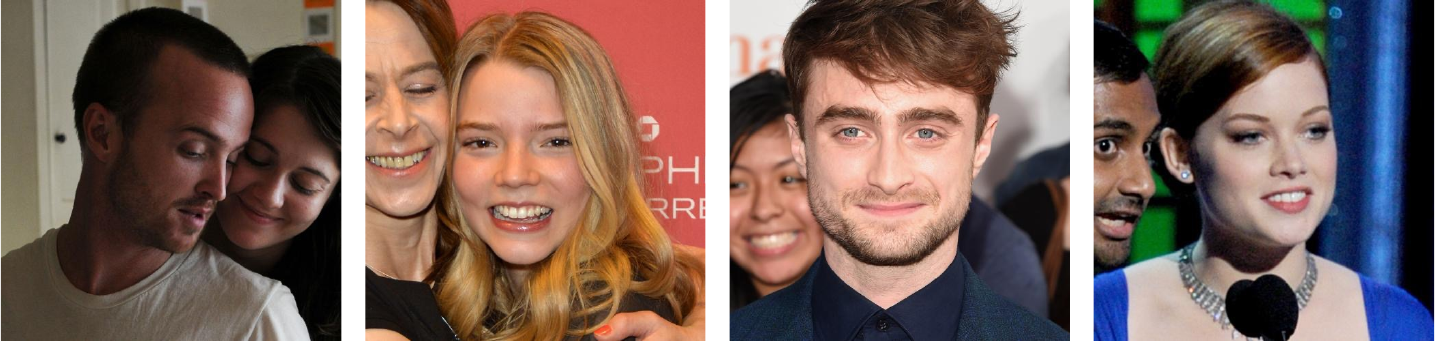}\label{fig:IMDB_two}}\quad

         % \vspace{0.5cm} 
\caption{Noisy examples of IMDB for sex classification.}
\label{fig:IMDB_noisy}

    \end{minipage}%
\end{figure*}

%% file: sections/tables/time_statistics.tex
\begin{table*}[htbp]
\caption{Time statistics based on one NVIDIA 1080 Ti GPU. The wavy line $\sim$ represents the approximation.}
\label{tab:time}
\centering
% \resizebox{0.7\textwidth}{!}{%

\begin{tabular}{lccccc}
\toprule
\multicolumn{1}{l}{\multirow{2}{*}{Experiment}} & \multirow{2}{*}{Resolution} & \multirow{2}{*}{Total image} & \multirow{2}{*}{Filter training} & \multicolumn{2}{c}{Classifier training} \\
\cmidrule(lr){5-6} 
\multicolumn{1}{c}{}                            &                             &                              &                                  & With filter        & Without filter     \\
\midrule
Colored MNIST                                   & 32$\times$32                       & 50000                        & $\sim$ 1h 10min                 & $\sim$ 26min      & $\sim$ 21min      \\
CelebA                                          & 224$\times$224                     & $\sim$ 180000               & $\sim$ 1d 2h                   & $\sim$ 6h 52min   & $\sim$ 5h 18min   \\
IMDB                                            & 224$\times$224                     & $\sim$ 50000                & $\sim$ 7h 31min                 & $\sim$ 3h 22min   & $\sim$ 2h 28min \\
Adult                                            & 108                     & $\sim$ 50000                & $\sim$ 23min                 & $\sim$ 5min   & $\sim$ 4min \\
\bottomrule
\end{tabular}

% }

\end{table*}

%% file: sections/tables/hp_filter.tex
\begin{table*}[htbp]
\caption{Hyper-parameters for filter-based approach training.}
\label{tab:hp_filter}
\centering
% \resizebox{0.7\textwidth}{!}{%

\begin{tabular}{lcccccc}
\toprule
Experiment    & $\lambda_{mi}$ & $\lambda_{pred}$ & $\lambda_{rec}$ & Learning rate & Batch size & Optimizer                 \\
\midrule
Colored MNIST & 10 & 100  & 100 & 0.0003        & 32         & Adam ($\beta_1$=0, $\beta_2$=0.9) \\
CelebA        & 50 & 50   & 100 & 0.0003        & 32         & Adam ($\beta_1$=0, $\beta_2$=0.9) \\
IMDB          & 50 & 50   & 100 & 0.0003        & 32         & Adam ($\beta_1$=0, $\beta_2$=0.9) \\
Adult          & 5 & 5   & 10 & 0.0003        & 32         & Adam ($\beta_1$=0, $\beta_2$=0.9) \\
\bottomrule
\end{tabular}

% }

\end{table*}

%% file: sections/tables/hp_task.tex
\begin{table*}[htbp]
\caption{Hyper-parameters for downstream task classifier training.}
\label{tab:hp_task}
\centering
% \resizebox{0.7\textwidth}{!}{%

\begin{tabular}{lccc}
\toprule
Experiment    & Learning rate & Batch size & Optimizer                     \\
\midrule
Colored MNIST & 0.001         & 32         & Adam ($\beta_1$=0.9, $\beta_2$=0.999) \\
CelebA        & 0.001         & 32         & Adam ($\beta_1$=0.9, $\beta_2$=0.999) \\
IMDB          & 0.001         & 32         & Adam ($\beta_1$=0.9, $\beta_2$=0.999) \\
Adult          & 0.001         & 32         & Adam ($\beta_1$=0.9, $\beta_2$=0.999) \\
\bottomrule
\end{tabular}

% }

\end{table*}

%% file: sections/tables/hp_SSL.tex
\begin{table*}[htbp]
% \caption{Hyper-parameters for training the approach based on self-supervised learning.}
\caption{Hyper-parameters for SSL-based approach training.}
\label{tab:hp_SSL}
\centering
% \resizebox{0.7\textwidth}{!}{%

\begin{tabular}{lcccc}
\toprule
Experiment    & Temperature & Learning rate & Batch size & Optimizer                     \\
\midrule
% SSL & 0.07 & 0.001         & 32         & Adam ($\beta_1$=0.9, $\beta_2$=0.999) \\
CelebA & 0.07 & 0.001         & 32         & Adam ($\beta_1$=0.9, $\beta_2$=0.999) \\
\bottomrule
\end{tabular}

% }

\end{table*}

%% file: sections/tables/CMNIST_filter_network.tex
\begin{table*}[htbp]
\caption{Networks of the target-agnostic adversarial filter on Colored MNIST.}
\label{tab:CMNIST_filter_network}
\centering
% \resizebox{0.9\textwidth}{!}{%

% \begin{tabular}{lllllll}
\begin{tabular}{lcccccc}
\toprule
Layer          & Kernel & Stride & Padding & Output shape      & Activation & Normalization \\
\midrule
\multicolumn{7}{c}{Encoder $G_{enc}$}                                                       \\
\midrule
Input          & -      & -      & -       & 3 $\times$ 32 $\times$ 32       & -          & -             \\
Convolution    & 4$\times$4    & 2$\times$2      & 1       & 64$\times$16$\times$16          & LeakyReLU  & BatchNorm     \\
Convolution    & 4$\times$4    & 2$\times$2      & 1       & 128$\times$8$\times$8           & LeakyReLU  & BatchNorm     \\
Convolution    & 4$\times$4    & 2$\times$2      & 1       & 256$\times$4$\times$4           & LeakyReLU  & BatchNorm     \\
            %   \midrule
               \midrule
\multicolumn{7}{c}{Decoder $G_{dec}$}                                                       \\
\midrule
% Layer          & Kernel & Stride & Padding & Output Shape      & Activation & Normalization \\
% \midrule
TransposedConv & 4$\times$4    & 2$\times$2      & 1       & 256$\times$8$\times$8           & ReLU       & BatchNorm     \\
TransposedConv & 4$\times$4    & 2$\times$2      & 1       & 128$\times$16$\times$16         & ReLU       & BatchNorm     \\
TransposedConv & 4$\times$4    & 2$\times$2      & 1       & 3$\times$32$\times$32           & Sigmoid    & BatchNorm     \\
% \midrule
\midrule
\multicolumn{7}{c}{Critic network $D$}                                                      \\
\midrule
% Layer          & Kernel & Stride & Padding & Output Shape      & Activation & Normalization \\
% \midrule
Convolution    & 4$\times$4    & 2$\times$2      & 1       & 64$\times$16$\times$16          & LeakyReLU  & InstanceNorm  \\
Convolution    & 4$\times$4    & 2$\times$2      & 1       & 128$\times$8$\times$8           & LeakyReLU  & InstanceNorm  \\
Convolution    & 4$\times$4    & 2$\times$2      & 1       & 256$\times$4$\times$4           & LeakyReLU  & InstanceNorm  \\
Linear         & -      & -      & -       & 256               & ReLU       & InstanceNorm  \\
Linear         & -      & -      & -       & 1                 & -          & -             \\
% \midrule
\midrule
\multicolumn{7}{c}{Regressor $R$}                                                           \\
\midrule
% Layer          & Kernel & Stride & Padding & Output Shape      & Activation & Normalization \\
% \midrule
Convolution    & 5$\times$5    & 1$\times$1      & 0       & 6$\times$28$\times$28           & LeakyReLU  & -             \\
AveargePooling & 2$\times$2    & 2$\times$2      & -       & 6$\times$14$\times$14           & -          & -             \\
Convolution    & 5$\times$5    & 1$\times$1      & 0       & 16$\times$10$\times$10          & LeakyReLU  & -             \\
AveargePooling & 2$\times$2    & 2$\times$2      & -       & 16$\times$5$\times$5            & -          & -             \\
Convolution    & 5$\times$5    & 1$\times$1      & 0       & 120$\times$1$\times$1           & LeakyReLU  & -             \\
Linear         & -      & -      & -       & 84                & ReLU       & -             \\
Linear         & -      & -      & -       & 3                 & Sigmoid    & -             \\
% \midrule
\midrule
\multicolumn{7}{c}{Auxiliary network $T$}                                           \\ 
\midrule
% Layer          & Kernel & Stride & Padding & Output Shape      & Activation & Normalization \\
% \midrule
Linear         & -      & -      & -       & 1024              & LeakyReLU  & -             \\
Linear         & -      & -      & -       & 256               & LeakyReLU  & -             \\
Linear         & -      & -      & -       & 64                & LeakyReLU  & -             \\
Linear         & -      & -      & -       & 10                & LeakyReLU  & -             \\
Linear         & -      & -      & -       & 1                 & -          & -             \\
\bottomrule
\end{tabular}

% }

\end{table*}

%% file: sections/tables/Adult_filter_network.tex
\begin{table*}[htbp]
\caption{Networks of the target-agnostic adversarial filter on Adult.}
\label{tab:Adult_filter_network}
\centering
% \resizebox{0.9\textwidth}{!}{%

\begin{tabular}{lcccccc}
\toprule
Layer                & Kernel               & Stride               & Padding              & Output shape         & Activation           & Normalization        \\
\midrule
\multicolumn{7}{c}{Encoder $G_{enc}$}                                                                                                                          \\
\midrule
Input                & -                    & -                    & -                    & 108                  & -                    & -                    \\
Linear               & -                    & -                    & -                    & 64                   & ReLU                 & -                    \\
Linear               & -                    & -                    & -                    & 10                   & ReLU                 & -                    \\
\midrule
                     % &                      &                      &                      &                      &                      &                      \\
\multicolumn{7}{c}{Decoder $G_{dec}$}                                                                                                                          \\
% Layer                & Kernel               & Stride               & Padding              & Output Shape         & Activation           & Normalization        \\
\midrule
Linear               & -                    & -                    & -                    & 64                   & ReLU                 & -                    \\
Linear               & -                    & -                    & -                    & 108                  & -                    & -                    \\
\midrule
% \multicolumn{1}{l}{} & \multicolumn{1}{l}{} & \multicolumn{1}{l}{} & \multicolumn{1}{l}{} & \multicolumn{1}{l}{} & \multicolumn{1}{l}{} & \multicolumn{1}{l}{} \\
\multicolumn{7}{c}{Critic network $D$}                                                                                                                         \\
\midrule
% Layer                & Kernel               & Stride               & Padding              & Output Shape         & Activation           & Normalization        \\
Linear               & -                    & -                    & -                    & 64                   & ReLU                 & -                    \\
Linear               & -                    & -                    & -                    & 1                    & -                    & -                    \\
\midrule
% \multicolumn{1}{l}{} & \multicolumn{1}{l}{} & \multicolumn{1}{l}{} & \multicolumn{1}{l}{} & \multicolumn{1}{l}{} & \multicolumn{1}{l}{} & \multicolumn{1}{l}{} \\
\multicolumn{7}{c}{Regressor $R$}                                                                                                                              \\
\midrule
% Layer                & Kernel               & Stride               & Padding              & Output Shape         & Activation           & Normalization        \\
Linear               & -                    & -                    & -                    & 64                   & ReLU                 & -                    \\
Linear               & -                    & -                    & -                    & 1                    & Sigmoid              & -                    \\
                     % &                      &                      &                      &                      &                      &                      \\
                     \midrule
\multicolumn{7}{c}{Auxiliary network $T$}                                                                                     \\
\midrule
% Layer                & Kernel               & Stride               & Padding              & Output Shape         & Activation           & Normalization        \\
Linear               & -                    & -                    & -                    & 10                   & LeakyReLU            & -                    \\
Linear               & -                    & -                    & -                    & 1                    & -                    & -               \\
\bottomrule
\end{tabular}

% }

\end{table*}

%% file: sections/tables/CelebA_filter_network.tex
\begin{table*}[htbp]
\caption{Networks of the target-agnostic adversarial filter on CelebA, IMDB.}
\label{tab:CelebA_filter_network}
\centering
% \resizebox{0.9\textwidth}{!}{%

% \begin{tabular}{lllllll}
\begin{tabular}{lcccccc}
\toprule
Layer          & Kernel & Stride & Padding & Output shape & Activation & Normalization \\
\midrule
\multicolumn{7}{c}{Encoder $G_{enc}$}                                                  \\
\midrule
Input          & -      & -      & -       & 3$\times$224$\times$224    & -          & -             \\
Convolution    & 4$\times$4    & 2$\times$2    & 1       & 64$\times$112$\times$112   & LeakyReLU  & BatchNorm     \\
Convolution    & 4$\times$4    & 2$\times$2    & 1       & 128$\times$56$\times$56    & LeakyReLU  & BatchNorm     \\
Convolution    & 4$\times$4    & 2$\times$2    & 1       & 256$\times$28$\times$28    & LeakyReLU  & BatchNorm     \\
Convolution    & 4$\times$4    & 2$\times$2    & 1       & 512$\times$14$\times$14    & LeakyReLU  & BatchNorm     \\
Convolution    & 4$\times$4    & 2$\times$2    & 1       & 1024$\times$7$\times$7     & LeakyReLU  & BatchNorm     \\
\midrule
\multicolumn{7}{c}{Decoder $G_{dec}$}                                                  \\
\midrule
% Layer          & Kernel & Stride & Padding & Output Shape & Activation & Normalization \\
TransposedConv & 4$\times$4    & 2$\times$2    & 1       & 1024$\times$14$\times$14   & ReLU       & BatchNorm     \\
TransposedConv & 4$\times$4    & 2$\times$2    & 1       & 512$\times$28$\times$28    & ReLU       & BatchNorm     \\
TransposedConv & 4$\times$4    & 2$\times$2    & 1       & 256$\times$56$\times$56    & ReLU       & BatchNorm     \\
TransposedConv & 4$\times$4    & 2$\times$2    & 1       & 128$\times$112$\times$112  & ReLU       & BatchNorm     \\
TransposedConv & 4$\times$4    & 2$\times$2    & 1       & 3$\times$224$\times$224    & Sigmoid    & BatchNorm     \\
\midrule
\multicolumn{7}{c}{Critic network $D$}                                                 \\
\midrule
% Layer          & Kernel & Stride & Padding & Output Shape & Activation & Normalization \\
Convolution    & 4$\times$4    & 2$\times$2    & 1       & 64$\times$112$\times$112   & LeakyReLU  & InstanceNorm  \\
Convolution    & 4$\times$4    & 2$\times$2    & 1       & 128$\times$56$\times$56    & LeakyReLU  & InstanceNorm  \\
Convolution    & 4$\times$4    & 2$\times$2    & 1       & 256$\times$28$\times$28    & LeakyReLU  & InstanceNorm  \\
Convolution    & 4$\times$4    & 2$\times$2    & 1       & 512$\times$14$\times$14    & LeakyReLU  & InstanceNorm  \\
Convolution    & 4$\times$4    & 2$\times$2    & 1       & 1024$\times$7$\times$7     & LeakyReLU  & InstanceNorm  \\
Linear         & -      & -      & -       & 1024         & ReLU       & InstanceNorm  \\
Linear         & -      & -      & -       & 1            & -          & -             \\
\midrule
\multicolumn{7}{c}{Regressor $R$}                                                      \\
\midrule
% Layer          & Kernel & Stride & Padding & Output Shape & Activation & Normalization \\
ResNet18$^*$      & -      & -      & -       & 128          & -          & -             \\
Linear         & -      & -      & -       & 1            & Sigmoid    & -             \\
\midrule
\multicolumn{7}{c}{Auxiliary network $T$}                                              \\
\midrule
% Layer          & Kernel & Stride & Padding & Output Shape & Activation & Normalization \\
Linear         & -      & -      & -       & 1024         & LeakyReLU  & -             \\
Linear         & -      & -      & -       & 256          & LeakyReLU  & -             \\
Linear         & -      & -      & -       & 64           & LeakyReLU  & -             \\
Linear         & -      & -      & -       & 10           & LeakyReLU  & -             \\
Linear         & -      & -      & -       & 1            & -          & -             \\
\bottomrule
\multicolumn{7}{l}{\small$^*$We modify the last layer of ResNet18 to be a linear layer with 128 output channels and ReLU as activation} \\
\multicolumn{7}{l}{followed by a dropout with $p=0.5$.}
\end{tabular}

% }

\end{table*}

% \footnote{We modify the last layer of ResNet18 to be a linear layer with 128 output channels and ReLU as activation following dropout with $p=0.5$.}

%% file: sections/tables/CelebA_SSL_network.tex
\begin{table*}[htbp]
\caption{Networks of the approach based on self-supervised learning on CelebA.}
\label{tab:CelebA_SSL_network}
\centering
% \resizebox{0.9\textwidth}{!}{%

% \begin{tabular}{lllllll}
\begin{tabular}{lcccccc}
\toprule
Layer     & Kernel & Stride & Padding & Output shape & Activation & Normalization \\
\midrule
\multicolumn{7}{c}{Encoder $E$}                                                       \\
\midrule
ResNet18 & -      & -      & -       & 1000         & -          & -             \\
\midrule
\multicolumn{7}{c}{Mapping network $M$}                                                     \\
\midrule
% Layer     & Kernel & Stride & Padding & Output Shape & Activation & Normalization \\
Linear    & -      & -      & -       & 512          & ReLU       & -             \\
Linear    & -      & -      & -       & 128          & ReLU       & -             \\
\midrule
\multicolumn{7}{c}{Regressor $R$}                                                    \\
\midrule
% Layer     & Kernel & Stride & Padding & Output Shape & Activation & Normalization \\
Linear    & -      & -      & -       & 512          & ReLU       & -             \\
Linear    & -      & -      & -       & 64           & ReLU       & -             \\
Linear    & -      & -      & -       & 10           & ReLU       & -             \\
Linear    & -      & -      & -       & 1            & -          & -            \\
\bottomrule
\end{tabular}

% }

\end{table*}